\documentclass{article}

\usepackage{microtype}
\usepackage{graphicx}
\usepackage{subcaption}
\usepackage{booktabs} 

\usepackage{hyperref}

\usepackage[preprint]{icml2026}

\usepackage{amsmath}
\usepackage{amssymb}
\usepackage{mathtools}
\usepackage{amsthm}

\usepackage{enumitem}
\usepackage{thmtools}
\usepackage[T1]{fontenc}
\usepackage{diagbox}
\usepackage{hhline}
\usepackage{makecell}
\usepackage{booktabs}

\usepackage[capitalize,noabbrev]{cleveref}

\theoremstyle{plain}
\newtheorem{theorem}{Theorem}[section]

\newtheorem{lemma}[theorem]{Lemma}

\theoremstyle{definition}
\newtheorem{definition}[theorem]{Definition}

\theoremstyle{remark}

\usepackage[textsize=tiny]{todonotes}

\icmltitlerunning{Submission and Formatting Instructions for ICML 2026}

\begin{document}

\twocolumn[
  \icmltitle{Breakthrough the Suboptimal Stable Point in Value-Factorization-Based Multi-Agent Reinforcement Learning}



  \icmlsetsymbol{equal}{*}

  \begin{icmlauthorlist}
    \icmlauthor{Lesong Tao}{xjtu}
    \icmlauthor{Yifei Wang}{xjtu}
    \icmlauthor{Haodong Jing}{xjtu}
    \icmlauthor{Jingwen Fu}{bza}
    \icmlauthor{Miao Kang}{xjtu}
    \icmlauthor{Shitao Chen}{xjtu}
    \icmlauthor{Nanning Zheng}{xjtu}
  \end{icmlauthorlist}

  \icmlaffiliation{xjtu}{State Key Laboratory of Human-Machine Hybrid Augmented Intelligence, Institute of Artificial Intelligence and Robotics, Xi'an Jiaotong University}
  \icmlaffiliation{bza}{Zhongguancun Academy}

  \icmlcorrespondingauthor{Nanning Zheng}{nnzheng@mail.xjtu.edu.cn}

  \icmlkeywords{multi-agent reinforcement learning, value factorization, representational limitation}

  \vskip 0.3in
]




\printAffiliationsAndNotice{}  

\begin{abstract}

   Value factorization, a popular paradigm in MARL, faces significant theoretical and algorithmic bottlenecks: its tendency to converge to suboptimal solutions remains poorly understood and unsolved. Theoretically, existing analyses fail to explain this due to their primary focus on the optimal case. To bridge this gap, we introduce a novel theoretical concept: the stable point, which characterizes the potential convergence of value factorization in general cases. Through an analysis of stable point distributions in existing methods, we reveal that non-optimal stable points are the primary cause of poor performance. However, algorithmically, making the optimal action the unique stable point is nearly infeasible. In contrast, iteratively filtering suboptimal actions by rendering them unstable emerges as a more practical approach for global optimality. Inspired by this, we propose a novel Multi-Round Value Factorization (MRVF) framework. Specifically, by measuring a non-negative payoff increment relative to the previously selected action, MRVF transforms inferior actions into unstable points, thereby driving each iteration toward a stable point with a superior action. Experiments on challenging benchmarks, including predator-prey tasks and StarCraft II Multi-Agent Challenge (SMAC), validate our analysis of stable points and demonstrate the superiority of MRVF over state-of-the-art methods.
    
\end{abstract}

\section{Introduction}
Multi-agent reinforcement learning (MARL) effectively addresses many real-world problems, such as robotics~\citep{Ref:robotic}, automated warehouses~\citep{Ref:poaql}, and games~\citep{Ref:dota2}. MARL's core issue is finding the optimal action in an action space that grows exponentially with the number of agents. To address this issue, value factorization methods represent the joint action value in a specific form. Early value factorization methods, such as VDN~\citep{Ref:vdn} and QMIX~\citep{Ref:qmix}, represent the joint action value through a monotonic mix of individual action values. This monotonic relationship ensures that the optimal joint action can be obtained by searching through individual action spaces. However, because of the monotonic particularity, these methods struggle to obtain the global optimum in non-monotonic cases. 

\begin{figure}[t]
    \centering
    \includegraphics[scale=0.32]{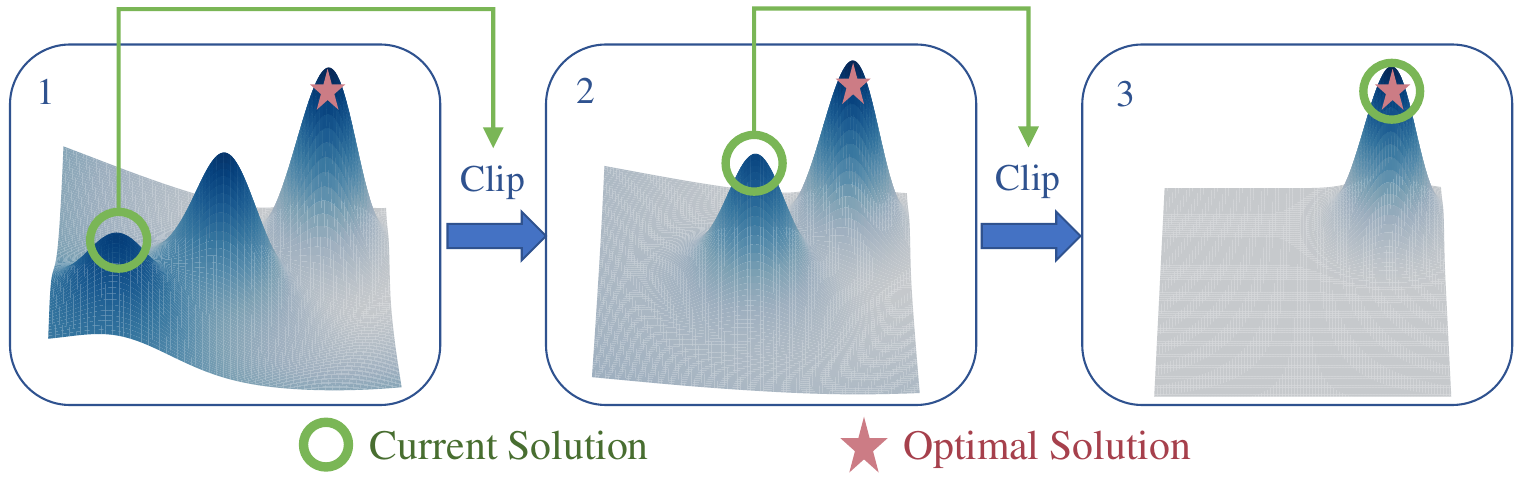}
    \caption{Clipping payoffs in the multi-round process. In this example, the two suboptimal solutions have relatively high payoffs. Therefore, existing methods are likely to select one of them. In contrast, by clipping the payoff of the selected solution, MRVF avoids convergence to it in subsequent rounds. Eventually, only the optimal solution retains a high payoff, making it easily identifiable.}
    
    
    \label{intro mrvf}
\end{figure}


To mitigate the representational limitation in VDN and QMIX, recent methods introduce new fitting functions (e.g., QPLEX~\citep{Ref:qplex} and ResQ~\citep{Ref:resq}) or loss functions (e.g., QTRAN~\citep{Ref:qtran} and WQMIX~\citep{Ref:wqmix}). However, these algorithms often converge to suboptimal solutions, particularly when initialized non-optimally. This phenomenon remains insufficiently understood, since existing analyses typically assume convergence from the optimal initialization. To bridge this gap, \textbf{we introduce a novel theory to analyze the convergence under general initialization.} A key issue is the gradient discontinuity inherent to value factorization under general initialization. To address this, we model the solution trajectory as a discrete dynamical system, with \textbf{stable points} corresponding to converged solutions. This perspective enables our analysis to more accurately reflect empirical outcomes than prior work~\citep{Ref:tadppo}. By characterizing the distribution of stable points, we demonstrate that the prevalence of suboptimal stable points is a primary cause of the poor performance exhibited by existing methods.

From the above discussion, a necessary condition for existing methods to achieve global optimality is that the optimal action is the unique stable point. However, this uniqueness rarely holds under a single factorization due to the high dimensionality of multi-agent action spaces. In contrast, we point out a more tractable alternative: iteratively transforming suboptimal actions into unstable points until only the optimal action remains stable. Inspired by this, \textbf{we propose a multi-round value factorization (MRVF) framework}. As illustrated in Figure~\ref{intro mrvf}, \textbf{1) for each round}, MRVF clips the difference between raw payoff and the payoff obtained in the preceding round to zero. This clipping operation not only renders inferior actions unstable, but also enhances the stability of superior actions by suppressing surrounding low-payoff influences.  Consequently, the solution in each round converges to a superior action. \textbf{2) As a whole}, since every round strictly improves the solution, MRVF can find the optimal solution given a sufficient number of rounds.



Environments with highly non-monotonic payoffs~\citep{Ref:wqmix} require significant coordination, which is extremely challenging for MARL. Accordingly, we conduct experiments on non-monotonic one-step games and the predator-prey task~\citep{Ref:DCG}. Existing methods perform poorly in these environments, corroborating our theoretical results on their suboptimality. In contrast, MRVF achieves better performance in both highly non-monotonic and nearly monotonic settings (e.g., the StarCraft Multi-Agent Challenge, SMAC~\citep{Ref:smac}), demonstrating that its multi-round iterations can attain a nearly optimal solution.


\section{Related Work}

\paragraph{Value-based Methods} IQL~\citep{Ref:IQL} is one of the earliest methods where agents independently learn individual action values. However, this method does not account for the dynamic policies of other agents. To address this issue, recent works propose learning a factorized joint action value under the IGM principle. Monotonic factorization is an intuitive way to satisfy the IGM principle. For example, VDN~\citep{Ref:vdn} represents the joint action value as the sum of individual action values. QMIX~\citep{Ref:qmix} represents the joint action value with a monotonic network that uses a positive weighted linear to deal with individual action values. Qattan~\citep{Ref:qatten} uses multi-head attention~\citep{Ref:attention} to generate mixing weights. NA\textsuperscript{2}Q~\citep{Ref:na2q} improves interpretability through Taylor expansion-based monotonic factorization. GoMARL~\citep{Ref:gomarl} and HYGMA~\citep{Ref:hygma} introduce a grouping mechanism into monotonic factorization. However, monotonic factorizations may suffer from representational limitations. To address this issue, QTRAN~\citep{Ref:qtran} and WQMIX~\citep{Ref:wqmix} pay more attention to the estimation of greedy action values. Specifically, QTRAN introduces a compensatory base to relax the non-optimal errors, while WQMIX reduces the weight on the approximation of non-optimal parts. Other methods like QPLEX~\citep{Ref:qplex} and ResQ~\citep{Ref:resq} design complete factorizations of the joint action value. QPLEX uses the maximum of a mixer value and a non-positive advantage function. ResQ combines a monotonic function with a non-positive function. GVR~\citep{Ref:gvr} takes a different approach by making \(\boldsymbol{u}^*\) the only stable point, although this requires numerous approximations.

\paragraph{Policy-based Methods}
In addition to value factorization methods, policy-based approaches also constitute a popular paradigm in MARL. COMA~\citep{Ref:COMA} introduces a counterfactual baseline in actor-critic framework, while MAPPO~\citep{Ref:mappo} adapts PPO~\citep{Ref:PPO} to multi-agent settings. Policy-based methods can be regarded as a special case of value factorization: value factorization methods factorize the joint action-value into individual Q-values, whereas policy-based methods factorize it into individual policy distributions. Building on this similarity, several works integrate value factorization into policy-based methods to handle continuous action spaces, exemplified by DOP~\citep{Ref:dop} and FACMAC~\citep{Ref:facmac}.

\paragraph{Communication-based Methods} Other researchers study communication mechanisms in multi-agent systems. CommNet~\citep{Ref:CommNet} enables multi-round communication for sharing observations. DIAL~\citep{Ref:DIAL} designs gradient-based messages exchanged between agents. To decide with whom to communicate, TarMAC~\citep{Ref:TarMac} uses a signature-based soft attention mechanism, and MAGIC~\citep{Ref:MAGIC} uses graph neural networks. The communication mechanisms can be integrated into either value-based or policy-based methods. For example, NDQ~\citep{Ref:NDQ} combines value factorization and communication by mixing individual \(Q\)s through communication, while CommFormer~\citep{Ref:commformer} establishes a learnable graph for sharing information in MAPPO. ACE~\citep{Ref:ACE} models communication as intermediate processes in the Markov Decision Process. In addition to parallel decision making, recent methods study sequential decision making through communication. PG-AR~\citep{Ref:pgar} randomizes the order of decision making. SeqComm~\citep{Ref:seqcomm} use intentions to prioritize agents in sequential decision-making.

\paragraph{Other MARL Methods} Recent works introduce additional innovations in MARL. RES~\citep{Ref:res} addresses value overestimation using regularized softmax losses. MAVEN~\citep{Ref:maven} enhances exploration efficiency through randomized latent vectors. LIGS~\citep{Ref:ligs} proposes learnable intrinsic rewards to improve multi-agent cooperation. SHAQ~\citep{Ref:shaq} integrates the Shapley value by modeling Q functions as marginal contributions of agents.

\section{Background}

\subsection{Dec-POMDP}
In cooperative multi-agent systems, agents interact with the environment to achieve common objectives. This process can be modeled as a decentralized partially observable Markov decision process (Dec-POMDP)~\citep{Ref:review}, defined by a tuple \(<\mathcal{S},\boldsymbol{\mathcal{U}},P,O,R,\gamma,n>\), in which \(n\) is the number of agents,  \(\mathcal{S}\) is the state space, \(\boldsymbol{\mathcal{U}}=\mathcal{U}_1\times\mathcal{U}_2\times\cdots\times\mathcal{U}_n\) is the action space, \(P(\boldsymbol{s}^{'}|\boldsymbol{s},\boldsymbol{u}):\mathcal{S} \times \boldsymbol{\mathcal{U}} \times \mathcal{S} \to [0,1]\) is the transition probability between the states, \(O:\mathcal{S}\to\boldsymbol{\mathcal{O}}\) is the joint observation function where \(\boldsymbol{\mathcal{O}}\) is the joint observation space, \(r \sim R:\mathcal{S}\times\boldsymbol{\mathcal{U}}\to\mathbb{R}\) is the reward, \(\gamma\in(0,1]\) is the discount, agent have an action-observation history \(\boldsymbol{\tau}\in \boldsymbol{\mathcal{T}} \equiv(\boldsymbol{\mathcal{O}}\times\boldsymbol{\mathcal{U}})^*\).  Similarly to RL with a single agent, the objective of MARL is to find the policy \(\pi:\boldsymbol{\mathcal{T}}\times\boldsymbol{\mathcal{U}}\to[0,1]\) that maximizes the joint action value \(Q_{\mathrm{jt}}(\boldsymbol{\tau},\boldsymbol{u})=\mathrm{E}_{\pi}[\sum_{t} \gamma^tr_t|\boldsymbol{\tau},\boldsymbol{u}]\), the expectation of return. However, obtaining the global optimal action \(\boldsymbol{u}^* = \arg\max_{\boldsymbol{u}} Q_{\mathrm{jt}}(\boldsymbol{\tau},\boldsymbol{u})\) by searching the large action space is intractable for value-based learning. 

\subsection{Monotonic Value Factorization}
Value factorization methods design \(Q_{\mathrm{tot}}\) to approximate \(Q_{\mathrm{jt}}\) by minimizing \(L_\mathrm{tot}\).

\begin{equation}
    \label{equ: Ltot of value factorization}
        L_{\mathrm{tot}}=\sum_{\boldsymbol{\tau},\boldsymbol{u}}{\pi(\boldsymbol{u}|\boldsymbol{\tau})(Q_{\mathrm{tot}}(\boldsymbol{\tau},\boldsymbol{u})-Q_{\mathrm{jt}}(\boldsymbol{\tau},\boldsymbol{u}))^2}
\end{equation}

Specifically, VDN~\citep{Ref:vdn} factorizes \(Q_{\mathrm{tot}}\) into the sum of individual \(Q\)s: \(Q_{\mathrm{tot}}=\sum_i{Q_i}\). And QMIX~\citep{Ref:qmix} uses a monotonic function \(f_{\mathrm{mon}}\) to represent the relationship: \(Q_{\mathrm{tot}}=f_{\mathrm{mon}}(Q_1, Q_2, \cdots, Q_n)\) where \(\frac{\partial{f_{\mathrm{mon}}}}{\partial{Q_i}} \geq 0, \forall i\in \{1,2,\cdots, N\}\). The monotonicity of \(Q_{\mathrm{tot}}\) with respect to \(Q_i\) ensures the alignment of their optimal actions, which is the Individual-Global-Max (IGM) principle~\footnote{We use "\(\subseteq\)" instead of "\(=\)" in Equation~(\ref{IGM principle}), as \(Q_{\mathrm{tot}}\) may have more maxima than individual \(Q\)s (Table~\ref{CW-QMIX fail example}).}:
\begin{equation}
    \prod\limits_{i=0}^n\arg\max \limits_{u_i} Q_i(\tau_i,u_i)
    \subseteq
    \arg\max _{\boldsymbol{u}}Q_{\mathrm{tot}}(\boldsymbol{\tau},\boldsymbol{u})
    \label{IGM principle}
\end{equation}

where \(\prod\limits_{i=0}^n\arg\max \limits_{u_i} Q_i(\tau_i,u_i)\) is the Cartesian product of the sets \(\arg\max \limits_{u_i} Q_i(\tau_i,u_i)\), and \(Q_i:\mathcal{T}_i\times \mathcal{U}_i\to\mathbb{R}\) is the individual action value of agent \(i\). Under the IGM principle, if \(Q_{\mathrm{tot}}\) approximates \(Q_{\mathrm{jt}}\) precisely, the maximum of \(Q_{\mathrm{jt}}\) would be obtained by searching the maximum of \(Q_i\). Here we define the greedy action \(\bar{\boldsymbol{u}}\) as
\begin{equation}
    \bar{\boldsymbol{u}}\in\prod\limits_{i=0}^n\arg\max \limits_{u_i} Q_i(\tau_i,u_i)
    \label{definition of greedy action}
\end{equation}

\subsection{Gradient-Free Components in Value Factorization}

Loss function \(L_{tot}\) defined in existing value factorization contains gradient-free components that are not updated via gradient backpropagation. In this paper, we focus on those whose values depend on the outcome of \(\arg\max \limits_{u_i} Q_i(\tau_i,u_i)\).

\begin{definition}[Gradient-free components and \(\widetilde{\boldsymbol{u}}\)]
    Gradient-free components are functions that take a formal parameter \(\widetilde{\boldsymbol{u}}\), where \(\widetilde{\boldsymbol{u}}\) is instantiated by the greedy action \(\bar{\boldsymbol{u}}\).
\end{definition}

Gradient-free components are common in value factorization, for example, the decentralized \(\epsilon\)-greedy policy \(\pi\) 
\begin{equation}
    \pi(\boldsymbol{u}|\boldsymbol{\tau})=\left(\frac{\epsilon}{|\mathcal{U}|}\right)^{n-m} \left(1-\epsilon+\frac{\epsilon}{|\mathcal{U}|}\right)^m
    \label{decentralized epsilon greedy}
\end{equation}
where \(m=|\{i|\boldsymbol{u}_i=\widetilde{\boldsymbol{u}}_i\}|\). Recent methods introduce additional terms involving \(\widetilde{\boldsymbol{u}}\) into monotonic factorization. For example, the weight \(w\) of the weighted \(L_\mathrm{tot}\) in WQMIX~\citep{Ref:wqmix}~\footnote{It stands for Centrally-Weighted QMIX.}.
\begin{equation}
    w(\boldsymbol{\tau},\boldsymbol{u})=
    \begin{cases}
        1 &   \hat{Q}_{\mathrm{jt}}(\boldsymbol{\tau},\boldsymbol{u})>\hat{Q}_{\mathrm{jt}}(\boldsymbol{\tau},\widetilde{\boldsymbol{u}}) \ \ or \ \ \boldsymbol{u}=\widetilde{\boldsymbol{u}}   \\
        \alpha<1 & otherwise
    \end{cases}
\end{equation}
where \(\hat{Q}_{\mathrm{jt}}\) is an approximation of \(Q_{\mathrm{jt}}\), and the residual mask \(w_\mathrm{r}\) in \(Q_{tot}\) of ResQ~\citep{Ref:resq}.
\begin{equation}
    w_\mathrm{r}(\boldsymbol{\tau},\boldsymbol{u})=
    \begin{cases}
        0   &   \boldsymbol{u}=\widetilde{\boldsymbol{u}} \\
        1   &   otherwise
    \end{cases}
\end{equation}



\section{A Unified Theory for the Suboptimality of Existing Value Factorization}\label{Sec:theory}

In this section, we discuss why existing value factorization methods converge to suboptimal solutions during training. To explain this phenomenon, we propose a novel theory to study how greedy actions change and converge during training. \textbf{1) Regarding the transition of greedy actions}, prior work~\citep{Ref:tadppo} overlooks the gradient discontinuities caused by the discrete changes in greedy actions (as illustrated in Figure~\ref{gradient discontinuity}). To address this issue, we model the transitions of greedy actions using a discrete dynamical system, thereby enabling gradient-based analysis at every transition step. \textbf{2) Regarding the convergence of greedy actions}, existing studies~\citep{Ref:qtran, Ref:wqmix, Ref:qplex, Ref:resq} only analyze convergence when the greedy action is optimal. In contrast, we provide conditions under which greedy actions converge in the general case and give specific examples showing how current algorithms converge to suboptimal solutions.


\begin{figure}[h!]
    \centering
    \includegraphics[scale=0.75]{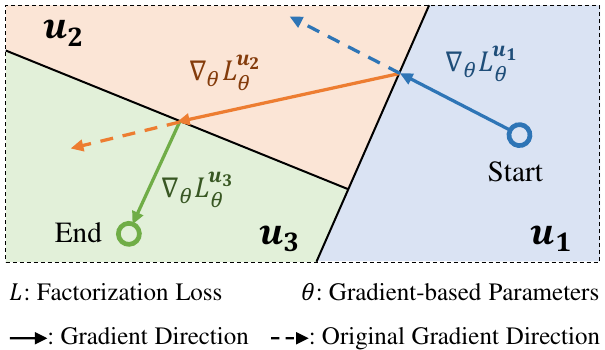}
    \caption{Gradient discontinuity induced by crossing different values of greedy action. The gradient is continuous within any single region. However, when the greedy action changes, the values of the gradient-free components change. Since these components appear in the gradient of gradient-based parameters (e.g., individual \(Q\)s), their sudden changes induce a discontinuity in the gradient.}
    \label{gradient discontinuity}

\end{figure}

\subsection{Transition of the Greedy Action}

To introduce the analysis of the greedy action's convergence, we first illustrate how the greedy action changes during training. We provide an example for WQMIX in Figure~\ref{img stable point}, where we observe that the transition of greedy actions is similar to that in a dynamical system \(x^{k+1}=g(x^{k})\): 

\begin{equation}
    \begin{aligned}
       x^{k+1}  = &g(x^{k}) \\
       \bar{\boldsymbol{u}} \coloneqq x^{k+1}, &\ \widetilde{\boldsymbol{u}} \coloneqq  x^{k}
    \end{aligned}
\end{equation}

where the current state \(\widetilde{\boldsymbol{u}}\) controls the transition, and the next state \(\bar{\boldsymbol{u}}\) is the result of the transition. Different factorization forms correspond to different functions \(g\), depending on how the loss function is defined and optimized. 



\begin{figure*}[h]
    \centering
    \includegraphics[scale=0.55]{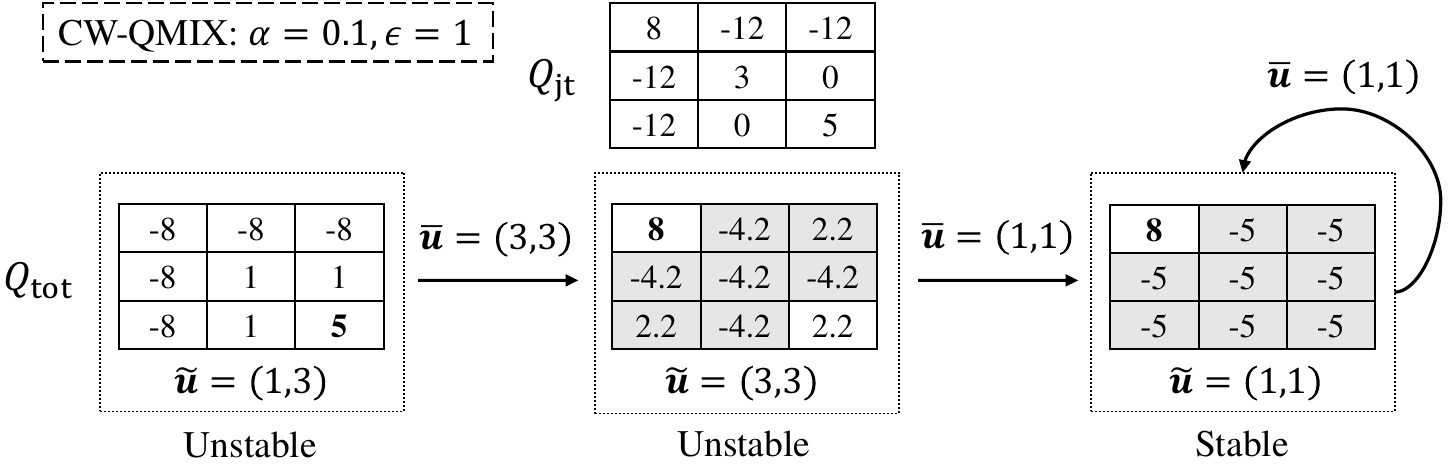}
    \caption{Transition of greedy action in WQMIX with \(\alpha=0.1\) under uniform visitation (\(\pi\equiv1\)). This process begins with \(\widetilde{\boldsymbol{u}}=(1,3)\) (top right) and stabilizes at \(\bar{\boldsymbol{u}}=(1,1)\) (top left). At each step we present \(Q_{\mathrm{tot}}\) optimized under the weight induced by current \(\widetilde{\boldsymbol{u}}\), where cells are shaded if  \(w(\boldsymbol{\tau},\boldsymbol{u})=\alpha\), and the value corresponding to \(\bar{\boldsymbol{u}}=(row, column)\) is in bold.}
    \label{img stable point}
\end{figure*}

\begin{table}[h!]
    \centering
    \begin{tabular}{ccc}
        \toprule
         & Dynamical System  & Value Factorization   \\
        \midrule
        \(\widetilde{\boldsymbol{u}}\) & \makecell{Current state of\\greedy action}  & \makecell{A formal parameter of\\gradient-free components}  \\
        \midrule
        \(\bar{\boldsymbol{u}}\) &\makecell{Next state of\\greedy action}   & \makecell{Individual-max action\\defined in Equation~(\ref{definition of greedy action})}   \\
        \midrule
        \(g\) & System function  & \makecell{Gradient-based optimization \\ of factorization parameters}  \\
        \bottomrule
    \end{tabular}
    \caption{The element-wise correspondence between dynamical system and value factorization.}
    \label{difference of greedy actions}
\end{table}

We illustrate the transition in Figure~\ref{img stable point} as follows:
\begin{enumerate}
    \item Initially, \(\widetilde{\boldsymbol{u}}=(1,3)\) makes an all-ones weight matrix.
    \item Minimizing \(L_\mathrm{tot}=w*(Q_\mathrm{tot}-Q_\mathrm{jt})^2\) under the all-ones weight \(w\) yields the \(Q_\mathrm{tot}\) (left table), corresponding to \(\bar{\boldsymbol{u}}=(3,3)\).
    \item \(\widetilde{\boldsymbol{u}}\) is assigned the new value \(\bar{\boldsymbol{u}}=(3,3)\), thereby altering the weight matrix (shaded regions).
    \item Minimizing \(L_\mathrm{tot}\) with the weight under \(\widetilde{\boldsymbol{u}}=(3,3)\) produces a new \(Q_\mathrm{tot}\) (middle table).
    \item \(\bar{\boldsymbol{u}}\) keeps on changing (steps 2-4) until \(\bar{\boldsymbol{u}}=\widetilde{\boldsymbol{u}}\) is reached (right table).
\end{enumerate}



\subsection{Convergence of the Greedy Action}
From the discussion above, the transition of \(\bar{\boldsymbol{u}}\) typically begins in a transient
phase (steps 2-4) and ends in a steady phase (step 5) when self-transition occurs. If self-transition occurs, both \(\bar{\boldsymbol{u}}\) and \(\widetilde{\boldsymbol{u}}\) remain unchanged, which reflects the convergence of value factorization. \textbf{We use stable points to describe the self-transition of greedy actions}, defined as follows:

\begin{definition}[Stable Point]
    In a specific value factorization, for a joint action value \(Q_{\mathrm{jt}}\) and  \(\boldsymbol{\tau}\in \boldsymbol{\mathcal{T}}\), if \(\check{\boldsymbol{u}}\) is a stable point of \(Q_{\mathrm{jt}}(\boldsymbol{\tau},\cdot)\), then for \(\widetilde{\boldsymbol{u}}=\check{\boldsymbol{u}}\) in the optimization of minimizing \(L_{tot}\), the converged \(\bar{\boldsymbol{u}}\) satisfies \(\bar{\boldsymbol{u}}=\check{\boldsymbol{u}}\).
    \label{def stable point}
\end{definition}

Note that multiple stable points may exist in value factorization, any of which could be the final result of greedy actions. To analyze the possible outcomes of greedy actions, we need to determine whether an action \(\check{\boldsymbol{u}}\) is a stable point, which is described as follows.

\begin{enumerate}
    \item Set \(\widetilde{\boldsymbol{u}}=\check{\boldsymbol{u}}\) and generate the corresponding gradient-free components.
    \item Minimize \(L_\mathrm{tot}\) under \(\widetilde{\boldsymbol{u}}=\check{\boldsymbol{u}}\) to obtain \(Q_\mathrm{tot}\) (multiple solutions may exist) and corresponding \(\bar{\boldsymbol{u}}\).
    \item If any \(\bar{\boldsymbol{u}}\) satisfies \(\bar{\boldsymbol{u}}=\check{\boldsymbol{u}}\), then \(\check{\boldsymbol{u}}\) is a stable point.
    \item Furthermore, if \(\bar{\boldsymbol{u}}=\check{\boldsymbol{u}}\) holds for all solutions, \(\check{\boldsymbol{u}}\) is a strongly stable point. Otherwise, \(\check{\boldsymbol{u}}\) is a weakly stable point.
\end{enumerate}

The difference between mixer types lies in optimization of the loss function (Step 2): For VDN-style mixers, we obtain \(Q_\mathrm{tot}\) using gradient descent. And for QMIX-style mixers, we obtain \(Q_\mathrm{tot}\) based on the Ideal QMIX assumption: The Ideal QMIX yields the optimal \(Q_\mathrm{tot}\) that minimizes \(L_\mathrm{tot}\) under any given \(\widetilde{\boldsymbol{u}}\) (Appendix~\ref{section: ideal qmix}).

\subsection{Suboptimality of Existing Value Factorization}
\label{Subsec:suboptimality of single-round}
We analyze stable points in existing methods to demonstrate cases where greedy actions converge to suboptimal actions, which usually happens in environments with highly non-monotonic payoff. Then, we derive the conditions under which the greedy action converges solely to the optimal action.

\textbf{WQMIX:} WQMIX can not guarantee the convergence to the optimal action. In some cases, as shown in Table~\ref{CW-QMIX fail example}, the optimal action is not a stable point for any \(\alpha\), which means that the greedy action does not converge to the optimal action.

\begin{table}[t]
    \centering
    \begin{subtable}[h]{0.5\linewidth}
        \centering
        \begin{tabular}{|c|c|c|}
            \multicolumn{3}{c}{\ } \\
            \hline
            4 & 0  & -8     \\
            \hline
            0     & 3 & 0      \\
            \hline
            -8     & 0       & -8  \\
            \hline
        \end{tabular}
        \caption{\(Q_{\mathrm{jt}}\)}
    \end{subtable}%
    \begin{subtable}[h]{0.5\linewidth}
        \centering
        \begin{tabular}{|c|c|c|}
            \multicolumn{3}{c}{\ } \\
            \hline
            \(1\) & \(\alpha\)  & \(\alpha\)     \\
            \hline
            \(\alpha\)     & \(\alpha\) & \(\alpha\)      \\
            \hline
            \(\alpha\)     & \(\alpha\)       & \(\alpha\)  \\
            \hline
        \end{tabular}
        \caption{\(w\)}
    \end{subtable}%
    \\
    \begin{subtable}[h]{0.5\linewidth}
        \centering
        \begin{tabular}{|c|c|c|}
            \multicolumn{3}{c}{\(L_\mathrm{tot}\)=70\(\alpha\)}    \\
            \hline
            \textbf{4} & 1  & -4     \\
            \hline
            1     & 1 & -4      \\
            \hline
            -4     & -4       & -8  \\
            \hline
        \end{tabular}
        \caption{\(Q_{\mathrm{tot}}^*\) under \(\bar{\boldsymbol{u}}=\boldsymbol{u}^*\)}
    \end{subtable}%
    \begin{subtable}[h]{0.5\linewidth}
        \centering
        \begin{tabular}{|c|c|c|}
            \multicolumn{3}{c}{\(L_\mathrm{tot}\)=33\(\alpha\)}    \\
            \hline
            4 & 4  & -8     \\
            \hline
            4     & \textbf{4} & 0      \\
            \hline
            -8     & 0       & -8  \\
            \hline
        \end{tabular}
        \caption{\(Q_{\mathrm{tot}}\) with \(\bar{\boldsymbol{u}}\neq\boldsymbol{u}^*\)}
    \end{subtable}
    \caption{An example for WQMIX in which \(\boldsymbol{u}^*\) is not a stable point. (a): The matrix of \(Q_{\mathrm{jt}}\) with the maximum value of \(4\). (b): The weight \(w\) of WQMIX where \(\widetilde{\boldsymbol{u}}=\boldsymbol{u}^*\). (c): The \(Q_{\mathrm{tot}}\) with minimum \(L_\mathrm{tot}\) under the condition that \(\bar{\boldsymbol{u}}=\boldsymbol{u}^*\). (d): Another \(Q_{\mathrm{tot}}\) where \(\bar{\boldsymbol{u}}\neq\boldsymbol{u}^*\) achieves less \(L_\mathrm{tot}\) than (c). Therefore, \(\boldsymbol{u}^*\) is not a stable point in this case.}
    \label{CW-QMIX fail example}
\end{table}

\textbf{QPLEX:} QPLEX proves that if \(\widetilde{\boldsymbol{u}}\) happens to be the optimal action, then \(\bar{\boldsymbol{u}}\) can converge to the optimal action. However, we find that multiple suboptimal stable points exist in QPLEX, which means that if \(\widetilde{\boldsymbol{u}}\) is certain suboptimal actions, \(\bar{\boldsymbol{u}}\) will converge to them (examples in Appendix~\ref{stability analysis of QPLEX}).

\textbf{ResQ:} ResQ proves that \(\bar{\boldsymbol{u}}\) can converge to the optimal action when \(\widetilde{\boldsymbol{u}}\) is the optimal action (similar to QPLEX). However, we find that although the optimal action is the only stable point, it is always a weak stable point (proof in Appendix~\ref{stablility analysis of resq}). This means that even if the optimal action has been found (\(\widetilde{\boldsymbol{u}}=\boldsymbol{u}^*\)), it may be lost subsequently.

From the above discussion, we conclude that in single-round value factorization, if the greedy action is to be guaranteed to converge to the optimal action, then the optimal action must be the unique strongly stable point, which means:
\begin{enumerate}
    \item The optimal action must be a strong stable point (WQMIX and ResQ fail)
    \item All suboptimal actions must be unstable points (QPLEX fails).
\end{enumerate}

\section{Breakthrough Suboptimality with Multi-Round Value Factorization}
\label{Sec:method}

In Section~\ref{Subsec:suboptimality of single-round}, we provide the condition under which the greedy action converges exclusively to the optimal action. This condition is highly stringent, to the extent that existing single-round factorization methods struggle to satisfy them. Therefore, for achieving optimality, we propose a multi-round factorization framework with a lenient condition --- the \textbf{Strict Improvement Condition}: if the preceding greedy action is not optimal, the current greedy action must attain a strictly higher payoff than the preceding one. Satisfying this condition guarantees that the greedy action in a certain round is the optimal one (Theorem~\ref{approximate convergence}).

\begin{definition}[Strict Improvement Condition]
    Given a joint action value \(Q_{\mathrm{jt}}\) and \(\boldsymbol{\tau}\in \boldsymbol{\mathcal{T}}\), let  \(\{\bar{\boldsymbol{u}}^k|\bar{\boldsymbol{u}}^k\in\boldsymbol{\mathcal{U}}\}\) be an infinite sequence of greedy actions. The Strict Improvement Condition is satisfied if, \(\forall k>1\) and \(\bar{\boldsymbol{u}}^{k-1}\neq\boldsymbol{u}^*\), we have \(Q_{\mathrm{jt}} (\boldsymbol{\tau},\bar{\boldsymbol{u}}^k)> Q_{\mathrm{jt}}(\boldsymbol{\tau},\bar{\boldsymbol{u}}^{k-1})\).
    \label{def strict improvement Condition}
\end{definition}

\begin{restatable}{theorem}{mainthmone}
    \label{approximate convergence}
    Given a joint action value \(Q_{\mathrm{jt}}\) with a finite action space \(\boldsymbol{\mathcal{U}}\) and \(\boldsymbol{\tau}\in \boldsymbol{\mathcal{T}}\), for any infinite sequence \(\{\bar{\boldsymbol{u}}^k|\bar{\boldsymbol{u}}^k\in\boldsymbol{\mathcal{U}}\}\) that satisfies the Strict Improvement Condition, then \(\exists K>0, \bar{\boldsymbol{u}}^K=\boldsymbol{u}^*\). 
\end{restatable}

\subsection{Designs for the Strict Improvement Condition}

We design the loss function of \(Q_{\mathrm{tot}}\) to derive, through backward computation, individual Q-values that meet the condition. Typically, \(\hat{Q}_{\mathrm{jt}}\), a payoff measurement, serves as the target value of \(Q_{\mathrm{tot}}\). In this paper, we replace it with the action-value increment for round \(k>1\). As shown in Equation~(\ref{tot loss round>1}), for each step \(t\), this increment is computed as the difference between \(\hat{Q}_{\mathrm{jt}}\) and its value at the preceding greedy action \(\bar{\boldsymbol{u}}_{t}^{k-1}\), clipped to a minimum of zero.

\vspace{-0.5cm}

\begin{equation}
    \begin{aligned}
        L_{\mathrm{tot}} & = 
        \mathrm{E}_{\boldsymbol{\tau}_t,\boldsymbol{u}_{t}^{k} \sim \pi_k}[(Q_{\mathrm{tot}}(\boldsymbol{\tau}_{t}, \bar{\boldsymbol{u}}_{t}^{k-1},\boldsymbol{u}_{t}^{k}) 
        \\
        & - \max\{\hat{Q}_{\mathrm{jt}}(\boldsymbol{\tau}_t,\boldsymbol{u}_t^{k})-\hat{Q}_{\mathrm{jt}}(\boldsymbol{\tau}_t,\bar{\boldsymbol{u}}_{t}^{k-1}),0 \})^2], \
        k>1 
    \end{aligned}
    \label{tot loss round>1}
\end{equation}
where \(\max\{\hat{Q}_{\mathrm{jt}}(\boldsymbol{\tau}_t,\boldsymbol{u}_t^{k})-\hat{Q}_{\mathrm{jt}}(\boldsymbol{\tau}_t,\bar{\boldsymbol{u}}_{t}^{k-1}),0 \}\) is the action-value increment.

We illustrate in Table~\ref{target of Qtot} how the target value of \(Q_{\mathrm{tot}}\) shapes during the multi-round process. For the first round, we use the raw \(\hat{Q}_{\mathrm{jt}}\) as the target for \(Q_{\mathrm{tot}}\), as shown in Equation~(\ref{tot loss round 1}).

\vspace{-0.5cm}

\begin{equation}
    L_{\mathrm{tot}} = 
        \mathrm{E}_{\boldsymbol{\tau}_t,\boldsymbol{u}_{t}^{k} \sim \pi_k}[(Q_{\mathrm{tot}}(\boldsymbol{\tau}_{t}, \boldsymbol{u}_{t}^{k}) -\hat{Q}_{\mathrm{jt}}(\boldsymbol{\tau}_t,\boldsymbol{u}_t^{k}) )^2], \ 
        k=1 
    \label{tot loss round 1}
\end{equation}

\begin{table}[h]
  \centering
  \begin{subtable}[h]{0.35\linewidth}
  \begin{tabular}{|c|c|c|}
    \hline
    8 & -12  & -12     \\
    \hline
    -12     & 3 & 0      \\
    \hline
    -12     & 0       & \textbf{5}  \\
    \hline
  \end{tabular}
    \caption{\(k=1\)}
    \end{subtable}   
    \begin{subtable}[h]{0.31\linewidth}
      \begin{tabular}{|c|c|c|}
        \hline
        \  \textbf{3} \  & \ 0 \  & \  0  \     \\
        \hline
        0     & 0 & 0      \\
        \hline
        0     & 0       & 0  \\
        \hline
      \end{tabular}
      \caption{\(k=2\)}
     \end{subtable}
    \begin{subtable}[h]{0.31\linewidth}
      \begin{tabular}{|c|c|c|}
        \hline
        \  0 \   &\   0  \  &\   0 \      \\
        \hline
        0     & 0 & 0      \\
        \hline
        0     & 0       & 0  \\
        \hline
      \end{tabular}
      \caption{\(k=3\)}
    \end{subtable}
    \hfill
  \caption{This table shows the target of \(Q_{\mathrm{tot}}\) in three round, in which the value corresponding to \(\bar{\boldsymbol{u}}\) is in bold when it is unique. (a): The target of \(Q_{\mathrm{tot}}\) in the first round which is the original \(\hat{Q}_{\mathrm{jt}}\). (b): The target of \(Q_{\mathrm{tot}}\) in the second round which is clipped with \(\hat{Q}_{\mathrm{jt}}(\boldsymbol{\tau}_t,\bar{\boldsymbol{u}}_{t}^{1})=5\). (c): The target of \(Q_{\mathrm{tot}}\) in the final round where any \(\boldsymbol{u}\) can be \(\bar{\boldsymbol{u}}^{3}\).}
  \label{target of Qtot}
\end{table}

Intuitively, as shown in Table~\ref{target of Qtot}, the action-value increment enhances the monotonicity of \(\hat{Q}_{\mathrm{jt}}\): It suppresses the locally high payoffs of inferior actions and, conversely, offsets the low payoffs around superior actions. This reshaping inclines monotonic factorization toward superior actions. \textbf{Theoretically, it ensures the Strict Improvement Condition when} \(\bar{\boldsymbol{u}}_t^{k-1}\neq \boldsymbol{u}^*\). This is grounded in a property of monotonic factorization: the greedy action in QMIX will not converge to actions with the lowest target value of \(Q_{\mathrm{tot}}\), since these lowest-value actions are not stable points (Theorem~\ref{centralized policy avoid the worst}). With the action-value increment, all inferior actions \(\boldsymbol{u}_t^{-}\) with \(\hat{Q}_{\mathrm{jt}}(\boldsymbol{\tau}_t,\boldsymbol{u}_t^{-})\leq \hat{Q}_{\mathrm{jt}}(\boldsymbol{\tau}_t,\bar{\boldsymbol{u}}_{t}^{k-1})\) are assigned zero, the minimum value. Meanwhile, all superior actions \(\boldsymbol{u}_t^{+}\) with \(\hat{Q}_{\mathrm{jt}}(\boldsymbol{\tau}_t,\boldsymbol{u}_t^{+}) > \hat{Q}_{\mathrm{jt}}(\boldsymbol{\tau}_t,\bar{\boldsymbol{u}}_{t}^{k-1})\) have positive values. Therefore, \(\bar{\boldsymbol{u}}_{t}^{k}\) only converges to actions that are superior to \(\bar{\boldsymbol{u}}_{t}^{k-1}\).


\begin{restatable}{theorem}{mainthmtwo}
    \label{centralized policy avoid the worst}
    Given a non-constant joint action value \(Q_{\mathrm{jt}}(\boldsymbol{\tau},\cdot)\) (\(\forall \boldsymbol{\tau} \in \boldsymbol{\mathcal{T}}, \forall C\in \mathbb{R}, Q_{\mathrm{jt}}(\boldsymbol{\tau},\cdot)\not\equiv C\)) with finite action-observation history space \(\boldsymbol{\mathcal{T}}\) and action space \(\boldsymbol{\mathcal{U}}\). For the ideal QMIX and the centralized \(\epsilon\)-greedy policy \(\pi\) defined in Equation~(\ref{centralized epsilon greedy}), we have \(\exists E\in(0,1), \forall \epsilon\in(0,E), \forall \boldsymbol{\tau} \in \boldsymbol{\mathcal{T}}, \forall \boldsymbol{u} \in \arg\min_{\boldsymbol{u}}Q_{\mathrm{jt}}(\boldsymbol{\tau},\boldsymbol{u})\), \(\boldsymbol{u}\) is not a stable point.
\end{restatable}

The forward computation in step \(t\) generates the greedy final action \(\bar{\boldsymbol{u}}_{t}\) during evalutaion (or the final action \(\boldsymbol{u}_t\) for exploration during training) that actually interacts with the environment. However, the greedy action in the final round is not always optimal. As shown in Table~\ref{target of Qtot}, since the greedy action in the round \(k=2\) is optimal, the action-value increment for all actions in the next round is zero. As a result, greedy action in the round \(k=3\) can converge to an arbitrary action, which could lead to a deterioration of \(\bar{\boldsymbol{u}}_{t}\). Therefore, when we need to output \(\bar{\boldsymbol{u}}_{t}\), \textbf{early termination is necessary once the optimal action is obtained}. To verify whether the optimal action has been found, we use the Strict Improvement Condition to check for a payoff increase between consecutive greedy actions. Specifically, when \(\hat{Q}_{\mathrm{jt}}(\boldsymbol{\tau}_t,\bar{\boldsymbol{u}}_t^{k}) \leq \hat{Q}_{\mathrm{jt}}(\boldsymbol{\tau}_t,\bar{\boldsymbol{u}}_{t}^{k-1})\), the forward process in step \(t\) is terminated with an output \(\bar{\boldsymbol{u}}_{t}^{k-1}\), otherwise the process proceed to the next round.


\subsection{Overview of the Architecture}
\begin{figure*}[h!]
    \centering
    \includegraphics[scale=0.55]{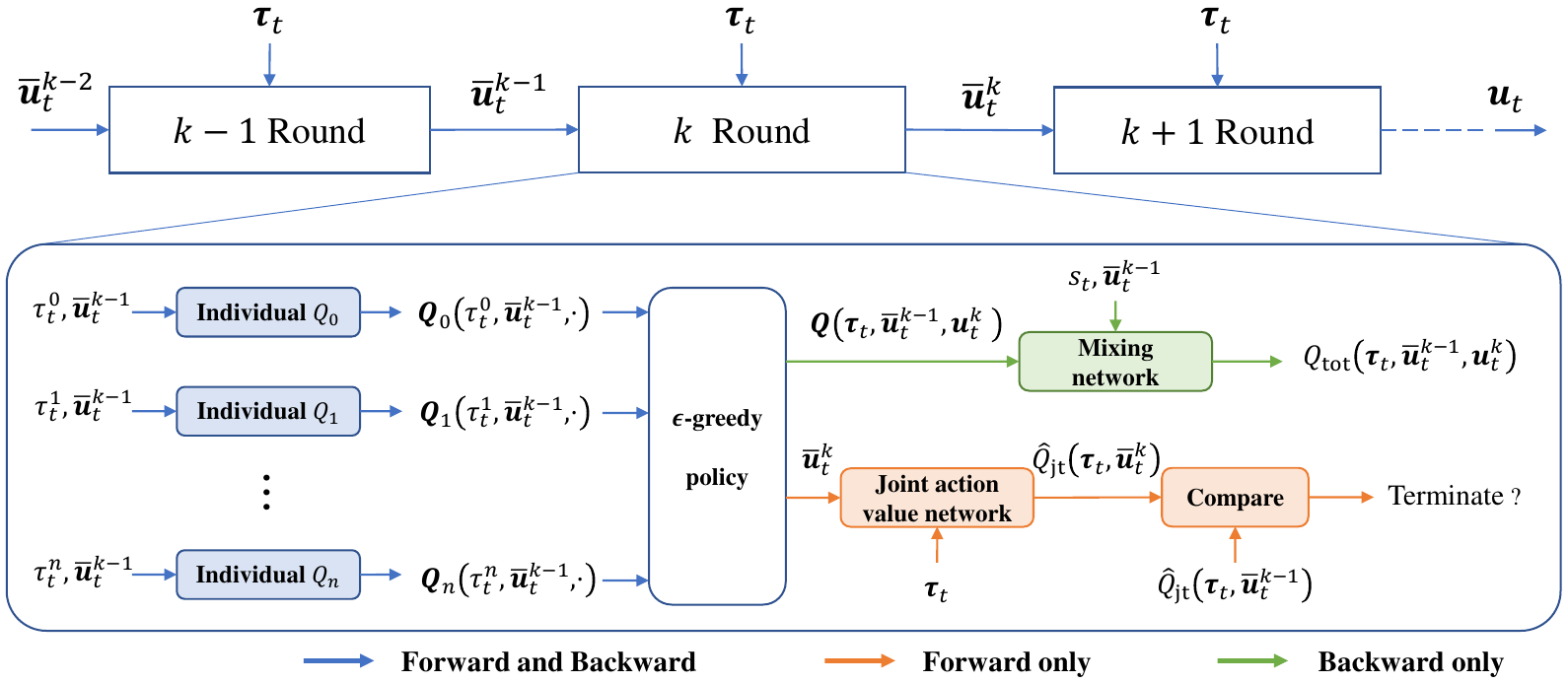}  
    \caption{The architecture of multi-round value factorization framework.}
    \label{img architecture of mrvd}
\end{figure*}

The architecture of our methods is shown in Figure~\ref{img architecture of mrvd}. The architecture consists of three crucial parts: individual Q networks, a mixing network with positive weights, and a joint action value network. Individual Q networks receive the action-observation history \(\boldsymbol{\tau}_t\) and the greedy action of the preceding round \(\bar{\boldsymbol{u}}_{t}^{k-1}\) (default action for the first round). Then, they output the actions of the current round \(\boldsymbol{u}_{t}^{k}\) through the \(\epsilon\)-greedy action selector. The mixing network receives individual Q values and outputs \(Q_{\mathrm{tot}}\). The joint action value network outputs \(\hat{Q}_{\mathrm{jt}}\) to approximate the real \(Q_{\mathrm{jt}}\). We update \(\hat{Q}_{\mathrm{jt}}\) by minimizing the temporal difference (TD) error in Equation~(\ref{jt loss}) (omitted in Figure~\ref{img architecture of mrvd}).
\begin{equation}
    L_{\mathrm{jt}} = \mathrm{E}_{\boldsymbol{\tau}_t, \boldsymbol{u}_t\sim\pi}[(\hat{Q}_{\mathrm{jt}}(\boldsymbol{\tau}_t,\boldsymbol{u}_t) - (r_t+\gamma\hat{Q}_{\mathrm{jt}}(\boldsymbol{\tau}_{t+1},\bar{\boldsymbol{u}}_{t+1})))^2]
    \label{jt loss}
\end{equation}
where \(\bar{\boldsymbol{u}}_t\) is the greedy final action, and \(\hat{Q}_{\mathrm{jt}}(\boldsymbol{\tau}_{t+1},\bar{\boldsymbol{u}}_{t+1})\) approximates \(\max_{\boldsymbol{u}_{t+1}}\hat{Q}_{\mathrm{jt}}(\boldsymbol{\tau}_{t+1},\boldsymbol{u}_{t+1})\). 

Within step \(t\), both the forward and backward processes take \(\boldsymbol{\tau}_{t}\) as input, and in each round, the individual Q networks are fed with \(\boldsymbol{\tau}_t\) and \(\bar{\boldsymbol{u}}_{t}^{k-1}\) to produce \(Q_i\) and \(\bar{\boldsymbol{u}}_{t}^{k}\). The difference is that the forward process computes and compares \(\hat{Q}_{\mathrm{jt}}\) to determine whether to terminate early, while the backward process does not involve early termination and computes \(\hat{Q}_{\mathrm{jt}}\) only once using the recorded \(\boldsymbol{u}_{t}\) from the batch (not shown in Figure~\ref{img architecture of mrvd}). In addition, \(Q_{\mathrm{tot}}\) is computed only in the backward.

\subsection{Sampling}
We approximate the expectation in \(L_{\mathrm{jt}}\) and \(L_{\mathrm{tot}}\) by sampling the action spaces. We design sampling strategies for multi-round frameworks to ensure stable training. For the final action \(\boldsymbol{u}_t\) in \(L_{\mathrm{jt}}\), its sampling should not only cover the greedy final action and random actions, but also cover the greedy action in each round, which obtains an accurate target value in Equation~(\ref{tot loss round>1}). Therefore, during training, we disable early termination and select \(\bar{\boldsymbol{u}}_{t}^{k}\) in a certain round \(k\) as the final action with probability \(p\). For \(\boldsymbol{u}_{t}^{k}\) in \(L_{\mathrm{tot}}\), its distribution \(\pi_k\) is the centralized \(\epsilon\)-greedy policy defined in Equation~(\ref{centralized epsilon greedy}) with \(\widetilde{\boldsymbol{u}}=\bar{\boldsymbol{u}}^k\).
\begin{equation}
    \pi(\boldsymbol{u}|\boldsymbol{\tau})=
    \begin{cases}
        1-\epsilon+\frac{\epsilon}{|\boldsymbol{\mathcal{U}}|}  &   \boldsymbol{u}=\widetilde{\boldsymbol{u}}   \\
        \frac{\epsilon}{|\boldsymbol{\mathcal{U}}|}    &   \boldsymbol{u} \neq \widetilde{\boldsymbol{u}}
    \end{cases}
    \label{centralized epsilon greedy}
\end{equation}
The reason for not using the decentralized one (defined in Equation~(\ref{decentralized epsilon greedy})) is that, according to Theorem~\ref{centralized policy avoid the worst}, the centralized \(\epsilon\)-greedy policy is required to guarantee the Strict Improvement Condition. More importantly, since \(Q_{\mathrm{tot}}\) learns the actual payoff from \(\hat{Q}_{\mathrm{jt}}\), the centralized \(\epsilon\)-greedy policy allows the same actions to be sampled in \(L_{\mathrm{jt}}\) and \(L_{\mathrm{tot}}\) when sampling random actions, which prevents \(Q_{\mathrm{tot}}\) from learning regions where \(\hat{Q}_{\mathrm{jt}}\) exhibits underfitting.

\begin{figure*}[t]
    \centering
    \includegraphics[scale=0.28]{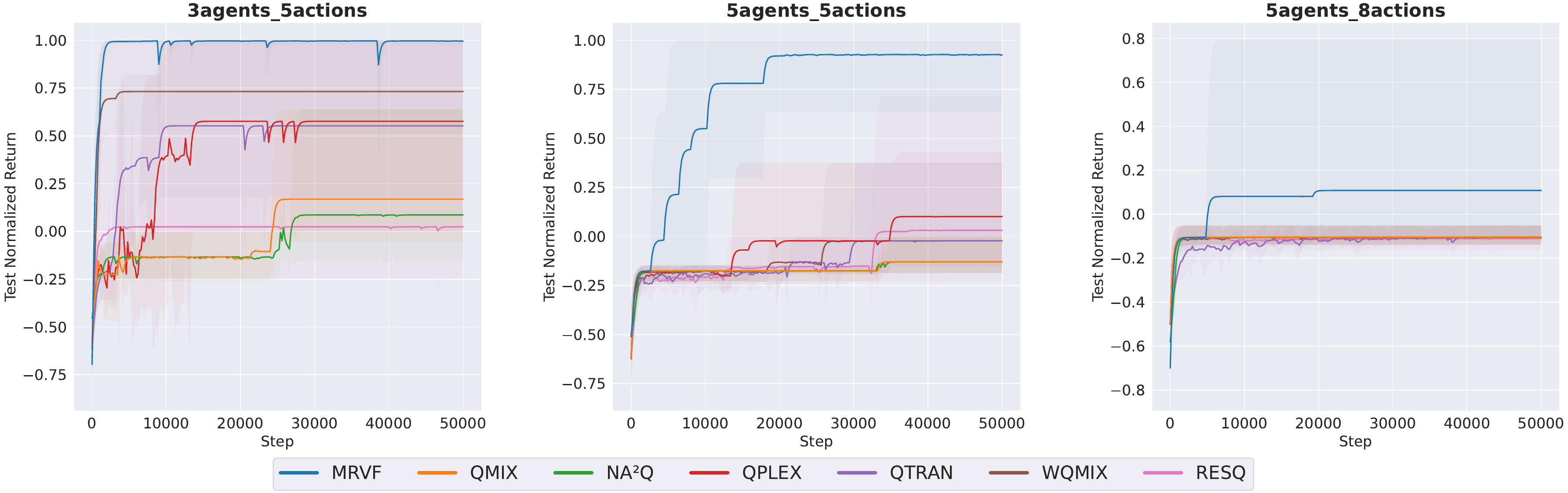}    
    \caption{Test normalized return in the risk-reward games. The positive returns are normalized to \([0,1]\) (\(0\) corresponds to the smallest positive return, and \(1\) corresponds to the largest). The negative returns are normalized to \([-1,0)\). Five random cases for each setting.}
    \label{risk result}
\end{figure*}

\begin{figure*}[t]
    \centering
    \includegraphics[scale=0.28]{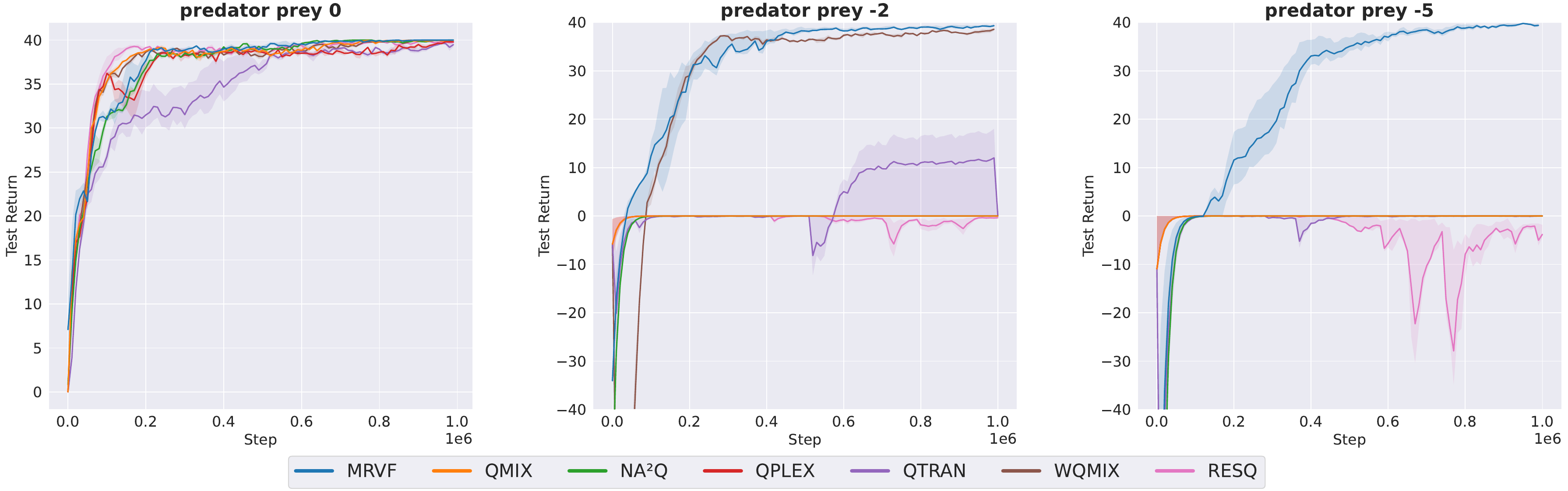}   
    \caption{Test return in the predator prey tasks with punishments 0 (left), -2 (middle), and -5 (right). The non-monotonicity of the payoff increases with the punishment.}
    \label{pred result}
\end{figure*}

\section{Experiment}
\label{experiment}
In this section, we evaluate the performance in one-step games, the predator-prey task~\citep{Ref:DCG} with increasing punishment, and a variety of SMAC~\citep{Ref:smac} scenarios, among which one-step games and predator-prey tasks with large punishment are environments with highly non-monotonic payoff (defined in Appendix~\ref{sec: def monotonic payoff}). In addition, experiments in SMACv2~\citep{Ref:smacv2} are included in Appendix~\ref{exp in SMACv2}, since this paper studies the impact of payoff structure rather than stochastic factors on value factorization. We report the median performance, with the shaded area denoting the 25\%-75\% percentile (0\%-100\% percentile for one-step games). 

We use the methods discussed in Section~\ref{Sec:theory} as our primary baseline. These algorithms are closely related to MRVF in that all aim to improve value factorization under non-monotonic payoffs. The results for other baselines (e.g., policy-based methods including MAPPO~\citep{Ref:mappo} and other recent value factorization methods) are presented in Appendix~\ref{exp on other baselines}. Our experimental setup is provided in Appendix~\ref{experimental setup}, and ablation results are presented in Appendix~\ref{ablation}.

\subsection{One-Step Game}

In Section~\ref{Subsec:suboptimality of single-round}, we observe that existing methods struggle to achieve optimal solutions in highly non-monotonic payoff matrices. However, benchmarks in matrix form~\citep{Ref:qplex, Ref:resq} suffer from limitations in both scale and generality. To address this, we introduce the \textit{risk-reward game}, which can randomly generate payoffs in the form of high-dimensional tensors. In a \textit{risk-reward game}, only \(|\mathcal{U}|\) out of  \(|\mathcal{U}|^n\) entries are positive in the payoff tensor (see Appendix~\ref{one-step game setup} for details). Therefore, agents must reach consensus to obtain positive rewards; otherwise, even a deviation by an individual agent yields large-magnitude negative rewards.


\begin{table*}[t]
    \centering
    \begin{tabular}{cccccccc}
        \toprule
        Methods & \textbf{MRVF}  & QMIX & NA\textsuperscript{2}Q & QPLEX & QTRAN & WQMIX & RESQ  \\
        \midrule
        Best & \textbf{6}  & 3 & 4 & 1 & 1 & 1 & 3  \\
        Worst & \textbf{0}  & 1 & 0 & 3 & 5 & 2 & 1  \\
        \bottomrule
    \end{tabular}
    \vspace{0.2cm}
    \caption{The number of SMAC scenarios where each method performs the best and the worst. We evaluate the final performance (averaged over steps after 1.5M) with a tolerance of \(\pm 0.05\).}
    \label{smac result table}
\end{table*}

\begin{figure*}[t]
    \centering
    \includegraphics[scale=0.28]{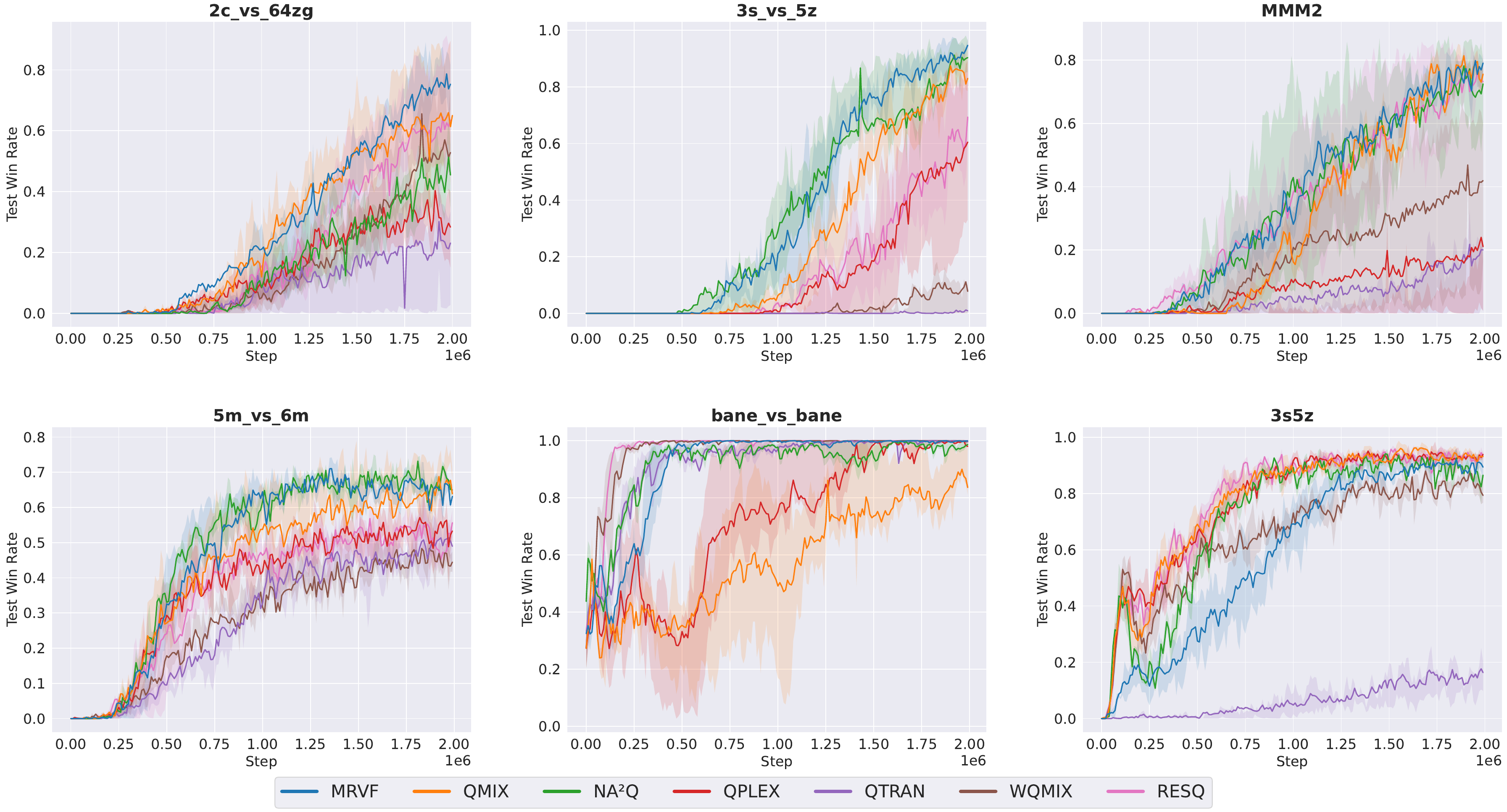}    
    \caption{Test win rate in the SMAC benchmarks.}
    \label{smac result}
\end{figure*}

As shown in Figure~\ref{risk result}, in cases with 3 agents and 5 actions, monotonic factorizations (QMIX~\citep{Ref:qmix} and NA\textsuperscript{2}Q~\citep{Ref:na2q}) obtain the smallest positive return, nearly zero, while QTRAN~\citep{Ref:qtran} and WQMIX~\citep{Ref:wqmix} only achieve a moderate positive return. In contrast, our method consistently attains the optimal return. When the number of agents increases to 5, the risk of obtaining positive returns increases, since consensus requires coordination among more agents. Existing methods rarely obtain positive returns, while our method still obtains the optimum with high probability. However, in cases with 5 agents and 8 actions, the positive rewards become too sparse to obtain. Despite this challenge, our method still obtains better results than others.

\subsection{Predator Prey}
We conduct experiments on the predator-prey task involving 8 agents, where capturing a prey requires cooperation between at least two agents. Agents are punished for capturing the prey alone. We vary the punishment from 0 to -5 to evaluate its impact on performance. The non-monotonic property of the payoff increases with the punishment, since agents must act more cautiously when choosing the "capture" action to ensure consensus and avoid punishments.

The results in Figure~\ref{pred result} show that all methods perform well when the punishment is zero. However, as the punishment intensifies, existing methods avoid the "capture" action to prevent punishments, resulting in \(return\approx0\). Among baseline approaches, WQMIX performs adequately only under moderate punishments, while the performance of ResQ shows significant fluctuations due to its weak stability. In contrast, our method consistently achieves the best performance, even under the strongest punishment of -5. Combined with the results in \textit{risk-reward} games, we conclude that our method outperforms existing methods in environments with highly non-monotonic payoff.

\subsection{StarCraft II Multi-Agent Challenge}

We conduct experiments on the SMAC benchmark in scenarios of varying difficulty. Monotonic value factorization methods, including QMIX~\citep{Ref:qmix} and NA\textsuperscript{2}Q~\citep{Ref:na2q}, are particularly well-suited for the SMAC because changes in an individual agent’s action within a single step have little impact on overall performance (see Appendix~\ref{smac setup} for details). Nevertheless, as shown in Figure~\ref{smac result}, MRVF even achieves an advantage over monotonic value factorization methods, particularly in \textit{3s\_vs\_5z},  \textit{2c\_vs\_64zg} and \textit{bane\_vs\_bane} scenarios. Meanwhile, the final performance of MRVF is the best in most scenarios, as shown in Table~\ref{smac result table}. Combined with the results in the \textit{risk-reward} games and the predator prey task, we conclude that MRVF performs the best in both monotonic and non-monotonic scenarios.


In contrast, other methods (QPLEX~\citep{Ref:qplex}, ResQ~\citep{Ref:resq}), QTRAN~\citep{Ref:qtran} and WQMIX~\citep{Ref:wqmix}) often converge to suboptimal stable points. This leads to a high variance in performance within a scenario and across scenarios. Within the same scenario, although they occasionally achieve strong performance (as seen in the upper bounds of the shaded regions in Figure~\ref{smac result}), their average performance lags behind. In addition, across different scenarios, their performance may vary drastically. For example, WQMIX performs well in \textit{bane\_vs\_bane} but poorly in both \textit{5m\_vs\_6m} and \textit{3s\_vs\_5z}. Therefore, MRVF achieves significant improvements in robustness and performance over them in almost all scenarios.

\section{Conclusion}
In this paper, we introduce a novel theoretical tool for studying the convergence of greedy action under value factorization, and propose MRVF, a novel framework for cooperative multi-agent reinforcement learning. In Section~\ref{Sec:theory}, we use this theoretical tool to derive the condition for global optimality in single-round factorization, and provide examples to demonstrate why existing methods struggle to satisfy this condition. In Section~\ref{Sec:method}, we propose the condition for global optimality in multi-round factorization, which is strictly improving the greedy action round by round. To satisfy this condition, we design the forward and backward computation of MRVF. In addition, we design new sampling strategies suitable for multi-round factorization to ensure training stability. Experiments on non-monotonic one-step games, predator-prey tasks, and the StarCraft II Multi-Agent Challenge show the superior performance and robustness of MRVF compared to state-of-the-art value factorization methods. 

\newpage        
\bibliography{mrvf}
\bibliographystyle{icml2026}

\newpage
\appendix
\onecolumn
\section{Relationship to Existing Work}

\paragraph{Relationship to Weighted QMIX} Both Weighted QMIX~\citep{Ref:wqmix} and the single-round value factorization in our method aim to enhance convergence toward actions superior to a given action. Weighted QMIX achieves this by reducing the fitting weight for inferior actions, whereas our method assigns them the minimum target value. However, Weighted QMIX cannot guarantee strict improvement over the given action, as suboptimal actions may be stable points. In contrast, our method ensures strict improvement by preventing actions with minimum target value from becoming stable points. Moreover, even when Weighted QMIX obtains the optimal action, it may fail to stabilize at this optimum because the optimal action might not be a stable point. Although our method similarly risks deviating from the optimal action, we mitigate this by early termination when no further improvement is detected.

\paragraph{Relationship to GVR} 

For the theoretical aspect, GVR~\citep{Ref:gvr} merely depicts the phenomenon of greedy-action transition caused by changes in the policy distribution, which can be regarded as a special case of the transition discussed in our paper. In contrast, we reveal that the essence of greedy-action transition lies in gradient-free components, and we model this transition with a discrete dynamical system. Unlike GVR, our theory serves as a theoretical paradigm for value factorization methods that involve gradient-free components. With this paradigm, we not only expose the suboptimality of prior work but also provide design guidelines for future factorization schemes. For the algorithmic aspect, the main idea of GVR~\citep{Ref:gvr} is to establish the optimal action as a stable point while rendering other actions non-stable. However, this approach requires prior knowledge of the optimal action, which is fundamentally infeasible in MARL since finding the optimal action is the primary objective. Consequently, GVR inevitably relies on approximations of ideal conditions, which not only increase the complexity but may also convert certain suboptimal actions into stable points. In contrast to GVR, our method converts inferior actions into unstable points in each iteration. Determining inferior actions is computationally straightforward, and by assigning them the minimum target value, existing monotonic value factorizations can effectively exclude these inferior actions. Therefore, our method strictly improves the current action through iterations, providing a more reliable approximation to the optimal action.

\paragraph{Relationship to Communication} 
Our method can also be interpreted as a multi-round communication algorithm within a parallel execution framework. In each iteration, agents exchange pre-decision information and generate new decisions through communication. The key challenge in such frameworks is ensuring that post-communication actions strictly improve upon pre-communication actions --- failure to achieve this prevents agents from reaching consensus, resulting in suboptimal solutions. In our method, by rendering actions inferior to pre-communication ones as non-stable points, our approach not only guarantees the improvement of post-communication actions but also approaches the optimal action within sufficient communication rounds. In contrast, sequential execution frameworks essentially extend single-step action selection across agents, where each agent's action depends on higher-priority agents' choices. Such frameworks not only demonstrate lower efficiency but also introduce additional concerns regarding policy convergence and suboptimality caused by execution ordering~\citep{Ref:seqcomm}.

\section{Definition}
\label{sec: definition}
In this section, we will define some important concepts that have been used in previous work but have not yet been rigorously defined.

\subsection{Monotonic Payoff}
\label{sec: def monotonic payoff}
Note that the definition of monotonicity for payoffs differs from that in functions of real variables, as there is no ordering relationship defined in the action space, the domain of the payoff. We define an order relation on the actions based on the function \(F\) as follows:

\begin{definition}[Order Relation \(\succeq\)]
    \label{def: order of actions}
    Let \(u_i^1, u_i^2\in\mathcal{U}_i\) be two actions of agent \(i\). For the function \(F\), if the order relation between \(u_i^1\) and  \(u_i^2\) is \(u_i^1\succeq u_i^2\), then
    \begin{equation}
        F(u_i^1,\boldsymbol{u}_{-i}) \geq F(u_i^2,\boldsymbol{u}_{-i}), \forall \boldsymbol{u}_{-i}\in\boldsymbol{\mathcal{U}}_{-i}
    \end{equation}
    where \(\boldsymbol{u}_{-i}=(u_1,\cdots,u_{i-1},u_{i+1},\cdots,u_n)\) is the joint action without \(u_i\).
\end{definition}

The function \(F\) can be replaced with \(Q_\mathrm{jt}\), a measurement of payoff. However, the order relation based on \(Q_\mathrm{jt}\) exists over the entire action space only if \(Q_\mathrm{jt}\) is monotonic, which is defined as follows:

\begin{definition}[Monotonic \(Q_\mathrm{jt}\)]
    \label{def: monotonic Qjt}
    \(\forall \boldsymbol{\tau}\in \boldsymbol{\mathcal{T}}\), if for every agent \(i\) and any two actions \(u_i^1, u_i^2\in\mathcal{U}_i\), there exists an order relation that is either \(u_i^1\succeq u_i^2\) or \(u_i^2\succeq u_i^1\), then \(Q_\mathrm{jt}\) is monotonic.
\end{definition}

\begin{figure*}[h]
    \centering
    \includegraphics[scale=0.5]{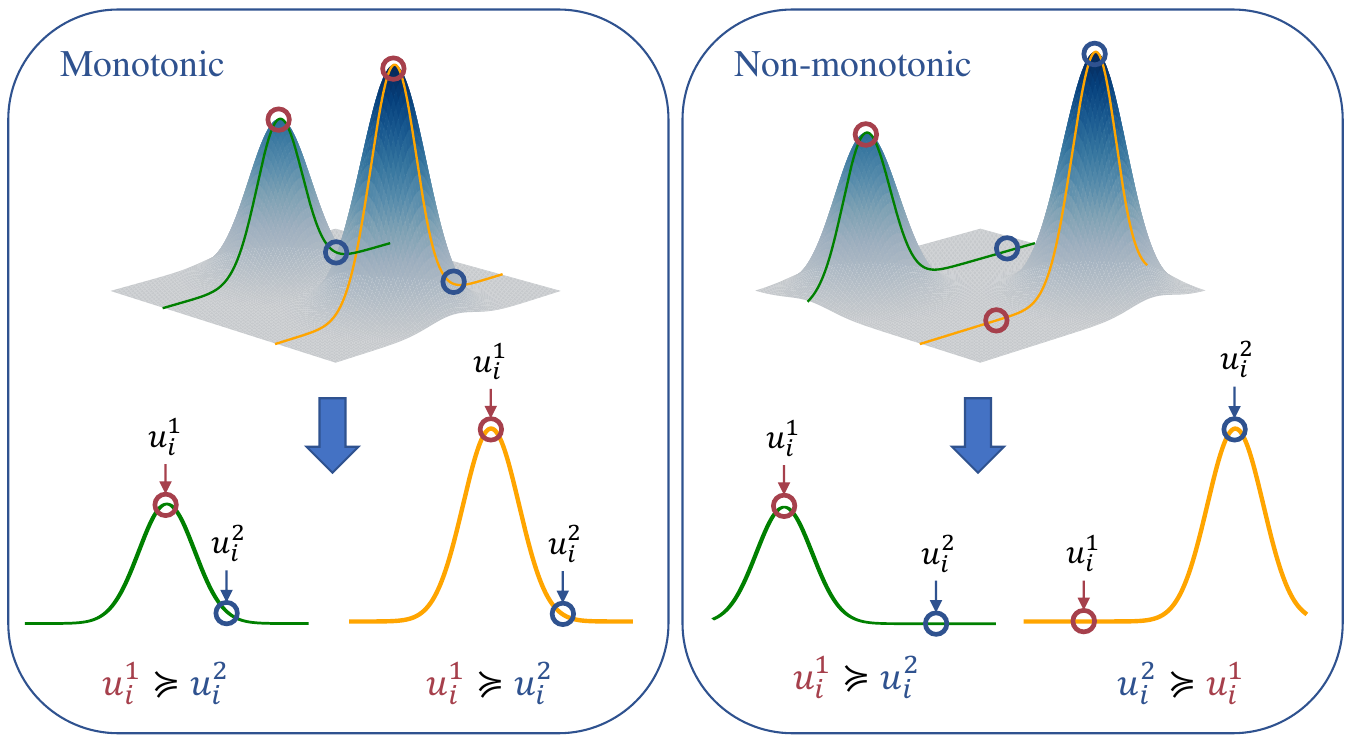}
    \caption{An illustration of Definition~\ref{def: monotonic Qjt}. We intercept curves corresponding to different \(\boldsymbol{u}_{-i}\) and check whether their monotonicity with respect to \(u_i\) is consistent.}
\end{figure*}

Table~\ref{monotonicity of payoff} shows the payoff matrices with different monotonicity. Monotonic factorization methods can easily obtain the optimal solution in monotonic cases ((a) and (b) in Table~\ref{monotonicity of payoff}) and sometimes in non-monotonic cases ((c) in Table~\ref{monotonicity of payoff}). However, monotonic factorizations struggle to obtain the optimal solution in highly non-monotonic cases ((d) in Table~\ref{monotonicity of payoff}).

\begin{table}[h]
    \centering
    \begin{subtable}[h]{0.22\linewidth}
        \centering
        \begin{tabular}{|c|c|c|}
            \hline
            9 & 8  & 7     \\
            \hline
            6     & 5 & 4      \\
            \hline
            3     & 2       & 1  \\
            \hline
        \end{tabular}
        \caption{Monotonic}
    \end{subtable}
    \hfill
    \begin{subtable}[h]{0.22\linewidth}
        \centering
        \begin{tabular}{|c|c|c|}
            \hline
            8 & 9  & 7     \\
            \hline
            2     & 3 & 1      \\
            \hline
            5     & 6       & 4  \\
            \hline
        \end{tabular}
        \caption{Monotonic}
    \end{subtable}%
    \hfill
    \begin{subtable}[h]{0.22\linewidth}
        \centering
        \begin{tabular}{|c|c|c|}
            \hline
            9 & 0  & 0     \\
            \hline
            0     & 5 & 0      \\
            \hline
            0     & 0       & 0  \\
            \hline
        \end{tabular}
        \caption{Non-monotonic}
    \end{subtable}%
    \hfill
    \begin{subtable}[h]{0.25\linewidth}
        \centering
        \begin{tabular}{|c|c|c|}
            \hline
            9 & -9  & -9     \\
            \hline
            -9     & 5 & 0      \\
            \hline
            -9     & 0       & 0  \\
            \hline
        \end{tabular}
        \caption{Highly non-monotonic}
    \end{subtable}
    \caption{Payoff matrices with different monotonicity, where the value in row \(r\) and column \(c\) denotes the payoff of joint action \(\boldsymbol{u}=(r,c)\). (a): A monotonic payoff matrix. (b) A monotonic payoff matrix generated by rearranging the rows and columns of (a). (c) A non-monotonic payoff matrix, where the relation between any two rows or columns is partially ordered. (d) A highly non-monotonic payoff matrix, where the relation between any two rows or columns is largely unordered.}
    \label{monotonicity of payoff}
\end{table}

\subsection{Ideal QMIX}
\label{section: ideal qmix}
The objective of QMIX is to find individual \(Q\)s and \(Q_{\mathrm{mon}}\) that minimize the MSE between \(Q_{\mathrm{tot}}\) and \(Q_{\mathrm{jt}}\). Considering a \(\boldsymbol{\tau}\in \boldsymbol{\mathcal{T}}\) in Equation~(\ref{equ: Ltot of value factorization}), the loss function is defined in Equation~(\ref{loss of QMIX}).
\begin{equation}
    L_{\mathrm{tot}}(\boldsymbol{\tau})=\sum_{\boldsymbol{u}}{\pi(\boldsymbol{u}|\boldsymbol{\tau})(Q_{\mathrm{mon}}(\boldsymbol{\tau},\boldsymbol{u})-Q_{\mathrm{jt}}(\boldsymbol{\tau},\boldsymbol{u}))^2}
    \label{loss of QMIX}
\end{equation}
Since the relationship between individual \(Q\)s and \(Q_{\mathrm{mon}}\) is monotonic, the magnitude relationship of \(Q_{\mathrm{mon}}\) is constrained by that of individual \(Q\)s, which is \(\forall \boldsymbol{\tau}\in \boldsymbol{\mathcal{T}}\) and \(\forall i\in\{1,2,\cdots,n\}, \forall u_i, v_i\in\mathcal{U}_i, \boldsymbol{u}_{-i}\in\boldsymbol{\mathcal{U}}_{-i}=\mathcal{U}_1\times\cdots\mathcal{U}_{i-1}\times\mathcal{U}_{i+1}\times\cdots\times\mathcal{U}_n\)
\begin{equation}
    (Q_{\mathrm{mon}}(\boldsymbol{\tau},\boldsymbol{u}_{-i},u_i)-Q_{\mathrm{mon}}(\boldsymbol{\tau},\boldsymbol{u}_{-i},v_i))(Q_i(\boldsymbol{\tau},u_i)-Q_i(\boldsymbol{\tau},v_i))\geq0
    \label{monotonic constrain}
\end{equation}
Combining the objective defined in Equation~(\ref{loss of QMIX}) and the constrain defined in Equation~(\ref{monotonic constrain}), we get the ideal optimization problem that QMIX solves. In practice, QMIX uses the network as a monotonic function \(f_{\mathrm{mon}}\) to represent the constraint between individual \(Q\)s and \(Q_{\mathrm{mon}}\), and applies stochastic gradient descent (SGD) to minimize the objective. In this way, \(Q_i\) is updated based on its gradient which involves \(\frac{\partial{f_{\mathrm{mon}}}}{\partial{Q_i}}\). However, since how \(\frac{\partial{f_{\mathrm{mon}}}}{\partial{Q_i}}\) varies is unknown, to avoid this confounding factor, we assume that QMIX can obtain the optimal combination of individual \(Q\)s and \(Q_{\mathrm{mon}}\) that satisfies Equation~(\ref{monotonic constrain}), and consequently minimize \(L_{\mathrm{tot}}\).
We name QMIX based on this assumption as the ideal QMIX which is defined as follows~\footnote{The policy \(\pi\) is omitted in the definition by ~\citet{Ref:wqmix}. However, it plays a crucial role in shaping \(Q_{\mathrm{tot}}\), which will be discussed in Lemma~\ref{Qtot* limit}.}:
\begin{definition}[Ideal QMIX]
    \label{definition ideal QMIX}
    Let \(\mathcal{Q}\) be the set of pairs \((\boldsymbol{Q},Q_{\mathrm{mon}})\), where the individual action values \(\boldsymbol{Q}=(Q_1,Q_2,\cdots,Q_n)\) and \(Q_{\mathrm{mon}}\) satisfy Equation~(\ref{monotonic constrain}). For a joint action value \(Q_{\mathrm{jt}}\) and policy \(\pi\), \(\boldsymbol{Q}\) and \(Q_{\mathrm{mon}}\) of the ideal QMIX converge to \(\boldsymbol{Q}^*\) and \(Q_{\mathrm{mon}}^*\) which satisfy \(\forall \boldsymbol{\tau}\in \boldsymbol{\mathcal{T}}\)
    \begin{equation}
        \boldsymbol{Q}^*,Q_{\mathrm{mon}}^*\in\mathop{\arg\min}\limits_{(\boldsymbol{Q},Q_{\mathrm{mon}})\in \mathcal{Q}}\sum_{\boldsymbol{u}}{\pi(\boldsymbol{u}|\boldsymbol{\tau})(Q_{\mathrm{mon}}(\boldsymbol{\tau},\boldsymbol{u})-Q_{\mathrm{jt}}(\boldsymbol{\tau},\boldsymbol{u}))^2}
        \label{ideal QMIX}
    \end{equation}
    where \(Q_{\mathrm{tot}}=Q_{\mathrm{mon}}\) in QMIX.
\end{definition}

As for other value factorization methods such as WQMIX~\citep{Ref:wqmix} and ResQ~\citep{Ref:resq} where \(Q_{\mathrm{mon}}\) appears in their design, we also assume they share the same property as the ideal QMIX, which aligns with the assumption used in their papers.

From Definition~\ref{definition ideal QMIX}, the ideal QMIX obtains a monotonic representation of \(Q_{\mathrm{jt}}\) with the minimum \(L_\mathrm{tot}\). We illustrate this process in matrix cases: Given the order of rows and columns, we can find the monotonic matrix with local minimum \(L_\mathrm{tot}\) by solving the constraint optimization problem. To find the monotonic matrix with minimum \(L_\mathrm{tot}\), we compare \(L_\mathrm{tot}\) for all possible orders. We give an example of the ideal QMIX shown in Table~\ref{non-monotonic example}.

\begin{table}[h]
  \centering
  \begin{subtable}[h]{0.18\linewidth}
  \begin{tabular}{|c|c|c|}
    \multicolumn{3}{c}{\ } \\
    \hline
    8 & -12  & -12     \\
    \hline
    -12     & 3 & 0      \\
    \hline
    -12     & 0       & 5  \\
    \hline
  \end{tabular}
  \caption{\(Q_{\mathrm{jt}}\)}
    \end{subtable}   
    \hspace{0.05\linewidth}
    \begin{subtable}[h]{0.15\linewidth}
      \begin{tabular}{|c|c|c|}
        \multicolumn{3}{c}{\(L_\mathrm{tot}\)=36.22}     \\
        \hline
        -8 & -8  & -8     \\
        \hline
        -8     & 1 & 1      \\
        \hline
        -8     & 1       & \textbf{5}  \\
        \hline
      \end{tabular}
      \caption{\(Q_{\mathrm{tot}}^*\)}
     \end{subtable}
     \hspace{0.05\linewidth}
    \begin{subtable}[h]{0.15\linewidth}
      \begin{tabular}{|c|c|c|}
        \multicolumn{3}{c}{\(L_\mathrm{tot}\)=45.56}   \\
        \hline
        \textbf{8} & -5  & -5     \\
        \hline
        -5     & -5 & -5      \\
        \hline
        -5     & -5       & -5  \\
        \hline
      \end{tabular}
      \caption{\(Q_{\mathrm{tot}}^*\) for \(\bar{\boldsymbol{u}}=\boldsymbol{u}^*\)}
    \end{subtable}
  \caption{(a): A non-monotonic payoff matrix, where the value presented in row \(r\) and column \(c\) denotes the \(Q_{\mathrm{jt}}\) value of joint action \((r,c)\). (b): The monotonic representation with the global minimum \(L_\mathrm{tot}\). (c) The optimal monotonic representation constrained by \(\bar{\boldsymbol{u}}=\boldsymbol{u}^*\). Since the \(L_\mathrm{tot}\) of (b) is less than the \(L_\mathrm{tot}\) of (c), \(Q_{\mathrm{tot}}\) of the ideal QMIX under uniform visitation (\(\pi\equiv1\)) will converge to (b), which fails to get the global optimal action.}
  \label{non-monotonic example}
\end{table}



\section{Proofs}
\label{proof of theorems}

\subsection{Theorem on MRVF}

\mainthmone*

\begin{proof}
\label{proof: approximate convergence}
Let \(\boldsymbol{\mathcal{U}}_+^k=\{\boldsymbol{u}|Q_{\mathrm{jt}}(\boldsymbol{\tau},\boldsymbol{u})> Q_{\mathrm{jt}}(\boldsymbol{\tau},\bar{\boldsymbol{u}}^{k}),\boldsymbol{u}\in\boldsymbol{\mathcal{U}}\}\) be the set of actions superior to \(\bar{\boldsymbol{u}}^{k}\). 

\textbf{Assume to the contrary that} \(\forall k>0, \bar{\boldsymbol{u}}^k\neq\boldsymbol{u}^*\). Since \(\forall k>1,Q_{\mathrm{jt}}(\boldsymbol{\tau},\bar{\boldsymbol{u}}^k)> Q_{\mathrm{jt}}(\boldsymbol{\tau},\bar{\boldsymbol{u}}^{k-1})\) when \(\bar{\boldsymbol{u}}^{k-1}\neq\boldsymbol{u}^*\) (Strict Improvement Condition), we have \(\boldsymbol{\mathcal{U}}_+^k\subset\boldsymbol{\mathcal{U}}_+^{k-1}\), which is equivalent to

\begin{equation}
    \label{equ:shrink superior set}
    |\boldsymbol{\mathcal{U}}_+^{k-1}|-|\boldsymbol{\mathcal{U}}_+^{k}|\geq1
\end{equation}

Telescoping Equation~(\ref{equ:shrink superior set}) from \(k=1\) to \(k=K\), we get \(\forall K>0, |\boldsymbol{\mathcal{U}}_+^{K}|\leq |\boldsymbol{\mathcal{U}}_+^{1}|-(K-1)\). 

If we choose \(K>|\boldsymbol{\mathcal{U}}|+1\), since \(\boldsymbol{\mathcal{U}}_+^{1}\subseteq\boldsymbol{\mathcal{U}}\), we get \(|\boldsymbol{\mathcal{U}}_+^{K}|< |\boldsymbol{\mathcal{U}}_+^{1}|-|\boldsymbol{\mathcal{U}}|\leq0\), which contradicts the fact that \( |\boldsymbol{\mathcal{U}}_+^{K}|\geq0\). Therefore, we prove that \(\exists K>0, \bar{\boldsymbol{u}}^K=\boldsymbol{u}^*\).
\end{proof}

\begin{lemma}
    \label{Qtot* limit}
    For the idea QMIX defined in Equation~(\ref{ideal QMIX}) and \(\epsilon\)-greedy policy \(\pi\) that is continuous for \(\epsilon\in[0,1]\) and satisfies Equation~(\ref{epsilon greedy policy}), we have \(\lim\limits_{\epsilon\to0^+}Q^*_{\mathrm{tot}}(\boldsymbol{\tau},\widetilde{\boldsymbol{u}})=Q_{\mathrm{jt}}(\boldsymbol{\tau},\widetilde{\boldsymbol{u}})\) for given \(\boldsymbol{\tau}\in \boldsymbol{\mathcal{T}}\), where \(Q_{\mathrm{tot}}^*\) is the converged \(Q_{\mathrm{tot}}\).
    \begin{equation}
        \lim_{\epsilon\to0^+}\pi(\boldsymbol{u}|\boldsymbol{\tau})=
        \begin{cases}
            1   &   \boldsymbol{u}=\widetilde{\boldsymbol{u}} \\
            0   &   otherwise
        \end{cases}
        \label{epsilon greedy policy}
    \end{equation}
\end{lemma}

\begin{proof}
    \textbf{Assume for a contradiction} that \(\lim\limits_{\epsilon\to0^+}Q^*_{\mathrm{tot}}(\boldsymbol{\tau},\widetilde{\boldsymbol{u}})\neq Q_{\mathrm{jt}}(\boldsymbol{\tau},\widetilde{\boldsymbol{u}})\), which means \(\exists \Delta>0, \forall E>0, \exists \epsilon \in (0,E),|Q^*_{\mathrm{tot}}(\boldsymbol{\tau},\widetilde{\boldsymbol{u}})-Q_{\mathrm{jt}}(\boldsymbol{\tau},\widetilde{\boldsymbol{u}})|\geq\Delta\).
    
     The loss of \(Q^*_{\mathrm{tot}}\) defined in Equation~(\ref{ideal QMIX}) is
    \begin{equation}
        \begin{aligned}
            L^*_{\mathrm{tot}} & = \sum_{\boldsymbol{u}}{\pi(\boldsymbol{u}|\boldsymbol{\tau})(Q^*_{\mathrm{tot}}(\boldsymbol{\tau},\boldsymbol{u})-Q_{\mathrm{jt}}(\boldsymbol{\tau},\boldsymbol{u}))^2} \\
            & =\sum_{\boldsymbol{u}\neq\widetilde{\boldsymbol{u}}}\pi(\boldsymbol{u}|\boldsymbol{\tau}){(Q^*_{\mathrm{tot}}(\boldsymbol{\tau},\boldsymbol{u})-Q_{\mathrm{jt}}(\boldsymbol{\tau},\boldsymbol{u}))^2}
            +
            \pi(\widetilde{\boldsymbol{u}}|\boldsymbol{\tau})(Q^*_{\mathrm{tot}}(\boldsymbol{\tau},\widetilde{\boldsymbol{u}})-Q_{\mathrm{jt}}(\boldsymbol{\tau},\widetilde{\boldsymbol{u}}))^2 \\
            & \geq \pi(\widetilde{\boldsymbol{u}}|\boldsymbol{\tau})\Delta^2
        \end{aligned}
    \end{equation}
    Let \(Q_{\mathrm{tot}}^{'}(\boldsymbol{\tau},\cdot)\equiv \frac{\Delta}{2}+Q_{\mathrm{jt}}(\boldsymbol{\tau},\widetilde{\boldsymbol{u}})\). The loss of \(Q_{\mathrm{tot}}^{'}\) is 
    \begin{equation}
        \begin{aligned}
            L_{\mathrm{tot}}^{'} =\sum_{\boldsymbol{u}\neq\widetilde{\boldsymbol{u}}}\pi(\boldsymbol{u}|\boldsymbol{\tau}){(\frac{\Delta}{2}+Q_{\mathrm{jt}}(\boldsymbol{\tau},\widetilde{\boldsymbol{u}})-Q_{\mathrm{jt}}(\boldsymbol{\tau},\boldsymbol{u}))^2}
            +
            \pi(\widetilde{\boldsymbol{u}}|\boldsymbol{\tau})(\frac{\Delta}{2})^2
        \end{aligned}
    \end{equation}
    Consider \(L_{\mathrm{tot}}^{*}-L_{\mathrm{tot}}^{'}\), we have
    \begin{equation}
        L_{\mathrm{tot}}^{*}-L_{\mathrm{tot}}^{'}
        \geq \pi(\widetilde{\boldsymbol{u}}|\boldsymbol{\tau})(\frac{\Delta}{2})^2
        -\sum_{\boldsymbol{u}\neq\widetilde{\boldsymbol{u}}}\pi(\boldsymbol{u}|\boldsymbol{\tau}){(\frac{\Delta}{2}+Q_{\mathrm{jt}}(\boldsymbol{\tau},\widetilde{\boldsymbol{u}})-Q_{\mathrm{jt}}(\boldsymbol{\tau},\boldsymbol{u}))^2} 
    \end{equation}
    Notice that 
    \begin{equation}
       \lim_{\epsilon\to0^+}(\pi(\widetilde{\boldsymbol{u}}|\boldsymbol{\tau})(\frac{\Delta}{2})^2
        -\sum_{\boldsymbol{u}\neq\widetilde{\boldsymbol{u}}}\pi(\boldsymbol{u}|\boldsymbol{\tau}){(\frac{\Delta}{2}+Q_{\mathrm{jt}}(\boldsymbol{\tau},\widetilde{\boldsymbol{u}})-Q_{\mathrm{jt}}(\boldsymbol{\tau},\boldsymbol{u}))^2}) 
        = (\frac{\Delta}{2})^2  > 0   
    \end{equation}
    Therefore, \(\exists E>0, \forall \epsilon\in(0,E),L^*_{\mathrm{tot}}>L_{\mathrm{tot}}^{'}\),  contradicting the condition that \(L_{tot}^*\) is the minimum (the definition of \(Q^*_{\mathrm{tot}}\) in Equation~(\ref{ideal QMIX}). 
    
    Hence the assumption is false, and we have \(\lim\limits_{\epsilon\to0^+}Q^*_{\mathrm{tot}}(\boldsymbol{\tau},\widetilde{\boldsymbol{u}})=Q_{\mathrm{jt}}(\boldsymbol{\tau},\widetilde{\boldsymbol{u}})\).
\end{proof}

\mainthmtwo*

\begin{proof}
    \label{proof: centralized policy avoid the worst}
    For \(\boldsymbol{\tau} \in \boldsymbol{\mathcal{T}}\), \textbf{Assume for a contradiction} that an action with the minimum payoff can be a stable point, which means \(\forall E_{\boldsymbol{\tau}}\in(0,1],\exists\epsilon\in(0,E_{\boldsymbol{\tau}}),\exists Q_{\mathrm{tot}}^*\) that \(Q_{\mathrm{jt}}(\boldsymbol{\tau},\bar{\boldsymbol{u}}^*)=\min_{\boldsymbol{u}}Q_{\mathrm{jt}}(\boldsymbol{\tau},\boldsymbol{u})\) and \(\widetilde{\boldsymbol{u}}=\bar{\boldsymbol{u}}^*\) (\(\bar{\boldsymbol{u}}^*\) is the stable point with the minimun \(Q_{\mathrm{jt}}\)), where \(Q_{\mathrm{tot}}^*\) is the converged \(Q_{\mathrm{tot}}\) of the ideal QMIX (see Definition~\ref{definition ideal QMIX}).
    
    We define \(\boldsymbol{\mathcal{U}}^+=\{\boldsymbol{u}|Q_{\mathrm{jt}}(\boldsymbol{\tau},\boldsymbol{u})>\min_{\boldsymbol{u}}Q_{\mathrm{jt}}(\boldsymbol{\tau},\boldsymbol{u})\}\) and \(\boldsymbol{\mathcal{U}}^-=\boldsymbol{\mathcal{U}}-\boldsymbol{\mathcal{U}}^+\) . 
    
    Since \(\lim\limits_{\epsilon\to0^+}Q^*_{\mathrm{tot}}(\boldsymbol{\tau},\bar{\boldsymbol{u}}^*)=Q_{\mathrm{jt}}(\boldsymbol{\tau},\bar{\boldsymbol{u}}^*)\) (Lemma~\ref{Qtot* limit} and \(\widetilde{\boldsymbol{u}}=\bar{\boldsymbol{u}}^*\)), we have \(\exists E_{\boldsymbol{\tau}}^0\in(0,1], \forall \epsilon\in(0,E_{\boldsymbol{\tau}}^0), Q^*_{\mathrm{tot}}(\boldsymbol{\tau},\bar{\boldsymbol{u}}^*)<\min_{\boldsymbol{u}\in\boldsymbol{\mathcal{U}}^+}Q_{\mathrm{jt}}(\boldsymbol{\tau},\boldsymbol{u})\). Thus, \(\forall \boldsymbol{u}\in\boldsymbol{\mathcal{U}}^+\), since \(Q^*_{\mathrm{tot}}(\boldsymbol{\tau},\boldsymbol{u})\leq Q^*_{\mathrm{tot}}(\boldsymbol{\tau},\bar{\boldsymbol{u}}^*)<Q_{\mathrm{jt}}(\boldsymbol{\tau},\boldsymbol{u})\). 
    
    Therefore, for the point-wise squared error of  \(L^*_{\mathrm{tot}}\), we have
    \begin{equation}
        \begin{aligned}
            L^*_{\mathrm{tot}}(\boldsymbol{u}) & = \pi(\boldsymbol{u}|\boldsymbol{\tau})(Q^*_{\mathrm{tot}}(\boldsymbol{\tau},\boldsymbol{u})-Q_{\mathrm{jt}}(\boldsymbol{\tau},\boldsymbol{u}))^2      \\
            & > \frac{\epsilon}{|\boldsymbol{\mathcal{U}}|}(Q^*_{\mathrm{tot}}(\boldsymbol{\tau},\bar{\boldsymbol{u}}^*)-Q_{\mathrm{jt}}(\boldsymbol{\tau},\boldsymbol{u})), \forall \boldsymbol{u}\in\boldsymbol{\mathcal{U}}^+
        \end{aligned}
    \end{equation}
    Therefore, for the loss of \(Q^*_{\mathrm{tot}}\) defined in Equation~(\ref{ideal QMIX}), we have
    \begin{equation}
        \begin{aligned}
            L^*_{\mathrm{tot}} & = \sum_{\boldsymbol{u}}{\pi(\boldsymbol{u}|\boldsymbol{\tau})(Q^*_{\mathrm{tot}}(\boldsymbol{\tau},\boldsymbol{u})-Q_{\mathrm{jt}}(\boldsymbol{\tau},\boldsymbol{u}))^2} \\
            & >\frac{\epsilon}{|\boldsymbol{\mathcal{U}}|}\sum_{\boldsymbol{u}\in\boldsymbol{\mathcal{U}}^+}{(Q^*_{\mathrm{tot}}(\boldsymbol{\tau},\bar{\boldsymbol{u}}^*)-Q_{\mathrm{jt}}(\boldsymbol{\tau},\boldsymbol{u}))^2}
            +(1-\epsilon+\frac{\epsilon}{|\boldsymbol{\mathcal{U}}|})(Q^*_{\mathrm{tot}}(\boldsymbol{\tau},\bar{\boldsymbol{u}}^*)-Q_{\mathrm{jt}}(\boldsymbol{\tau},\bar{\boldsymbol{u}}^*))^2 \\
        \end{aligned}
    \end{equation}
    Consider certain \(\boldsymbol{u}^+\in \boldsymbol{\mathcal{U}}^+\). We define \(Q_{\mathrm{tot}}^{'}\) as
    \begin{equation}
        Q_{\mathrm{tot}}^{'}(\boldsymbol{\tau},\boldsymbol{u}) = 
        \begin{cases}
            Q_{\mathrm{jt}}(\boldsymbol{\tau},\boldsymbol{u}^+) & \boldsymbol{u}=\boldsymbol{u}^+ \\
            Q^*_{\mathrm{tot}}(\boldsymbol{\tau},\bar{\boldsymbol{u}}^*) & otherwise
        \end{cases}
    \end{equation}
    The loss of \(Q_{\mathrm{tot}}^{'}\) defined in Equation~(\ref{ideal QMIX}) is (\(\widetilde{\boldsymbol{u}}=\bar{\boldsymbol{u}}^*\))
    \begin{equation}
        \begin{aligned}
            L_{\mathrm{tot}}^{'}= & \frac{\epsilon}{|\boldsymbol{\mathcal{U}}|}(\sum_{\boldsymbol{u}\in(\boldsymbol{\mathcal{U}}^+-\{\boldsymbol{u}^+\})}{(Q^*_{\mathrm{tot}}(\boldsymbol{\tau},\bar{\boldsymbol{u}}^*)-Q_{\mathrm{jt}}(\boldsymbol{\tau},\boldsymbol{u}))^2}+\sum_{\boldsymbol{u}\in\boldsymbol{\mathcal{U}}^-}{(Q^*_{\mathrm{tot}}(\boldsymbol{\tau},\bar{\boldsymbol{u}}^*)-Q_{\mathrm{jt}}(\boldsymbol{\tau},\bar{\boldsymbol{u}}^*))^2}) \\
            & +(1-\epsilon+\frac{\epsilon}{|\boldsymbol{\mathcal{U}}|})(Q^*_{\mathrm{tot}}(\boldsymbol{\tau},\bar{\boldsymbol{u}}^*)-Q_{\mathrm{jt}}(\boldsymbol{\tau},\bar{\boldsymbol{u}}^*))^2            
        \end{aligned}
    \end{equation}
    Comparing \(L_{\mathrm{tot}}^{*}\) and \(L_{\mathrm{tot}}^{'}\), we have
    \begin{equation}
        L_{\mathrm{tot}}^{*} - L_{\mathrm{tot}}^{'} > 
        \frac{\epsilon}{|\boldsymbol{\mathcal{U}}|}(
        (Q^*_{\mathrm{tot}}(\boldsymbol{\tau},\bar{\boldsymbol{u}}^*)-Q_{\mathrm{jt}}(\boldsymbol{\tau},\boldsymbol{u}^+))^2
        -|\boldsymbol{\mathcal{U}}^-|(Q^*_{\mathrm{tot}}(\boldsymbol{\tau},\bar{\boldsymbol{u}}^*)-Q_{\mathrm{jt}}(\boldsymbol{\tau},\bar{\boldsymbol{u}}^*))^2)
    \end{equation}
    Notice that 
    \begin{equation}
        f(x) = (x-Q_{\mathrm{jt}}(\boldsymbol{\tau},\boldsymbol{u}^+))^2
        -|\boldsymbol{\mathcal{U}}^-|(x-Q_{\mathrm{jt}}(\boldsymbol{\tau},\bar{\boldsymbol{u}}^*))^2
    \end{equation}
    is a continuous function and
    \begin{equation}
        \lim\limits_{x\to Q_{\mathrm{jt}}(\boldsymbol{\tau},\bar{\boldsymbol{u}}^*)}f(x)=(Q_{\mathrm{jt}}(\boldsymbol{\tau},\bar{\boldsymbol{u}}^*)-Q_{\mathrm{jt}}(\boldsymbol{\tau},\boldsymbol{u}^+))^2>0
    \end{equation}
    Since \(\lim\limits_{\epsilon\to0^+}Q^*_{\mathrm{tot}}(\boldsymbol{\tau},\bar{\boldsymbol{u}}^*)=Q_{\mathrm{jt}}(\boldsymbol{\tau},\bar{\boldsymbol{u}}^*)\), we have \(\lim\limits_{\epsilon\to0^+}f(Q^*_{\mathrm{tot}}(\boldsymbol{\tau},\bar{\boldsymbol{u}}^*))\), which means \(\exists E_{\boldsymbol{\tau}}^1\in(0,1], \forall \epsilon\in(0,E_{\boldsymbol{\tau}}^1),f(Q^*_{\mathrm{tot}}(\boldsymbol{\tau},\bar{\boldsymbol{u}}^*))>0\).
    
    Therefore, we find \(E_{\boldsymbol{\tau}}=\min\{E_{\boldsymbol{\tau}}^0,E_{\boldsymbol{\tau}}^1\}\in(0,1]\) that \(\forall \epsilon\in(0,E_{\boldsymbol{\tau}})\) \(L_{\mathrm{tot}}^{*} > L_{\mathrm{tot}}^{'}\), which contradicts the condition that \(Q_{\mathrm{tot}}^*\) satisfies Equation~(\ref{ideal QMIX}). Thus, \(\exists E_{\boldsymbol{\tau}}\in(0,1],\forall\epsilon\in(0,E_{\boldsymbol{\tau}}),\forall \boldsymbol{u} \in \arg\min_{\boldsymbol{u}}Q_{\mathrm{jt}}(\boldsymbol{\tau},\boldsymbol{u})\), \(\boldsymbol{u}\) is not a stable point. 
    
    Finally, we complete the proof by taking \(E=\min\limits_{\boldsymbol{\tau}\in\boldsymbol{\mathcal{T}}}E_{\boldsymbol{\tau}}\).
\end{proof}

\subsection{Theorem on ResQ}
\label{stablility analysis of resq}
We will prove that \(\boldsymbol{u}^*\) is the unique yet weakly stable point of ResQ~\citep{Ref:resq}. Here, we present the expression of \(Q_{\mathrm{tot}}\) of ResQ.
\begin{equation}
    Q_{\mathrm{tot}}(\boldsymbol{\tau},\boldsymbol{u})=Q_{\mathrm{mon}}(\boldsymbol{\tau},\boldsymbol{u})+w_\mathrm{r}(\boldsymbol{\tau},\boldsymbol{u})*Q_{\mathrm{r}}(\boldsymbol{\tau},\boldsymbol{u})
\end{equation}
where \(Q_{\mathrm{mon}}=f_{\mathrm{mon}}(Q_1,Q_2,\cdots,Q_n)\), \(Q_{\mathrm{r}}\leq0\) and
\begin{equation}
    \label{equ: resq wr}
    w_\mathrm{r}(\boldsymbol{\tau},\boldsymbol{u})=
    \begin{cases}
        0   &   \boldsymbol{u}=\widetilde{\boldsymbol{u}} \\
        1   &   otherwise
    \end{cases}
\end{equation}

The loss function for certain \(\boldsymbol{\tau}\) is 

\begin{equation}
    L_\mathrm{tot}(\boldsymbol{\tau})=\sum_{\boldsymbol{u}}\pi(\boldsymbol{u}|\boldsymbol{\tau})(Q_\mathrm{tot}(\boldsymbol{\tau},\boldsymbol{u})-Q_\mathrm{jt}(\boldsymbol{\tau},\boldsymbol{u}))^2
\end{equation}

We provide a specific instance for the weak stability of ResQ in Table~\ref{resq example}. Then, we prove that the weak stability occurs in general cases in Theorem~\ref{resq weakly stable point}.
\begin{table}[h]
  \centering
  \begin{subtable}[h]{0.2\linewidth}
  \begin{tabular}{|c|c|c|}
    \hline
    \textcolor{red}{8} & -12  & -12     \\
    \hline
    -12     & 3 & 0      \\
    \hline
    -12     & 0       & 5  \\
    \hline
  \end{tabular}
    \caption{\(Q_{\mathrm{jt}}\)}
    \end{subtable}   
    \hspace{0.08\linewidth}
    \begin{subtable}[h]{0.2\linewidth}
      \begin{tabular}{|c|c|c|}
        \hline
        \  \textcolor{red}{8} \  & \ 8 \  & \  8  \     \\
        \hline
        8     & 8 & \textcolor{blue}{9}      \\
        \hline
        8     & 8       & 8  \\
        \hline
      \end{tabular}
      \caption{\(Q_{\mathrm{mon}}\)}
     \end{subtable}
     \hspace{0.05\linewidth}
    \begin{subtable}[h]{0.2\linewidth}
      \begin{tabular}{|c|c|c|}
        \hline
        \  \textcolor{red}{0} \   &   -20    &   -20       \\
        \hline
        -20     & -5 & \textcolor{blue}{-9}      \\
        \hline
        -20     & -8       & -4  \\
        \hline
      \end{tabular}
      \caption{\(w_\mathrm{r}*Q_\mathrm{r}\)}
    \end{subtable}
    \hfill
  \caption{This table shows the weak stability of ResQ for a specific \(Q_{\mathrm{jt}}\) in (a), where we let \(\widetilde{\boldsymbol{u}}=\boldsymbol{u}^*\) initially to check whether ResQ will stabilize at \(\boldsymbol{u}^*\). We highlight the values related to \(\widetilde{\boldsymbol{u}}\) in \textcolor{red}{red} and \(\bar{\boldsymbol{u}}\) in \textcolor{blue}{blue}. \(Q_{\mathrm{tot}}\) is the sum of \(Q_{\mathrm{mon}}\) in (b) and \(w_\mathrm{r}*Q_\mathrm{r}\) in (c), which equals to \(Q_{\mathrm{jt}}\) (\(L_{\mathrm{tot}}=0\)). However, since \(\bar{\boldsymbol{u}}\neq\boldsymbol{u}^*\), \(\boldsymbol{u}^*\) is not a strongly stable point for ResQ.}
  \label{resq example}
\end{table}

\begin{lemma}
    \label{resq zero loss}
    Given a joint action value \(Q_{\mathrm{jt}}\) and \(\boldsymbol{\tau} \in \boldsymbol{\mathcal{T}}\), \(\forall \widetilde{\boldsymbol{u}}\in \boldsymbol{\mathcal{U}}\), we can find \(Q_{\mathrm{tot}}\) with \(L_{\mathrm{tot}}=0\) and \(\bar{\boldsymbol{u}}\neq\widetilde{\boldsymbol{u}}\).
\end{lemma}

\begin{proof}
    We can find a \(Q_{\mathrm{mon}}\) that satisfies
    \begin{enumerate}
        \item \(Q_{\mathrm{mon}}\geq Q_{\mathrm{jt}}\)
        \item \(Q_{\mathrm{mon}}(\boldsymbol{\tau},\widetilde{\boldsymbol{u}})=Q_{\mathrm{jt}}(\boldsymbol{\tau},\widetilde{\boldsymbol{u}})\)
        \item \( \bar{\boldsymbol{u}}\neq\widetilde{\boldsymbol{u}}\)
    \end{enumerate}
    Here we give a specific \(Q_{\mathrm{mon}}\) that satisfies these conditions:
    \begin{equation}
        Q_{\mathrm{mon}}(\boldsymbol{\tau},\boldsymbol{u}) = \begin{cases}
            Q_{\mathrm{jt}}(\boldsymbol{\tau},\widetilde{\boldsymbol{u}})   &   \boldsymbol{u}=\widetilde{\boldsymbol{u}} \\
            \max_{\boldsymbol{u}}Q_{\mathrm{jt}}(\boldsymbol{\tau},\boldsymbol{u})+\Delta   &   otherwise \\
        \end{cases}
    \end{equation}
    where \(\Delta>0\). We let \(Q_\mathrm{r}=Q_{\mathrm{jt}}-Q_{\mathrm{mon}}\). Since \(Q_{\mathrm{mon}}\geq Q_{\mathrm{jt}}\), we have \(Q_\mathrm{r}\leq 0\). Therefore, we obtain \(Q_\mathrm{mon}\), \(w_\mathrm{r}\) and \(Q_\mathrm{r}\), which together constitute the \(Q_\mathrm{tot}\) of ResQ.
    
    \(\forall \boldsymbol{u}\neq\widetilde{\boldsymbol{u}}\), we have \(Q_{\mathrm{jt}}(\boldsymbol{\tau},\boldsymbol{u})=Q_{\mathrm{tot}}(\boldsymbol{\tau},\boldsymbol{u})=Q_{\mathrm{mon}}(\boldsymbol{\tau},\boldsymbol{u})+1*Q_{\mathrm{r}}(\boldsymbol{\tau},\boldsymbol{u})\). For \(\widetilde{\boldsymbol{u}}\), since \( Q_{\mathrm{mon}}(\boldsymbol{\tau},\widetilde{\boldsymbol{u}})=Q_{\mathrm{jt}}(\boldsymbol{\tau},\widetilde{\boldsymbol{u}})\), we have \(Q_{\mathrm{jt}}(\boldsymbol{\tau},\widetilde{\boldsymbol{u}})=Q_{\mathrm{tot}}(\boldsymbol{\tau},\widetilde{\boldsymbol{u}})=Q_{\mathrm{mon}}(\boldsymbol{\tau},\widetilde{\boldsymbol{u}})+0*Q_{\mathrm{r}}(\boldsymbol{\tau},\widetilde{\boldsymbol{u}})\). 
    
    Therefore, \(\forall \boldsymbol{u}\in\boldsymbol{\mathcal{U}}\), we have \(Q_{\mathrm{jt}}(\boldsymbol{\tau},\boldsymbol{u})=Q_{\mathrm{tot}}(\boldsymbol{\tau},\boldsymbol{u})\), which means \(L_{\mathrm{tot}}=0\). Thus, we find \(Q_{\mathrm{tot}}\) with \(L_{\mathrm{tot}}=0\) and \(\bar{\boldsymbol{u}}\neq\widetilde{\boldsymbol{u}}\).
\end{proof}

Lemma~\ref{resq zero loss} demonstrates that in ResQ, for any two actions in the entire action space, the greedy action may transit from one to another, which indicates that no action can be a strongly stable point.

\begin{theorem}
    \label{resq weakly stable point}
    Given a joint action value \(Q_{\mathrm{jt}}\) and \(\boldsymbol{\tau} \in \boldsymbol{\mathcal{T}}\), \(\forall \boldsymbol{u}^* \in \arg\max_{\boldsymbol{u}}Q_{\mathrm{jt}}(\boldsymbol{\tau},\boldsymbol{u})\), \(\boldsymbol{u}^*\) is a weakly stable point of ResQ~\citep{Ref:resq}, while \(\forall \boldsymbol{u}^{-}\notin \arg\max_{\boldsymbol{u}}Q_{\mathrm{jt}}(\boldsymbol{\tau},\boldsymbol{u})\) is not a stable point of ResQ.
\end{theorem}


\begin{proof}
    \label{proof: resq weakly stable point}
    \textbf{First, we illustrate that \(\boldsymbol{u}^*\) is a weak stable point.} 
    
    Let \(\widetilde{\boldsymbol{u}}=\boldsymbol{u}^*\). We can find a \(Q_{\mathrm{mon}}\) that satisfies
    \begin{enumerate}
        \item \(Q_{\mathrm{mon}}\geq Q_{\mathrm{jt}}\)
        \item \(Q_{\mathrm{mon}}(\boldsymbol{\tau},\widetilde{\boldsymbol{u}})=Q_{\mathrm{jt}}(\boldsymbol{\tau},\widetilde{\boldsymbol{u}})\)
        \item \( \bar{\boldsymbol{u}}=\boldsymbol{u}^*\)
    \end{enumerate}
    where \(Q_{\mathrm{mon}}\equiv Q_{\mathrm{jt}}(\boldsymbol{\tau},\boldsymbol{u}^*)\) is a specific case. 
    
    We let \(Q_\mathrm{r}=Q_{\mathrm{jt}}-Q_{\mathrm{mon}}\). Since \(Q_{\mathrm{mon}}\geq Q_{\mathrm{jt}}\), we have \(Q_\mathrm{r}\leq 0\). Therefore, we obtain \(Q_\mathrm{mon}\), \(w_\mathrm{r}\) and \(Q_\mathrm{r}\), which together constitute the \(Q_\mathrm{tot}\) of ResQ.
    
    According to \(w_\mathrm{r}\) defined in Equation~(\ref{equ: resq wr}), we have 
    \begin{equation}
        \label{equ: made Qtot of resq}
        Q_{\mathrm{tot}}=
        \begin{cases}
            Q_{\mathrm{mon}}(\boldsymbol{\tau},\boldsymbol{u})+1*Q_{\mathrm{r}}(\boldsymbol{\tau},\boldsymbol{u})   & \forall \boldsymbol{u}\neq\widetilde{\boldsymbol{u}}\\
            Q_{\mathrm{mon}}(\boldsymbol{\tau},\widetilde{\boldsymbol{u}})+0*Q_{\mathrm{r}}(\boldsymbol{\tau},\widetilde{\boldsymbol{u}})   & \boldsymbol{u}=\widetilde{\boldsymbol{u}}
        \end{cases}
    \end{equation}

    Substituting \(Q_\mathrm{r}=Q_{\mathrm{jt}}-Q_{\mathrm{mon}}\) and \(Q_{\mathrm{mon}}(\boldsymbol{\tau},\widetilde{\boldsymbol{u}})=Q_{\mathrm{jt}}(\boldsymbol{\tau},\widetilde{\boldsymbol{u}})\) into Equation~(\ref{equ: made Qtot of resq}) yields \(Q_{\mathrm{tot}}\equiv Q_{\mathrm{jt}}\).
    
    Since \(Q_{\mathrm{tot}}\equiv Q_{\mathrm{jt}}\), we have \(L_{\mathrm{tot}}=0\), which is the lower bound of loss. Therefore, \(\boldsymbol{u}^*\) is a \textbf{stable point} of ResQ. 
    
    However, according to Lemma~\ref{resq zero loss} for \(\widetilde{\boldsymbol{u}}=\boldsymbol{u}^*\), we can also find \(Q_{\mathrm{tot}}^{'}\) with \(L_{\mathrm{tot}}^{'}=0\) while \(\bar{\boldsymbol{u}}^{'}\neq\boldsymbol{u}^*\). Thus, \(\boldsymbol{u}^*\) is a \textbf{weakly stable point} of ResQ.
    
    \textbf{Second, we illustrate that \(\boldsymbol{u}^{-}\) is not a stable point.} 
    
    Let \(\widetilde{\boldsymbol{u}}=\boldsymbol{u}^{-}\). \textbf{Assume for a contradiction} that there exists \(Q_{\mathrm{tot}}\) satisfying \(\bar{\boldsymbol{u}}=\boldsymbol{u}^{-}\) and achieving the minimum loss. According to Lemma~\ref{resq zero loss}, we have \(L_{\mathrm{tot}}=0\), otherwise \(L_{\mathrm{tot}}\) is not the minimum. 
    
    Since \(L_{\mathrm{tot}}=0\), we have \(\forall \boldsymbol{u}\in \boldsymbol{\mathcal{U}}, Q_{\mathrm{tot}}(\boldsymbol{\tau},\boldsymbol{u})=Q_{\mathrm{jt}}(\boldsymbol{\tau},\boldsymbol{u})\). Therefore, we have \(Q_{\mathrm{tot}}(\boldsymbol{\tau},\boldsymbol{u}^{-})=Q_{\mathrm{jt}}(\boldsymbol{\tau},\boldsymbol{u}^{-})\) for \(\boldsymbol{u}^{-}\) and  \(Q_{\mathrm{tot}}(\boldsymbol{\tau},\boldsymbol{u}^*)=Q_{\mathrm{jt}}(\boldsymbol{\tau},\boldsymbol{u}^*)\) for \(\boldsymbol{u}^*\). 
    
    However, since \(Q_{\mathrm{tot}}(\boldsymbol{\tau},\boldsymbol{u}^{-})\geq Q_{\mathrm{tot}}(\boldsymbol{\tau},\boldsymbol{u}^*)\) due to \(\bar{\boldsymbol{u}}=\boldsymbol{u}^{-}\), we get a contradiction where \(Q_{\mathrm{jt}}(\boldsymbol{\tau},\boldsymbol{u}^{-})\geq Q_{\mathrm{jt}}(\boldsymbol{\tau},\boldsymbol{u}^*)\). Therefore, \(\boldsymbol{u}^{-}\) is not a stable point.
    
\end{proof}

\section{Stability Analysis of QPLEX}
\label{stability analysis of QPLEX}
In this section, we provide a suboptimal case of QPLEX in matrix games. Here, we present the expression of \(Q_{\mathrm{tot}}\) of QPLEX.
\begin{equation}
    Q_{\mathrm{tot}}(\boldsymbol{\tau},\boldsymbol{u})=\sum_i{\max_{u_i}{Q_i(\tau_i,u_i)}}+\sum_i{w_i(\boldsymbol{\tau},\boldsymbol{u})*(Q_i(\tau_i,u_i)-\max_{u_i}{Q_i(\tau_i,u_i))}}
\end{equation}
where \(w_i(\boldsymbol{\tau},\boldsymbol{u})\geq0\), and \(\max_{u_i}Q_i(\tau_i,u_i)\) can be treated as a gradient-free component, since
\begin{equation}
    \max_{u_i}Q_i(\tau_i,u_i)=Q_i(\tau_i,\bar{\boldsymbol{u}}_i), \quad \bar{\boldsymbol{u}}_i=\arg\max_{u_i}Q_i(\tau_i,u_i)
\end{equation}

The loss function for certain \(\boldsymbol{\tau}\) under uniform visitation is 

\begin{equation}
    L_\mathrm{tot}(\boldsymbol{\tau})=\sum_{\boldsymbol{u}}(Q_\mathrm{tot}(\boldsymbol{\tau},\boldsymbol{u})-Q_\mathrm{jt}(\boldsymbol{\tau},\boldsymbol{u}))^2
\end{equation}

For QPLEX, stable point occurs when the gradients of \(\boldsymbol{Q}\) and \(\boldsymbol{w}\) satisfy that \(\forall i\in\{1,2,\cdots,N\}\)
\begin{equation}
    \frac{\partial L_{\mathrm{tot}}}{\partial \boldsymbol{Q}_i}=\boldsymbol{0},
    \frac{\partial L_{\mathrm{tot}}}{\partial \boldsymbol{w}_i}\geq\boldsymbol{0}
    \label{QPLEX condition}
\end{equation}
where \(\frac{\partial L_{\mathrm{tot}}}{\partial w_i(\boldsymbol{\tau},\boldsymbol{u})}>0\) only when \(w_i(\boldsymbol{\tau},\boldsymbol{u}) = 0\). 


The suboptimal cases of QPLEX are presented in Table~\ref{QPLEX fail example} and Table~\ref{QPLEX fail while QMIX success example}, where there are three stable points but only one is optimal. For readers to verify those stable points, we expand Equation~(\ref{QPLEX condition}) into Equation~(\ref{QPLEX gradient}).
\begin{equation}
    \begin{aligned}
        \frac{\partial L_{\mathrm{tot}}}{\partial Q_i(\tau_i,u_i)}= &
        \begin{cases}
            \sum\limits_{\boldsymbol{u}\in\boldsymbol{\mathcal{U}}(\boldsymbol{u}_i=u_i)}(Q_{\mathrm{tot}}(\boldsymbol{\tau},\boldsymbol{u})-Q_{\mathrm{jt}}(\boldsymbol{\tau},\boldsymbol{u}))w_i(\boldsymbol{\tau},\boldsymbol{u})   &   u_i\neq\bar{\boldsymbol{u}}_i \\
            \\
            \\
            \begin{aligned}
            \sum\limits_{\boldsymbol{u}\in\boldsymbol{\mathcal{U}}(\boldsymbol{u}_i=u_i)}(Q_{\mathrm{tot}}(\boldsymbol{\tau},\boldsymbol{u})-Q_{\mathrm{jt}}(\boldsymbol{\tau},\boldsymbol{u}))w_i(\boldsymbol{\tau},\boldsymbol{u}) + \\\sum\limits_{\boldsymbol{u}\in\boldsymbol{\mathcal{U}}}(Q_{\mathrm{tot}}(\boldsymbol{\tau},\boldsymbol{u})-Q_{\mathrm{jt}}(\boldsymbol{\tau},\boldsymbol{u}))(1-w_i(\boldsymbol{\tau},\boldsymbol{u}))  \end{aligned} &   u_i=\bar{\boldsymbol{u}}_i
        \end{cases}
        \\
        \frac{\partial L_{\mathrm{tot}}}{\partial w_i(\boldsymbol{\tau},\boldsymbol{u})}= & 
        (Q_{\mathrm{tot}}(\boldsymbol{\tau},\boldsymbol{u})-Q_{\mathrm{jt}}(\boldsymbol{\tau},\boldsymbol{u}))(Q_i(\tau_i, \boldsymbol{u}_i)-Q_i(\tau_i, \bar{\boldsymbol{u}}_i))
    \end{aligned}
    \label{QPLEX gradient}
\end{equation}
where \(\boldsymbol{\mathcal{U}}(\boldsymbol{u}_i=u_i)=\{\boldsymbol{u}|\boldsymbol{u}\in\boldsymbol{\mathcal{U}},\boldsymbol{u}_i=u_i\}\).

\begin{table}[h]
  \centering
  \begin{subtable}[h]{0.2\linewidth}
        \begin{tabular}{|c|c|c|}
        \hline
        8 & -12  & -12     \\
        \hline
        -12     & 3 & 0      \\
        \hline
        -12     & 0       & 5  \\
        \hline
      \end{tabular}
      \caption{\(Q_{\mathrm{jt}}\)}
  \end{subtable}    
  \\
  \begin{subtable}[h]{0.3\linewidth}
  \centering
        \resizebox{0.95\textwidth}{!}{
        \begin{tabular}{|c|c|c|c|}
            \hline
            \diagbox[width=3em]{\(Q_1\)}{\(Q_2\)}  & \textbf{2.77} & 0.31  & 0.37  \\
            \hline
            \textbf{5.23}   &   \textbf{8} & -12  & -12     \\
            \hline
            0.49   &   -12     & 3 & 0      \\
            \hline
            0.33   &   -12     & 0       & 5  \\
            \hline
        \end{tabular}
        }
        \caption{\(\check{Q}_{\mathrm{tot}}\)}
  \end{subtable}  
  \hfill
  \begin{subtable}[h]{0.3\linewidth}
  \centering
        \resizebox{0.95\textwidth}{!}{
        \begin{tabular}{|c|c|c|c|}
            \hline
            \diagbox[width=3em]{\(Q_1\)}{\(Q_2\)}  & 1.82 & \textbf{3.09}  & 0.70  \\
            \hline
            1.42   &   5.5 & -12  & -12     \\
            \hline
            \textbf{2.41}   &   -12     & \textbf{5.5} & 0      \\
            \hline
            1.09   &   -12     & 0       & 5  \\
            \hline
        \end{tabular}
        }
        \caption{\(\check{Q}_{\mathrm{tot}}\)}
  \end{subtable}
  \hfill
  \begin{subtable}[h]{0.3\linewidth}
  \centering
        \resizebox{0.95\textwidth}{!}{
        \begin{tabular}{|c|c|c|c|}
            \hline
            \diagbox[width=3em]{\(Q_1\)}{\(Q_2\)}  & 1.29 & 0.66  & \textbf{3.41}  \\
            \hline
            1.38   &   6.5 & -12  & -12     \\
            \hline
            0.68   &   -12     & 3 & 0      \\
            \hline
            \textbf{3.09}   &   -12     & 0       & \textbf{6.5}  \\
            \hline
        \end{tabular}
        }
        \caption{\(\check{Q}_{\mathrm{tot}}\)}
  \end{subtable}
  \hfill
  \\
  \hspace{0.035\linewidth}
  \begin{subtable}[h]{0.2\linewidth}
        \begin{tabular}{|c|c|c|}
        \hline
        \textbf{0.25} & 1.51  & 1.71     \\
        \hline
        4.21     & 0.48 & 0.10      \\
        \hline
        4.08     & 0.59       & 0.09  \\
        \hline
      \end{tabular}
      \caption{\(\check{w}_1\)}
  \end{subtable}   
  \hfill
  \begin{subtable}[h]{0.2\linewidth}
        \begin{tabular}{|c|c|c|}
        \hline
        0.00 & 17.64  & 5.01     \\
        \hline
        8.53     & \textbf{7.65} & 1.41      \\
        \hline
        5.30     & 4.16       & 0.15  \\
        \hline
      \end{tabular}
      \caption{\(\check{w}_1\)}
  \end{subtable}   
  \hfill
  \begin{subtable}[h]{0.2\linewidth}
        \begin{tabular}{|c|c|c|}
        \hline
        0.00 & 6.70  & 10.86     \\
        \hline
        3.35     & 0.26 & 2.71      \\
        \hline
        5.10     & 1.03       & \textbf{1.92}  \\
        \hline
      \end{tabular}
      \caption{\(\check{w}_1\)}
  \end{subtable}
  \hspace{0.05\linewidth}
  \hfill
  \\
  \hspace{0.035\linewidth}
  \begin{subtable}[h]{0.2\linewidth}
        \begin{tabular}{|c|c|c|}
        \hline
        \textbf{2.01} & 8.13  & 8.34     \\
        \hline
        0.85     & 1.12 & 3.12      \\
        \hline
        4.62     & 2.08       & 1.05  \\
        \hline
      \end{tabular}
      \caption{\(\check{w}_2\)}
  \end{subtable}   
  \hfill
  \begin{subtable}[h]{0.2\linewidth}
        \begin{tabular}{|c|c|c|}
        \hline
        0.00 & 9.07  & 5.22     \\
        \hline
        13.71     & \textbf{7.94} & 2.29      \\
        \hline
        8.23     & 3.76       & 0.13  \\
        \hline
      \end{tabular}
      \caption{\(\check{w}_2\)}
  \end{subtable}   
  \hfill
  \begin{subtable}[h]{0.2\linewidth}
        \begin{tabular}{|c|c|c|}
        \hline
        0.00 & 2.58  & 2.50     \\
        \hline
        4.93     & 1.05 & 0.23      \\
        \hline
        8.72     & 2.36       & \textbf{1.47}  \\
        \hline
      \end{tabular}
      \caption{\(\check{w}_2\)}
  \end{subtable}
  \hspace{0.05\linewidth}
  \caption{The example of QPLEX in which \(\boldsymbol{u}^*\) is not the only one stable points. The quantities of \(\check{Q}_{\mathrm{tot}}^*\) and \(\check{w}_i\) listed in the same column correspond to a particular stable point, where we highlight \(\check{\boldsymbol{u}}\) in bold.}
  \label{QPLEX fail example}
\end{table} 

\begin{table}[h!]
  \centering
  \begin{subtable}[h]{0.2\linewidth}
        \begin{tabular}{|c|c|c|}
        \hline
        \ 8 \ & \ 0 \  & \ 0 \     \\
        \hline
        0     & 3 & 0      \\
        \hline
        0     & 0       & 5  \\
        \hline
      \end{tabular}
      \caption{\(Q_{\mathrm{jt}}\)}
  \end{subtable}    
  \\
  \begin{subtable}[h]{0.3\linewidth}
  \centering
        \resizebox{0.95\textwidth}{!}{
        \begin{tabular}{|c|c|c|c|}
            \hline
            \diagbox[width=3em]{\(Q_1\)}{\(Q_2\)}  & \textbf{4.14} & 1.85  & 0.73  \\
            \hline
            \textbf{3.86}   &   \textbf{8} & 0  & 0     \\
            \hline
            0.44   &   0     & 3 & 0      \\
            \hline
            0.54   &   0     & 0       & 5  \\
            \hline
        \end{tabular}
        }
        \caption{\(\check{Q}_{\mathrm{tot}}\)}
  \end{subtable}  
  \hfill
  \begin{subtable}[h]{0.3\linewidth}
  \centering
        \resizebox{0.95\textwidth}{!}{
        \begin{tabular}{|c|c|c|c|}
            \hline
            \diagbox[width=3em]{\(Q_1\)}{\(Q_2\)}  & 2.57 & \textbf{2.75}  & 0.99  \\
            \hline
            2.59   &   5.5 & 0  & 0     \\
            \hline
            \textbf{2.75}   &   0    & \textbf{5.5} & 0      \\
            \hline
            1.11   &   0     & 0       & 5  \\
            \hline
        \end{tabular}
        }
        \caption{\(\check{Q}_{\mathrm{tot}}\)}
  \end{subtable}
  \hfill
  \begin{subtable}[h]{0.3\linewidth}
  \centering
        \resizebox{0.95\textwidth}{!}{
        \begin{tabular}{|c|c|c|c|}
            \hline
            \diagbox[width=3em]{\(Q_1\)}{\(Q_2\)}  & 2.86 & 0.57  & \textbf{3.15}  \\
            \hline
            3.06   &   6.5 & 0  & 0     \\
            \hline
            0.82   &   0     & 3 & 0      \\
            \hline
            \textbf{3.35}   &   0     & 0       & \textbf{6.5}  \\
            \hline
        \end{tabular}
        }
        \caption{\(\check{Q}_{\mathrm{tot}}\)}
  \end{subtable}
  \hfill
  \\
  \hspace{0.035\linewidth}
  \begin{subtable}[h]{0.2\linewidth}
        \begin{tabular}{|c|c|c|}
        \hline
        \textbf{0.94} & 1.31  & 1.10     \\
        \hline
        2.34     & 0.83 & 0.91      \\
        \hline
        2.41     & 1.49       & 0.83  \\
        \hline
      \end{tabular}
      \caption{\(\check{w}_1\)}
  \end{subtable}   
  \hfill
  \begin{subtable}[h]{0.2\linewidth}
        \begin{tabular}{|c|c|c|}
        \hline
        0.00 & 33.25  & 1.48     \\
        \hline
        8.82     & \textbf{13.44} & 0.44      \\
        \hline
        3.00     & 3.35       & 0.16  \\
        \hline
      \end{tabular}
      \caption{\(\check{w}_1\)}
  \end{subtable}   
  \hfill
  \begin{subtable}[h]{0.2\linewidth}
        \begin{tabular}{|c|c|c|}
        \hline
        0.00 & 2.91  & 22.31     \\
        \hline
        2.13     & 0.54 & 2.57      \\
        \hline
        4.35     & 1.53       & \textbf{14.12}  \\
        \hline
      \end{tabular}
      \caption{\(\check{w}_1\)}
  \end{subtable}
  \hspace{0.05\linewidth}
  \hfill
  \\
  \hspace{0.035\linewidth}
  \begin{subtable}[h]{0.2\linewidth}
        \begin{tabular}{|c|c|c|}
        \hline
        \textbf{1.44} & 3.50  & 2.34     \\
        \hline
        0.76     & 0.95 & 1.43      \\
        \hline
        0.66     & 1.34       & 0.07  \\
        \hline
      \end{tabular}
      \caption{\(\check{w}_2\)}
  \end{subtable}   
  \hfill
  \begin{subtable}[h]{0.2\linewidth}
        \begin{tabular}{|c|c|c|}
        \hline
        0.00 & 19.52  & 2.99     \\
        \hline
        31.74     & \textbf{20.50} & 3.13      \\
        \hline
        3.35     & 1.73       & 0.14  \\
        \hline
      \end{tabular}
      \caption{\(\check{w}_2\)}
  \end{subtable}   
  \hfill
  \begin{subtable}[h]{0.2\linewidth}
        \begin{tabular}{|c|c|c|}
        \hline
        0.00 & 2.20  & 0.10     \\
        \hline
        3.89     & 0.83 & 1.94      \\
        \hline
        22.74     & 2.53       & \textbf{6.15}  \\
        \hline
      \end{tabular}
      \caption{\(\check{w}_2\)}
  \end{subtable}
  \hspace{0.05\linewidth}
  \caption{The example in which QPLEX has two suboptimal stable points, while QMIX can easily obtain the optimum. The quantities of \(\check{Q}_{\mathrm{tot}}\) and \(\check{w}_i\) listed in the same column correspond to a particular stable point, where we highlight \(\check{\boldsymbol{u}}\) in bold.}
  \label{QPLEX fail while QMIX success example}
\end{table} 

\section{Experimental Setup}
\label{experimental setup}
We adopt a setup similar to PyMARL~\citep{Ref:smac}. We use the implementation of QMIX, QTRAN, QPLEX, WQMIX, ResQ, and NA\textsuperscript{2}Q~\citep{Ref:na2q} from their open-source repositories~\footnote{https://github.com/oxwhirl/pymarl}~\footnote{https://github.com/oxwhirl/wqmix}~\footnote{https://github.com/xmu-rl-3dv/ResQ}~\footnote{https://github.com/zichuan-liu/NA2Q}.  These codes are released under the Apache License V2.0. 

All algorithms apply TD(\(0\)) to update \(Q_{\mathrm{tot}}\) including ResQ, as TD(\(\lambda\)) result in catastrophic performance in \textit{5m\_vs\_6m} while no measurable improvement in other scenarios. The hyperparameters for all algorithms and environments are presented in Table~\ref{hyperparameters}. For WQMIX, we set \(\alpha=0.1\) for all environments.  For MRVF, we use three rounds of value factorization, with a probability \(p=0.2\) for sampling preselected actions. 

For hardware, we run experiments on an NVIDIA 3090 GPU for risk-reward games, predator-prey tasks, and SMAC scenarios (excluding \textit{bane\_vs\_bane}). In addition, we use an NVIDIA A800 GPU for MRVF in  \textit{2c\_vs\_64zg} and \textit{MMM2} scenarios.

\begin{table}[h]
    \centering
    \begin{tabular}{lll}
        \toprule
        \textbf{Hyperparameter} & \textbf{Value} & \textbf{Description}  \\
        \midrule
        Batch Size & 32 &   Number of episodes per update\\
        Replay buffer size &  5000  & Maximum number of stored episodes     \\
        Target update interval  & 200   & Frequency of updating the target network \\
        Initial \(\epsilon\)    &   1.0 &   The initial \(\epsilon\) in the \(\epsilon\)-greedy policy  \\
        Final \(\epsilon\)  & 0.05  & The final \(\epsilon\) in the \(\epsilon\)-greedy policy  \\
        Anneal steps for \(\epsilon\) & 50,000 & Number of steps for linearly decay of \(\epsilon\) \\
        Discount \(\gamma\) &   0.99    &   Discount of future return \\
        Test interval   &   10,000  &   Frequency of test evaluation    \\
        Test episodes   &   32  &   Number pf episodes to test  \\
        \bottomrule
    \end{tabular}
    \caption{The hyperparameters for experiments}
    \label{hyperparameters}
\end{table}


\paragraph{Network structure}
We provide details of the network structures presented in Figure~\ref{img architecture of mrvd}. The individual Q network and the mixing network (\(Q_{\mathrm{tot}}\)) are shown in Figure~\ref{tot network structure}, and the joint action value network (\(\hat{Q}_{\mathrm{jt}}\)) is shown in Figure~\ref{jt network structure}. 

\begin{figure*}[h!]
    \centering
    \includegraphics[scale=0.65]{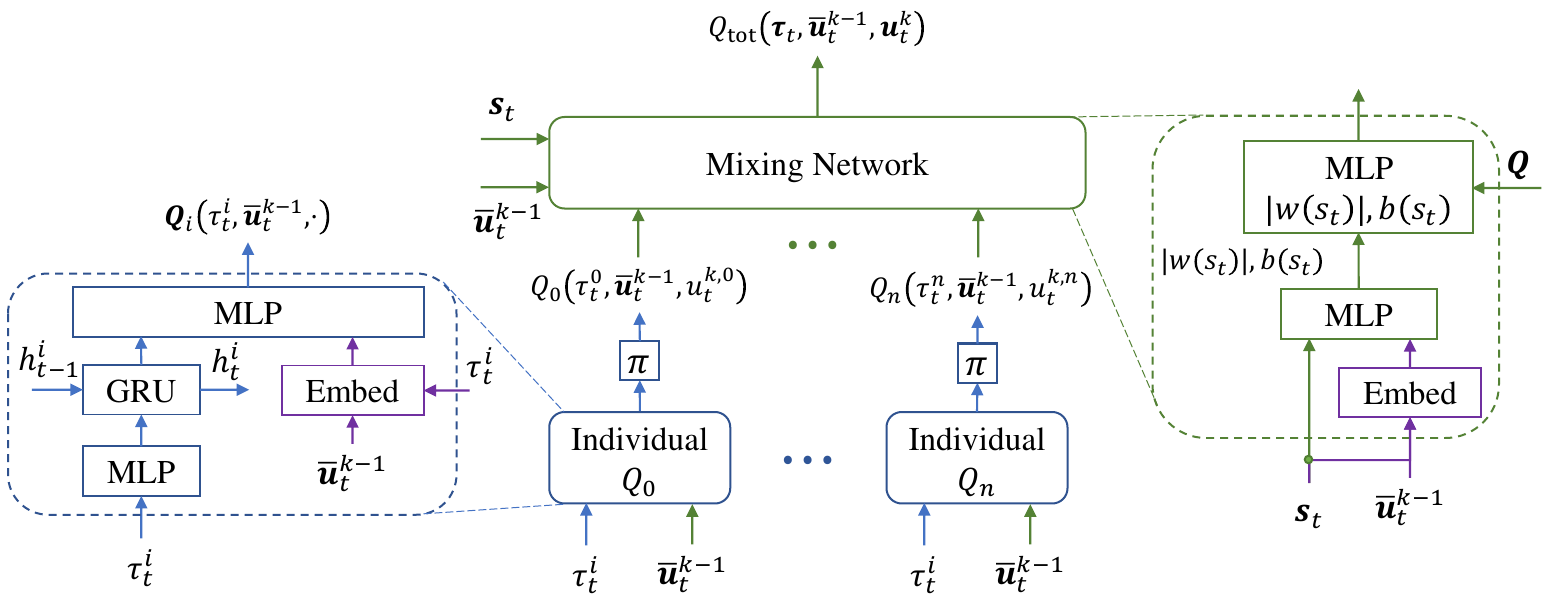}  
    \caption{The network structure of \(Q_{\mathrm{tot}}\).}
    \label{tot network structure}
\end{figure*}

\begin{figure*}[h!]
    \centering
    \includegraphics[scale=0.65]{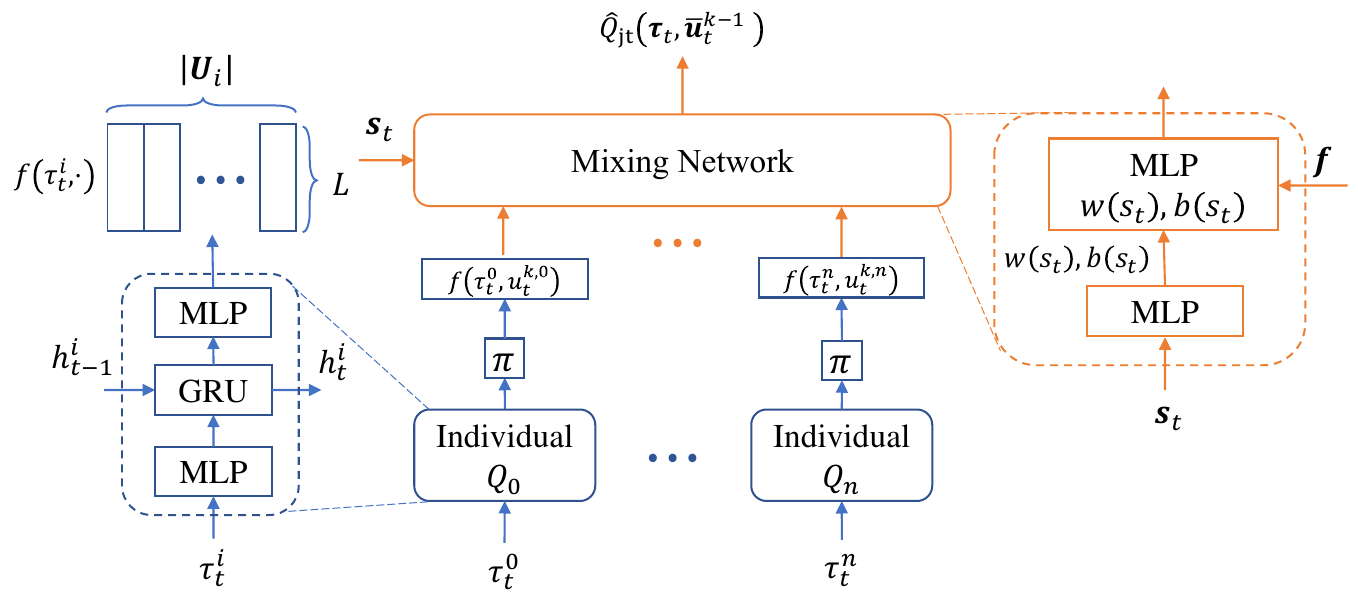}   
    \caption{The network structure of \(Q_{\mathrm{jt}}\).}
    \label{jt network structure}
\end{figure*}

The individual Q network processes the observation using a linear layer followed by a GRU~\citep{Ref:gru} with 64 hidden dimensions. The mixing network processes the state with 2-layer linear network (64 hidden dimension) and processes the individual Qs with 2-layer linear network, where the weights and bias (generated by the state's linear processor) have 32 hidden dimension. Inspired by the embedding module~\citep{Ref:embed2003}~\citep{Ref:embed2013}, we generate the features of \(\bar{\boldsymbol{u}}^k\) by selecting them from the continuous features (128 dimensions) of the observation or state. For the concatenated feature, we double the hidden dimensions to 128. 

The network structure of \(\hat{Q}_{\mathrm{jt}}\) is similar to that of \(Q_{\mathrm{tot}}\), with two key differences: First, we use a 3-layer linear network to process the state. Second, the output of the individual network has 64 dimensions. We input the state to \(Q_{\mathrm{jt}}\) in all SMAC scenarios except \textit{3s5z} and \textit{3s\_vs\_5z}, because masking observations of dead units loses critical information for return prediction in complex scenarios. In the \textit{bane\_vs\_bane} scenario, we replace the embedding feature with the action index to avoid GPU out-of-memory errors.

\subsection{One-Step Game}
\label{one-step game setup}
We randomly generate \(Q_{\mathrm{jt}}\) of the risk-reward game with the following steps: First, randomly generate the individual reward vectors \(r_i\geq0\). Second, we randomly generate a bijection \(\boldsymbol{\mathcal{U}}\to\boldsymbol{\mathcal{U}}\) that maps \(u_i\) to another action \(v_i\) for all agents \(i\), which is the exchange of rows and columns for a matrix. Finally, we let \(Q_{\mathrm{jt}}(\boldsymbol{u})=sign(\boldsymbol{v})*\sum\limits_{i=0}^n r_i(v_i)\) where \(sign(\boldsymbol{v})=1\) if all agents choose the same mapped action, and \(sign(\boldsymbol{v})=-1\) otherwise. Here, we give an example of the risk-reward matrix in Table~\ref{risk-reward matrix example}.

\begin{table}[h!]
    \centering
    \hspace{0.1\linewidth}
    \begin{subtable}{0.2\linewidth}
        \begin{tabular}{|c|c|c|c|}
            \hline
            \diagbox[width=3em]{\(r_1\)}{\(r_2\)}  & 3 & 4  &5  \\
            \hline
            0   &   \textbf{3} & -4  & -5     \\
            \hline
            1   &   -4     & \textbf{5} & -6      \\
            \hline
            2   &   -5     & -6       & \textbf{7}  \\
            \hline
        \end{tabular}
        \caption{ }
    \end{subtable}
    \hspace{0.1\linewidth}
    \begin{subtable}{0.2\linewidth}
        \begin{tabular}{|c|c|c|c|}
            \hline
            \diagbox[width=3em]{\(r_1\)}{\(r_2\)}  & 5 & 4  &3  \\
            \hline
            2   &   \textbf{7} & -6  & -5     \\
            \hline
            0   &   -5     & -4 & \textbf{3}      \\
            \hline
            1   &   -6     & \textbf{5}   &    -4  \\
            \hline
        \end{tabular}
        \caption{ }
    \end{subtable}
    \hspace{0.1\linewidth}
    \caption{An example of the risk-reward matrix. (a): the risk-reward matrix before action mapping, where positive rewards are in the diagonal. (b): the risk-reward matrix after action mapping, which scatters positive rewards.}
    \label{risk-reward matrix example}
\end{table}

We constrain the values of \(Q_{\mathrm{jt}}\) to the interval \([-10,10]\). To achieve this, we first generate the reward vectors uniformly from \([0,1]\), then divide each component by \(n\), and finally multiply the result by \(10\). The normalized return in Figure~\ref{risk result} is defined as
\begin{equation}
    return(\boldsymbol{u})=
    \begin{cases}
        \frac{Q_{\mathrm{jt}}(\boldsymbol{u})-\min\limits_{\boldsymbol{u}^+\in \boldsymbol{\mathcal{U}}^+} Q_{\mathrm{jt}}(\boldsymbol{u}^+)}{\max\limits_{\boldsymbol{u}^+\in \boldsymbol{\mathcal{U}}^+} Q_{\mathrm{jt}}(\boldsymbol{u}^+)-\min\limits_{\boldsymbol{u}^+\in \boldsymbol{\mathcal{U}}^+} Q_{\mathrm{jt}}(\boldsymbol{u}^+)}   &   Q_{\mathrm{jt}}(\boldsymbol{u})\geq0 \\
        \\
        \frac{Q_{\mathrm{jt}}(\boldsymbol{u})-\min\limits_{\boldsymbol{u}^-\in \boldsymbol{\mathcal{U}}^-} Q_{\mathrm{jt}}(\boldsymbol{u}^-)}{\min\limits_{\boldsymbol{u}^-\in \boldsymbol{\mathcal{U}}^-} Q_{\mathrm{jt}}(\boldsymbol{u}^-)}   &   Q_{\mathrm{jt}}(\boldsymbol{u})<0 
    \end{cases}
\end{equation}
where \(\boldsymbol{\mathcal{U}}^+=\{\boldsymbol{u}^+|\boldsymbol{u}^+\in\boldsymbol{\mathcal{U}}, Q_{\mathrm{jt}}(\boldsymbol{u}^+)\geq0\}\) and \(\boldsymbol{\mathcal{U}}^+\cup\boldsymbol{\mathcal{U}}^-=\boldsymbol{\mathcal{U}}\).

\subsection{Predator Prey}
\label{predator prey setup}
The predator-prey task~\citep{Ref:DCG} involves a multi-agent scenario where 8 predators cooperate to capture 8 prey. Each predator can choose from six possible actions: four directional movements (up, down, left, right), staying still, or attempting to "capture" prey.

\paragraph{Capture Conditions}
\begin{itemize}
    \item A predator can only select the "capture" action when occupying the same position with prey.
    \item Successful capture requires at least two predators simultaneously choosing "capture" on the same prey.
    \item Successful capture yields \(+10\) reward.
    \item If only one predator chooses "capture" for a prey, the team receives a punishment.
    \item The captured prey and successful predators are immediately removed from the environment. Observations of removed agents are masked with zeros.
\end{itemize}


\subsection{StarCraft II Multi-Agent Challenge}
\label{smac setup}
The SMAC (SC2.4.6 version) involves two opposing teams competing to defeat each other. Players control one team's units with the objective of eliminating all enemy units. The reward system consists of three components:
\begin{itemize}
    \item \textbf{Kill Reward}: \(+10\) for eliminating an enemy unit by reducing its HP to zero.
    \item \textbf{Win Reward}: \(+200\) for defeating all enemy units and winning the game.
    \item \textbf{Damage}: The amount of change in an enemy's HP.
\end{itemize}

In the original PyMARL implementation~\citep{Ref:smac}, rewards are enforced as positive by taking absolute values. However, in scenarios where enemies can regenerate health/shields, agents might artificially inflate rewards by allowing recovery and reinflicting damage, rather than focusing on winning. To address this, we adopt the reward without taking its absolute values.

We illustrate why monotonic value factorizations excel in SMAC from two aspects. First, the individual action has slight impact on returns. Take \textit{3s\_vs\_5z} scenario for example, the enemy unit, Zealot, has \(100\) health and \(50\) recoverable shields. Thus, the maximum return is
\begin{equation}
    \underbrace{150*5}_{Damage} + \underbrace{10 *5}_{Kill} + \underbrace{200}_{Win} = 1000
\end{equation}
In particular, damages constitutes the major component of return (75\% proportion). Within a step, an individual agent's action only affects its damage, and its action will not cause a sharp change of the state. Consequently, a slight deviation from the optimal joint action is insufficient to cause a significant drop in return. Therefore, monotonic value factorization has advantages in SMAC, as without the sharp drop in return near the global optimum, monotonic value factorization can attain the global optimum.

Second, choosing the optimal individual action also performs the best in average. We consider a "focus fire" case that frequently appears in SMAC. In this case, two agents fight with three enemies denoted as "A", "B", and "C", and focusing fire on the same target is the trick to win. Those enemies are identical units, except that Enemy A currently has less HP. We present the \(Q_{\mathrm{jt}}\) of this case in Table~\ref{focus fire case}. Although \(Q_{\mathrm{jt}}\) in Table~\ref{focus fire case} is non-monotonic, the average return of choosing action A is superior to that of choosing other actions. In this case, monotonic value factorization can obtain the optimal value.

\begin{table}[h]
    \centering
        \begin{tabular}{c|c|c|c|c|}
             \multicolumn{1}{c}{ } &\multicolumn{1}{c}{\(A\)}   &\multicolumn{1}{c}{\(B\)}  &\multicolumn{1}{c}{\(C\)}  &\multicolumn{1}{c}{\(\emptyset\)}     \\
            \hhline{~|-|-|-|-|}
            \(A\)   & 2.2 &  1.1  & 1.1  & -0.9     \\
            \hhline{~|-|-|-|-|}
            \(B\)   &1.1     & 2 & 1     &-1      \\
            \hhline{~|-|-|-|-|}
            \(C\)   &1.1     & 1       & 2   & -1 \\
            \hhline{~|-|-|-|-|}
            \(\emptyset\)   & -0.9    &   -1    & -1  & -2    \\
            \hhline{~|-|-|-|-|}
        \end{tabular}
        \hspace{0.03\linewidth}
        \vspace{0.02\textwidth}
    \caption{\(Q_{\mathrm{jt}}\) in the "focus fire" case of the SMAC environment: The action is denoted by its attack target,  and the action \(\emptyset\) means no attack. The data in the table only represents the relative advantages of each joint action rather than actual values. The agents receive a higher return (+2) when both focus fire on the same target, a moderate return (+1) when attacking different targets, and lower returns for not attacking. Additionally, attacking low-health units (Enemy A) grants a +0.1 bonus to the returns.}
    \label{focus fire case}
\end{table}

\section{Ablation Study}
\label{ablation}
\subsection{Experiment and Analysis}
\label{Subsec: main ablation}
We analyze the impact of multi-round iteration and strictness in improving greedy action (Equation~(\ref{tot loss round>1})). We design MRVF-single, which uses a single round, and MRVF-non-strict, which also uses the original \(Q_{\mathrm{jt}}\) as the target value of \(Q_{\mathrm{tot}}\) (Equation~(\ref{tot loss round 1})) for round \(k>1\). We present the results in risk-reward games (Figure~\ref{risk ablation}), predator-prey tasks (Figure~\ref{pred ablation}), and SMAC (Figure~\ref{smac ablation}). 

\begin{figure*}[h!]
    \centering
    \includegraphics[scale=0.28]{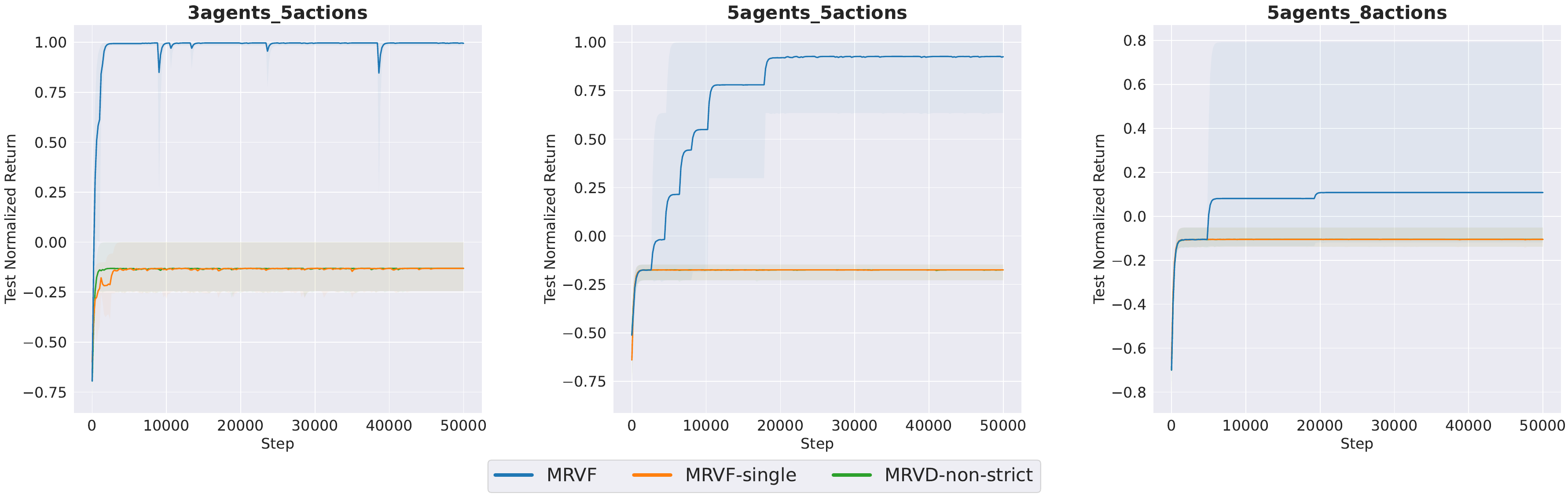} 
    \caption{Test normalized return in the risk-reward games.}
    \label{risk ablation}
\end{figure*}

\begin{figure*}[h!]
    \centering
    \includegraphics[scale=0.28]{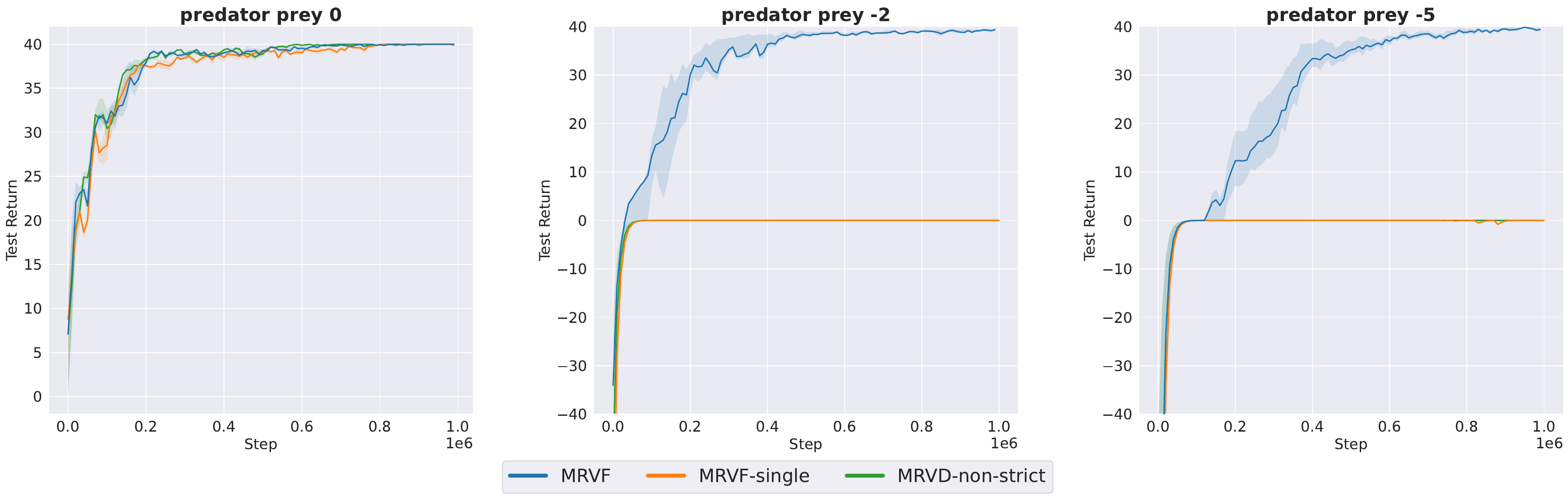} 
    \caption{Test return in the predator-prey tasks with punishments 0 (left), -2 (middle), and -5 (right).}
    \label{pred ablation}
\end{figure*}

The results presented in Figure~\ref{risk ablation} and Figure~\ref{pred ablation} demonstrate that the strict improvement contributes to better performance in highly non-monotonic scenarios. The explanation for this is that MRVF-non-strict cannot guarantee improvement in the greedy action compared to the preceding round, which prevents it from obtaining the optimal action within a given number of rounds.

\begin{figure*}[h!]
    \centering
    \includegraphics[scale=0.28]{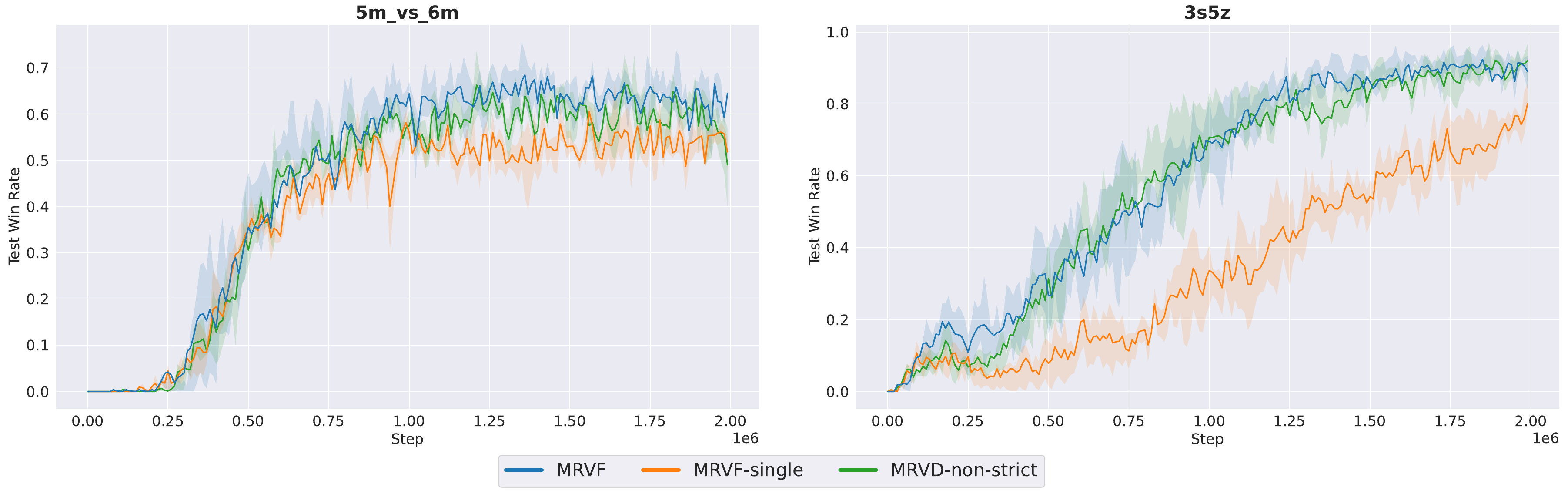} 
    \caption{Test win rate in the SMAC benchmarks.}
    \label{smac ablation}
\end{figure*}

As for approximately monotonic scenarios like SMAC, Figure~\ref{smac ablation} shows a slight improvement compared to single-round factorization. In these scenarios, the improvement comes from additional information in decision-making. Specifically, by feeding \(\bar{\boldsymbol{u}}_{t}^{k-1}\) into the individual Q networks, each agent knows the intentions of others and could infer their states accordingly. Without global information, this additional information contributes to better decision-making. In addition, due to the approximate monotonicity of the payoff, the first round of MRVF already produces a satisfactory solution, which explains the slightness of the improvement.


\subsection{Discussion on Rounds}
We will discuss how many rounds MRVF requires to obtain the optimal solution under scenarios with different levels of non-monotonicity. As discussed in Section~\ref{Subsec: main ablation}, a second round is required to obtain the optimal solution in highly non-monotonic environments. To support our claim, we calculate the proportion of taking \(\boldsymbol{u}_t^k\) as the final action for each round \(k\), as shown in Figure~\ref{round result}. 

\begin{figure*}[h]
    \centering
    \includegraphics[scale=0.28]{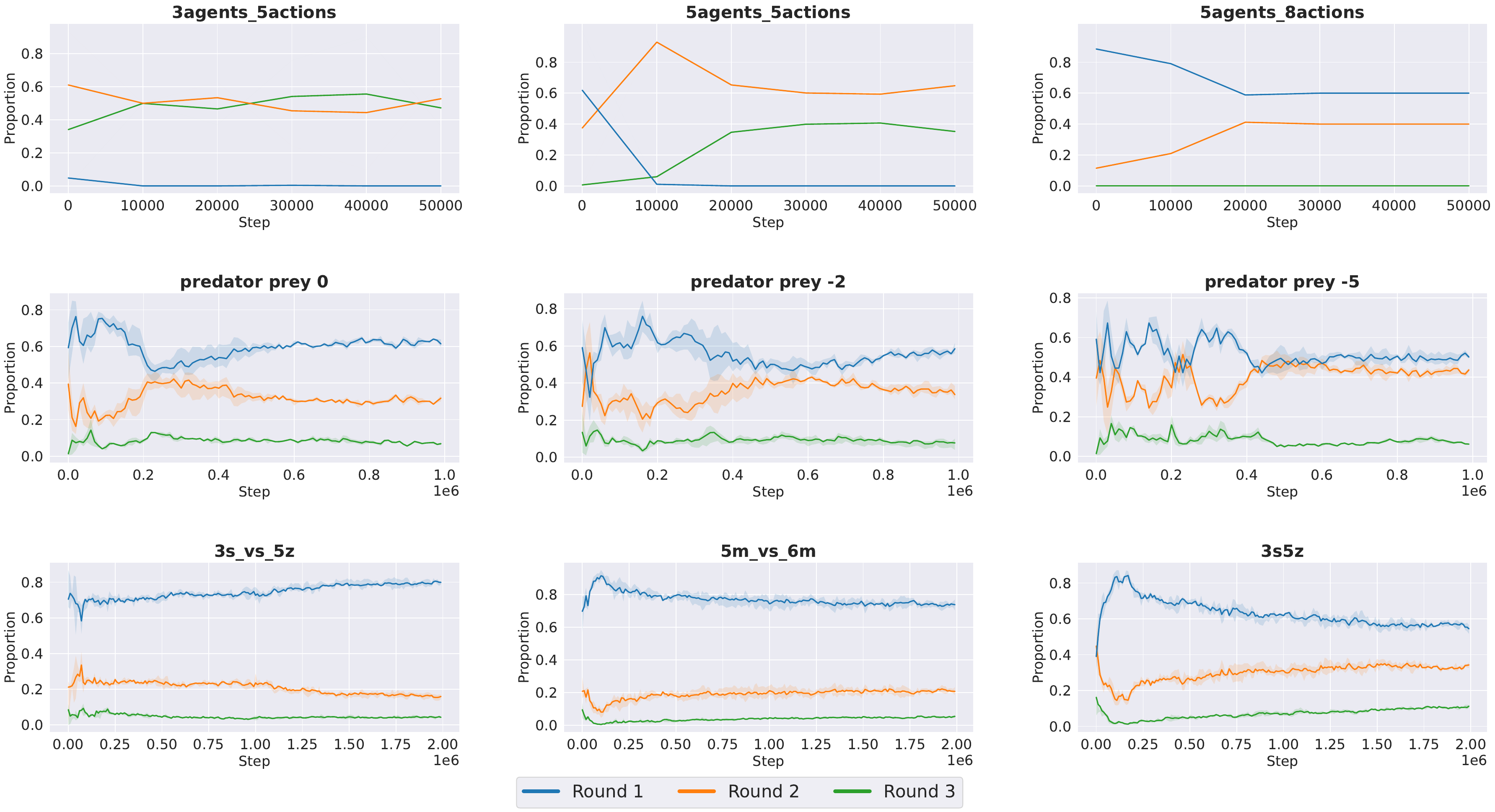}  
    \caption{Proportion of final actions generated in each round throughout test episodes.}
    \label{round result}
\end{figure*}

From Figure~\ref{round result}, we find that in the predator-prey environment, as the penalty increases, the proportion of second-round decisions rises from 30\% to nearly 50\%. In contrast, for approximately monotonic scenarios such as SMAC, the proportion of second-round decisions typically remains around 20\%, and the performance gap between multi-round and single-round is small. In addition, a third round is generally unnecessary unless the reward is so sparse that the second round fails to find the optimal action. As shown in Figure~\ref{round result}, third-round decisions account for only about 5\% across all predator-prey and SMAC scenarios. Nevertheless, to balance performance and efficiency, we recommend starting with a maximum of three rounds and adjusting it based on the proportion.

\section{Limitations}
\label{limitations}
Although MRVF achieves significant improvements over existing methods in highly non-monotonic scenarios, it faces the issue of high computational complexity: MRVF requires calculating \(Q_{\mathrm{tot}}\) and \(Q_{\mathrm{jt}}\) in each round. Although we reduce computational complexity by reusing parts of network outputs (\textit{e.g.}, individual features in \(Q_{\mathrm{jt}}\)) and early terminating the multi-round iterations for evaluation, the computational complexity of MRVF is still \(O(K(N+1))\), while that of QMIX is at \(O(N)\). In addition, as shown in Equation~(\ref{tot loss round 1}) and Equation~(\ref{tot loss round>1}), the update of \(Q_{\mathrm{tot}}\) depends on the results of \(Q_{\mathrm{jt}}\). As a result, the update of \(Q_{\mathrm{tot}}\) is delayed compared to standard TD learning, which leads to a slow rise of the win rate in \textit{3s\_vs\_5z} (Figure~\ref{smac result}). However, learning directly from \(Q_{\mathrm{jt}}\) is inevitable for environments that are not retraceable. Otherwise, if exploration of branching trajectories is allowed, TD targets could be used in the update of \(Q_{\mathrm{tot}}\).

We summarize the reasons why MRVF may fail to achieve the optimal solution in practice: First, the approximation errors of \(\hat{Q}_{\mathrm{jt}}\) propagate to \(Q_{\mathrm{tot}}\), which may prevent the \(Q_{\mathrm{tot}}\) from reaching the global optimum. Second, monotonic factorization is not ideal. Since we use a neural network as the mixing network, the approximation errors introduced by the network make it difficult to obtain the \(Q_{\mathrm{tot}}\) that minimizes \(L_{\mathrm{tot}}\). Third, overly sparse rewards may further amplify these approximation errors. As shown in Figure~\ref{risk result}, MRVF struggles to achieve optimal performance in \textit{5agents\_8actions}.


\section{Supplementary Experiment and Analysis}
In this section, we provide analysis and experiments that are not included in the main part of this paper. 
\subsection{StarCraft II Multi-Agent Challenge}
In Section~\ref{Subsec:suboptimality of single-round}, we analyze the stability of QPLEX~\citep{Ref:qplex}, WQMIX~\citep{Ref:wqmix} and ResQ~\citep{Ref:resq}. We reveal that suboptimal stable points are the main cause of the suboptimality in highly non-monotonic scenarios including risk-reward games and predator-prey tasks with large punishments. However, in Figure~\ref{smac result}, these methods show inferior performance compared to monotonic value factorizations in SMAC scenarios. We provide the following explanation for this phenomenon.

\paragraph{QPLEX and ResQ} From Table~\ref{QPLEX fail while QMIX success example} in Appendix~\ref{stability analysis of QPLEX}, in cases where QMIX can easily obtain the global optimum, QPLEX may still fail due to multiple suboptimal stable points. In addition, for ResQ, Theorem~\ref{resq weakly stable point} shows that \(\boldsymbol{u}^*\) is a weakly stable point in general cases. The weak stability of ResQ is particularly severe because any action can potentially become \(\bar{\boldsymbol{u}}\) (Lemma~\ref{resq zero loss}). Consequently, QPLEX and ResQ do not perform well in SMAC. 

\paragraph{WQMIX} WQMIX is designed for highly non-monotonic scenarios, and from Figure~\ref{risk result} and Figure~\ref{pred result}, it outperforms monotonic value factorization methods. Although WQMIX is supposed to perform similarly to QMIX in weakly non-monotonic or monotonic scenarios, it shows poor performance in SMAC (Figure~\ref{smac result}). We provide three factors resulting in this phenomenon. 
\begin{itemize}
    \item[(1)] \textbf{The joint action-value network}: Individual features constitute only a small portion of the mixer's input in the joint action-value network. Consequently, the joint action-value network hardly distinguishes different actions, since action information is primarily embedded in these individual features.
    \item[(2)] \textbf{Decentralized \(\epsilon\)-greedy policy}: The decentralized \(\epsilon\)-greedy policy (defined in Equation~(\ref{decentralized epsilon greedy})) combined with the weight in WQMIX enhances the convergence to the local optimum, since the decentralized \(\epsilon\)-greedy policy explores locally and the weight discards lower values around the local optimum.
    \item[(3)] \textbf{Weight}: When \(\alpha<1\), the weight reduces sample efficiency, as downweighting samples is similar to discarding them.
\end{itemize}

We conduct additional experiments to support our analysis. We replace the network of \(\hat{Q}_{\mathrm{jt}}\) with that of ours and replace the decentralized \(\epsilon\)-greedy policy with the centralized one. The result in Figure~\ref{smac cwqmix result} demonstrates that these three factors do result in the performance gap between WQMIX and QMIX. 

\begin{figure*}[h!]
    \centering
    \includegraphics[scale=0.4]{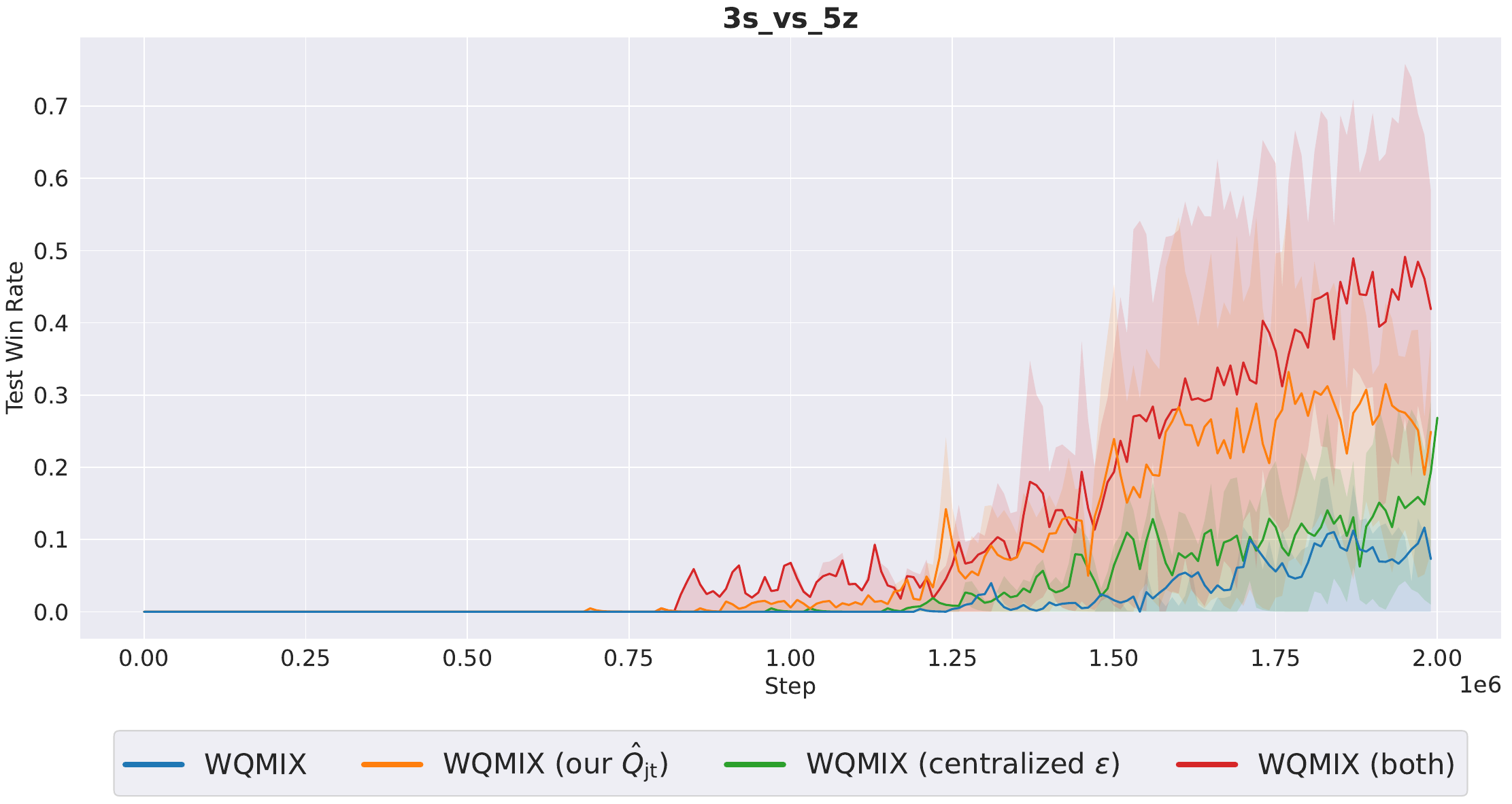} 
    \caption{Test win rate in \textit{3s\_vs\_5z}. WQMIX (\(\hat{Q}_{\mathrm{jt}}\)) represents WQMIX replacing the network of \(\hat{Q}_{\mathrm{jt}}\) with that of ours. WQMIX (centralized \(\epsilon\)) represents WQMIX replacing the policy with the centralized one. WQMIX (both) represents WQMIX with both modifications.}
    \label{smac cwqmix result}
\end{figure*}

\subsection{Super-hard Scenario of SMAC}
We conduct additional experiments on the \textit{6h\_vs\_8z} scenario, which is a particularly challenging scenario where the "focus fire" is a key tactic for victory~\citep{Ref:smac}. We adopt a setup similar to the \textit{bane\_vs\_bane} scenario, except that we set anneal steps for \(\epsilon\) to \(10^6\). The results in Figure~\ref{smac hard result} show a slight difference between MRVF and QMIX, supporting our analysis in Appendix~\ref{smac setup}.

\begin{figure*}[h!]
    \centering
    \includegraphics[scale=0.4]{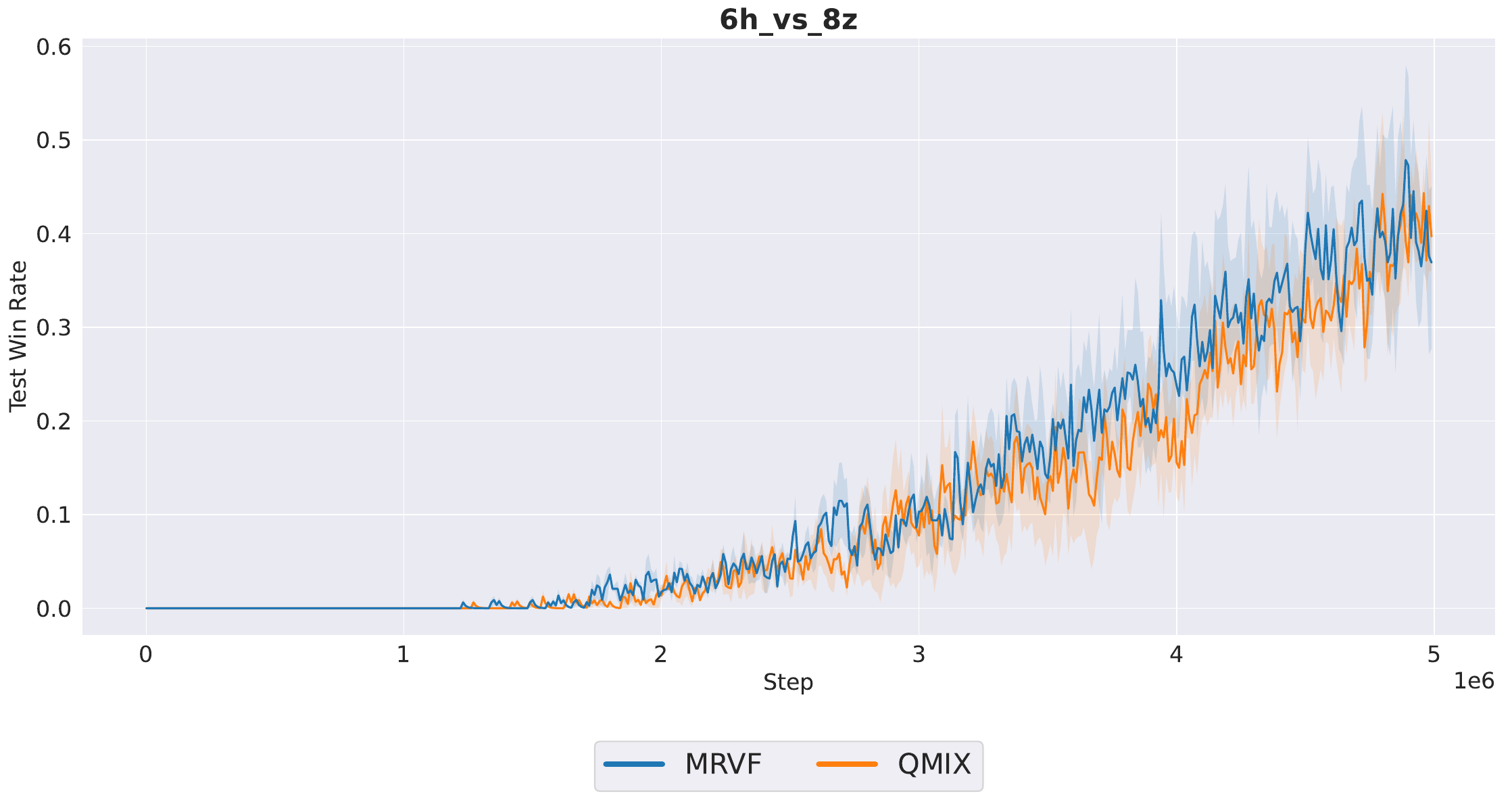}   
    \caption{Test win rate in \textit{6h\_vs\_8z} in SMAC benchmarks.}
    \label{smac hard result}
\end{figure*}

\subsection{SMACv2}
\label{exp in SMACv2}
Results on SMACv2~\citep{Ref:smacv2} are shown in Figure~\ref{smacv2 result}. In terms of average return, MRVF slightly outperforms QMIX, while the two algorithms achieve comparable win rates. Overall, the performance gap between MRVF and QMIX is small. This is because SMACv2 inherits the reward design of SMAC; as shown in Appendix~\ref{smac setup}, its payoff is nearly monotonic. However, we observe a mismatch between the return and the win rate in SMACv2, suggesting that the reward design may need to account for environmental stochasticity. When this mismatch occurs, we argue that the return better reflects the performance than the win rate, since the objective of MARL is to maximize the return.

\begin{figure*}[h!]
    \centering
    \includegraphics[scale=0.28]{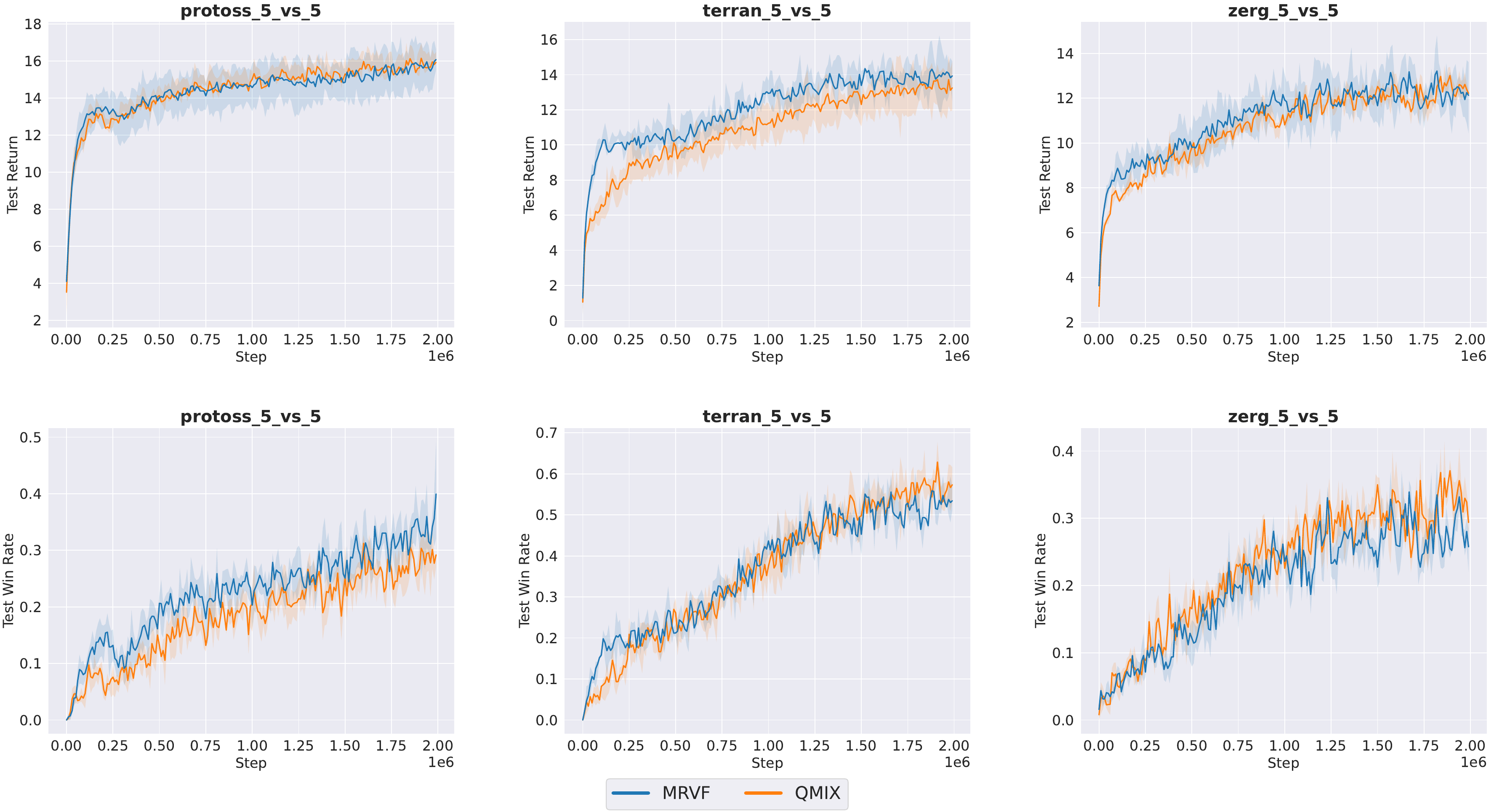} 
    \caption{Test return (first row) and win rate (second row) in SMACv2 benchmarks.}
    \label{smacv2 result}
\end{figure*}

\subsection{Predator Prey}
From the results in Figure~\ref{pred result}, QMIX fails to obtain the optimal action when the penalty is -2. However, an analysis based on ideal QMIX will reveal that the optimal action is a stable point for any \(\epsilon\in(0,1]\). In fact, this does not indicate a theoretical flaw. An explanation is that suboptimal actions are also stable points as \(\epsilon\to0\).

We will provide an analysis of the stable points when \(\epsilon\to0\). Taking two agents as an example, the reward function for the predator-prey environment with a penalty of -2 is given in Equation~(\ref{pred reward}).
\begin{equation}
    reward =
    \begin{bmatrix}
        10 & -2 & \cdots & -2 \\
        -2 & 0 & \cdots & 0   \\
        \vdots  & \vdots & \ddots   & \vdots \\
        -2 & 0 & \cdots & 0   \\
    \end{bmatrix}
    \label{pred reward}
\end{equation}
where the reward is a matrix of shape \(6\times6\).

Consider \(\widetilde{\boldsymbol{u}}\) that neither agent executes the "capture" action, which yields a reward of 0. As \(\epsilon\to0\), \(Q^*_{\mathrm{tot}}\) that minimizes \(L_{\mathrm{tot}}\) is given in Equation~(\ref{lim for non optimum}).

\begin{equation}
    \lim_{\epsilon\to0}Q^*_{\mathrm{tot}}|_{\widetilde{\boldsymbol{u}}\neq\boldsymbol{u}^*}  =
    \begin{bmatrix}
        -2 & -2 & \cdots & -2 \\
        -2 & 0  & \cdots & 0   \\
        \vdots   & \vdots & \ddots   & \vdots \\
        -2 & 0 & \cdots & 0   \\
    \end{bmatrix}
    , L_\mathrm{tot}^*=0+0*o(\epsilon)+12^2*o(\epsilon^2)
    \label{lim for non optimum}
\end{equation}
where \(\bar{{\boldsymbol{u}}}^* = \widetilde{\boldsymbol{u}}\), and both correspond to the agents that perform the "non-capture" action. Therefore, "non-capture" action pairs are stable points in predator prey with penalty -2.

Since suboptimal actions are in the majority, it is difficult for the greedy action to converge to the optimal one. However, because suboptimal actions are stable points only when \(\epsilon\) is small, we can enhance the convergence to the optimal action by slowing down the decay of \(\epsilon\). To support our analysis, we conduct additional experiments on QMIX with \(\epsilon\) fixed at \(1\) and QMIX with \(\epsilon\) decaying in one million steps, as shown in Figure~\ref{pred qmix result}. From the result in predator prey -2, agents successfully learn to capture prey in both settings, yet they still fail when the penalty is -5. In addition, increasing the weight of the optimal action in the policy can enhance the convergence to the optimal action. For example, the weighting mechanism in WQMIX serves a similar purpose. However, these methods still fail when the penalty is set to -5.

\begin{figure*}[h!]
    \centering
    \includegraphics[scale=0.28]{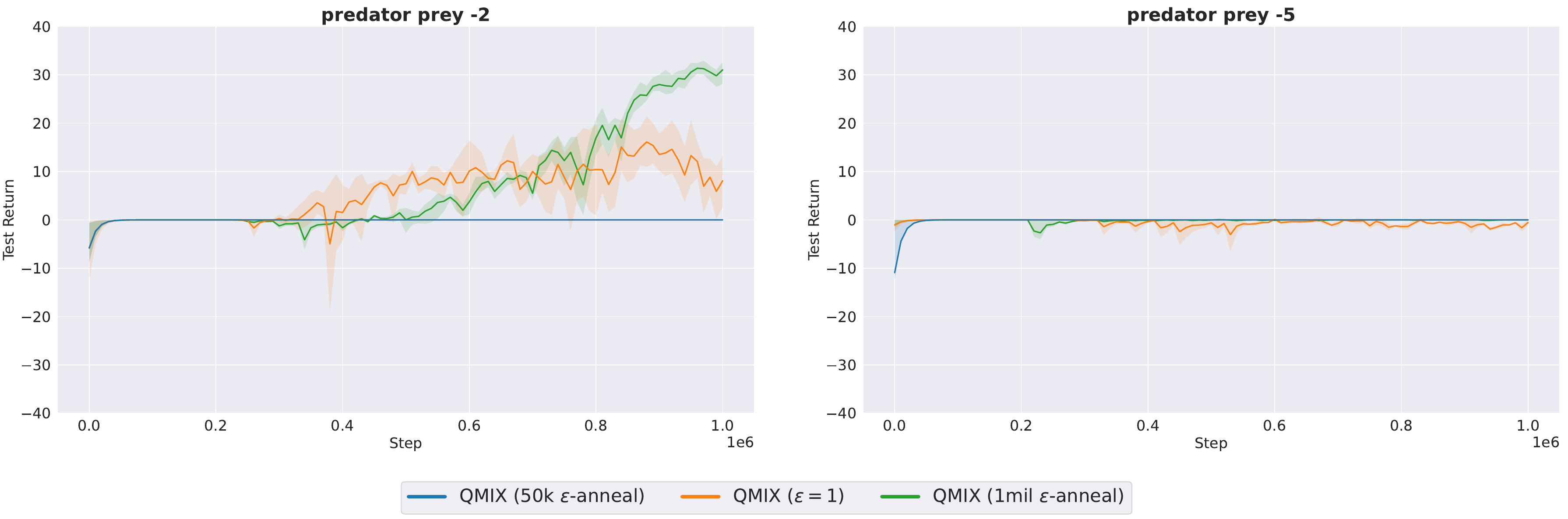}  
    \caption{Test return in the predator prey tasks with punishments -2 (left), and -5 (right). QMIX (50k \(\epsilon\)-anneal) represents QMIX with \(5*10^4\) anneal steps for \(\epsilon\). QMIX (1mil \(\epsilon\)-anneal) represents QMIX with \(10^6\) anneal steps for \(\epsilon\). QMIX (\(\epsilon=1\)) represents QMIX with \(\epsilon=1\) constantly.}
    \label{pred qmix result}
\end{figure*}

\subsection{Comparison with other MARL baselines}
\label{exp on other baselines}
In Section~\ref{experiment}, we compare MRVF with algorithms such as WQMIX~\citep{Ref:wqmix} that claim to address the representation limitation in value factorization. However, the results show that none of these methods fully resolves this limitation. Other value-factorization approaches, including NA2Q~\citep{Ref:na2q} HYGMA~\citep{Ref:hygma} and GoMARL~\citep{Ref:gomarl}, are not designed to address the representation limitation. On-policy methods such as MAPPO~\citep{Ref:mappo} suffer from similar issues, converging to suboptimal Nash Equilibria~\citep{Ref:tadppo}, which are analogous to suboptimal stable points in value factorization. Consequently, all of these methods perform poorly in environments with highly non-monotonic payoffs, as illustrated in Figure~\ref{mappo pp result}.
\begin{figure*}[h]
    \centering
    \includegraphics[scale=0.28]{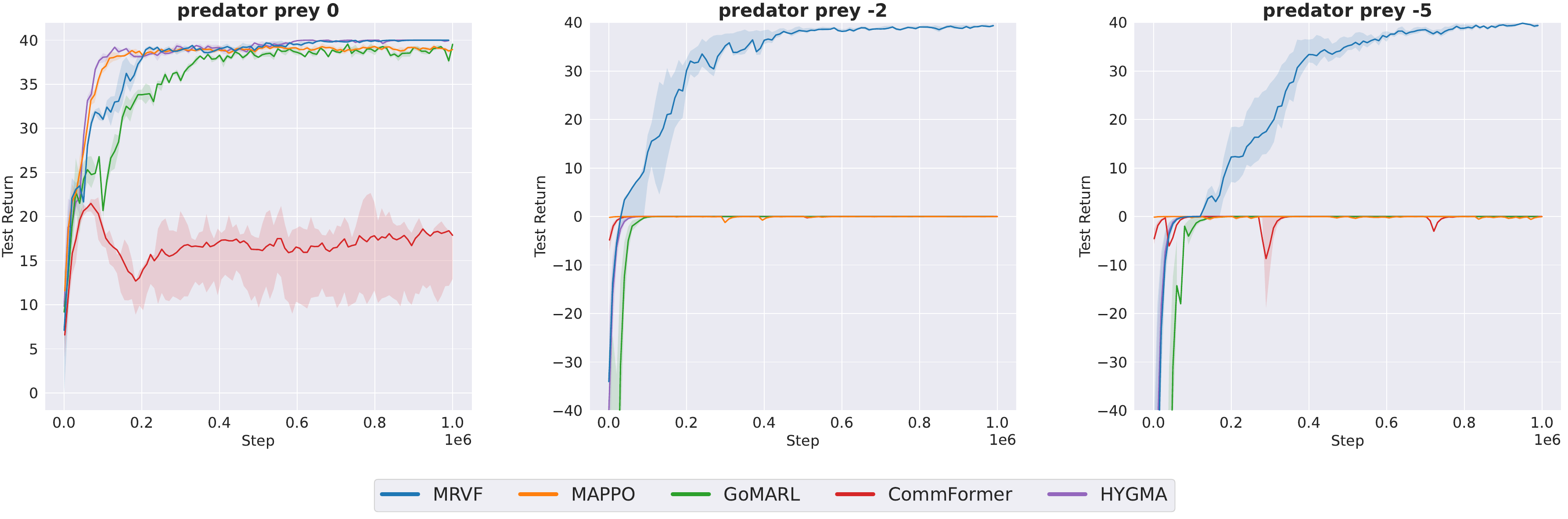}   
    \caption{Test return in the predator prey tasks with punishments 0 (left), -2 (middle), and -5 (right).}
    \label{mappo pp result}
\end{figure*}

The comparison results on the SMAC environment are shown in Figure~\ref{mappo smac result}. It can be observed that the win rates of the on-policy methods (including MAPPO~\citep{Ref:mappo} and CommFormer~\citep{Ref:commformer}) increase more slowly than that of MRVF, and their final performance is usually lower. Moreover, the performance of the on-policy methods depends on hyperparameters, as in Table~\ref{on-policy MARL hyperparameters}, the network and its inputs for \textit{3s\_vs\_5z} differ from those used in the other scenarios.

\begin{figure*}[h]
    \centering
    \includegraphics[scale=0.28]{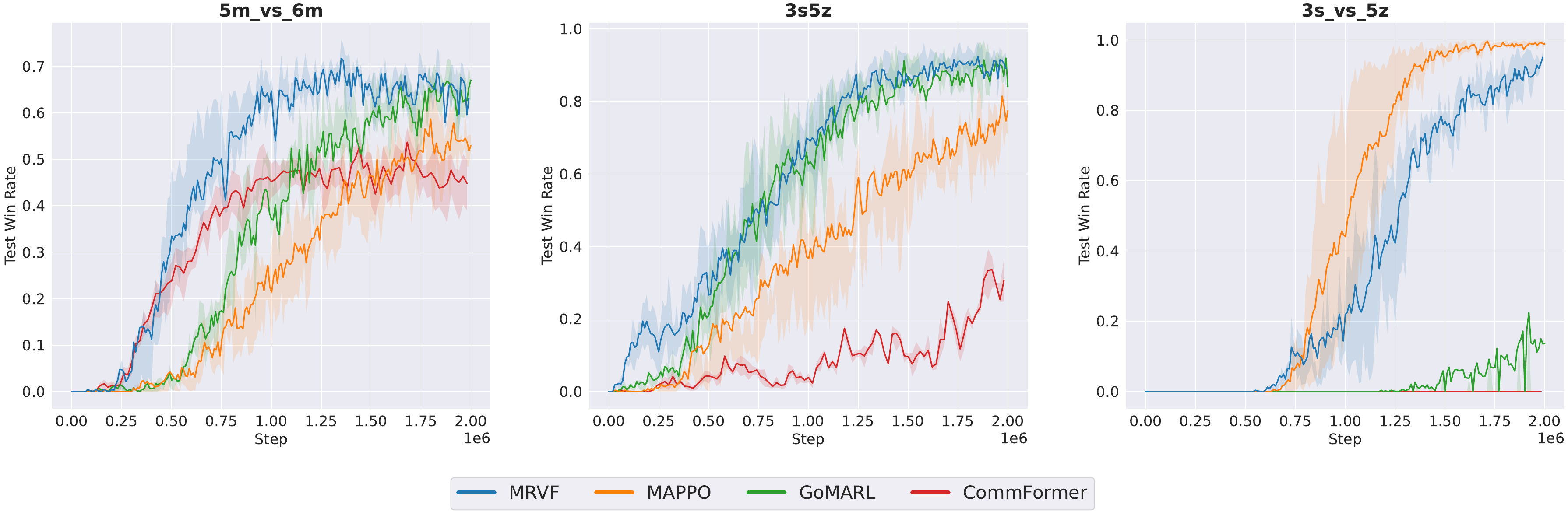} 
    \caption{Test win rate in the SMAC benchmarks.}
    \label{mappo smac result}
\end{figure*}

We use the implementation of GoMARL~\footnote{https://github.com/zyfsjycc/GoMARL}, MAPPO~\footnote{https://github.com/marlbenchmark/on-policy}, CommFormer~\footnote{https://github.com/charleshsc/CommFormer} and HYGMA~\footnote{https://github.com/mysteryelder/HYGMA} are from their open-source repositories. The hyperparameter settings for the on-policy methods are listed in Table~\ref{on-policy MARL hyperparameters}. Please refer to \citep{Ref:mappo} and \citep{Ref:commformer} for further details.
\begin{table}[h]
    \centering
    \begin{tabular}{ccccccc}
        \toprule
         & epoch & clip & network & stacked frames & training threads & rollout threads  \\
        \midrule
        \textit{3s\_vs\_5z} & 15 & 0.05 & mlp & 4 & 1 & 8  \\
        \textit{5m\_vs\_6m} & 10 & 0.05 & rnn & 1 & 1 & 8  \\
        \textit{3s5z} & 5 & 0.2 & rnn & 1 & 1 & 8  \\
        predator prey & 10 & 0.2 & rnn & 1 & 1 & 8  \\
        \bottomrule
    \end{tabular}
    \caption{The hyperparameters for on-policy MARL}
    \label{on-policy MARL hyperparameters}
\end{table}

\section{Pseudocode of MRVF}

\begin{algorithm}[htbp]
\caption{MRVF Forward Computation}
\label{alg:mrvf forward}
\begin{algorithmic}[1]  
    \STATE \textbf{Input:} time step \(t\), training flag \(is\_training\)
    \STATE Get action-observation history \(\boldsymbol{\tau}_t\) from the environment.
    \STATE Set \(\bar{\boldsymbol{u}}_t^0\) to a default action
    \STATE \(Q_\mathrm{jt}(\boldsymbol{\tau}_t,\bar{\boldsymbol{u}}_t^{0})\gets -\infty\) 
    \STATE \(flag \gets True\)
    \FOR{\(k = 1, 2, \cdots, K\)}
        \STATE Compute individual Q values based on \(\boldsymbol{\tau}_t\) and \(\bar{\boldsymbol{u}}_t^{k-1}\)
        \STATE Obtain \(\bar{\boldsymbol{u}}_t^{k}\) based on individual Q values
        \IF{\(Q_\mathrm{jt}(\boldsymbol{\tau}_t,\bar{\boldsymbol{u}}_t^{k})>Q_\mathrm{jt}(\boldsymbol{\tau}_t,\bar{\boldsymbol{u}}_t^{k-1})\) and \(flag\) }
            \STATE \(\bar{\boldsymbol{u}}_t \gets \bar{\boldsymbol{u}}_t^{k}\)
        \ELSE
            \STATE \(flag \gets False\)
            \IF{ not \(is\_training\)}
                \STATE \textbf{break}
            \ENDIF
        \ENDIF
        \STATE \(\boldsymbol{u}_t^k \gets \bar{\boldsymbol{u}}_t^{k}\)
    \ENDFOR
    \STATE \(\boldsymbol{u}_t\gets \bar{\boldsymbol{u}}_t\)
    \IF{\(is\_training\)}
        \STATE Generate a random number \(x\in[0,1]\) 
        \IF{\(x<p\)}
            \STATE Randomly selet a round \(k\in\{1,2,\cdots,K\}\)
            \STATE \(\boldsymbol{u}_t\gets \bar{\boldsymbol{u}}_t^{k}\)
        \ELSE
            \STATE Generate a random number \(y\in[0,1]\)
            \IF{\(y<\epsilon\)}
                \STATE Randomly selet an action \(\boldsymbol{u}_\mathrm{rand}\in\boldsymbol{\mathcal{U}}\)
                \FOR{\(k = 1, 2, \cdots, K\)}
                    \STATE \(\boldsymbol{u}_t^{k}\gets \boldsymbol{u}_\mathrm{rand}\)
                \ENDFOR
                \STATE \(\boldsymbol{u}_t\gets \boldsymbol{u}_\mathrm{rand}\)
            \ENDIF
        \ENDIF
        \STATE return \(\boldsymbol{u}_t\) and \(\boldsymbol{u}_t^1,\cdots, \boldsymbol{u}_t^K\)
    \ELSE
        \STATE return \(\boldsymbol{u}_t\)
    \ENDIF
\end{algorithmic}
\end{algorithm}

\begin{algorithm}[htbp]
\caption{MRVF Training}
\label{alg:mrvf training}
\begin{algorithmic}[1]  
\STATE Initial parameters of mixing network, joint Q network and individual Q networks
\STATE Initialize replay buffer \(\mathcal{D}\)
\STATE \(is\_training\gets True\)
\WHILE{\(total\_step<max\_total\_step\)}
    \STATE Initialize the environment
    \FOR{\(t = 1, 2, \cdots, max\_step\)}
        \STATE Get action-observation history \(\boldsymbol{\tau}_t\) from the environment.
        \STATE Obtain \(\boldsymbol{u}_t\) and \(\boldsymbol{u}_t^1,\cdots, \boldsymbol{u}_t^K\) through Forward Computation of MRVF
        \STATE Execute actions \(\boldsymbol{u}_t\), observe reward \(r_t\) and next action-observation history \(\boldsymbol{\tau}_{t+1}\)
        \STATE Store transition \((\boldsymbol{\tau}_t,\{ \boldsymbol{u}_t^1,\cdots, \boldsymbol{u}_t^K\}, \boldsymbol{u}_t, r_t, \boldsymbol{\tau}_{t+1})\) in replay buffer \(\mathcal{D}\)
        \STATE Decay \(\epsilon\)
        \STATE \(total\_step\gets total\_step + 1\)
    \ENDFOR
    \STATE Sample minibatch \(\mathcal{B}\) from \(\mathcal{D}\)
    \FOR{\(t = 1, 2, \cdots, max\_step\)}
        \STATE Get transition \((\boldsymbol{\tau}_t, \boldsymbol{u}_t, r_t, \boldsymbol{\tau}_{t+1},\bar{\boldsymbol{u}}_{t+1})\) from \(\mathcal{B}\)
        \STATE Calculate \(L_\mathrm{jt}\) accroding to Equation~\ref{jt loss}
        \FOR{\(k = 1, 2, \cdots, K\)}
            \IF{\(k=1\)}
                \STATE Get tuple \((\boldsymbol{\tau}_t, \boldsymbol{u}_t^k)\) from \(\mathcal{B}\)
                \STATE Calculate \(L_\mathrm{tot}\) accroding to Equation~\ref{tot loss round 1}
            \ELSE
                \STATE Get tuple \((\boldsymbol{\tau}_t, \boldsymbol{u}_t^k,\bar{\boldsymbol{u}}_t^{k-1})\) from \(\mathcal{B}\)
                \STATE Calculate \(L_\mathrm{tot}\) accroding to Equation~\ref{tot loss round>1}
            \ENDIF
        \ENDFOR
    \ENDFOR
    \STATE Update parameters using gradient descent
\ENDWHILE
\end{algorithmic}
\end{algorithm}

\end{document}